\documentclass[final,leqno]{siamltex}

\pdfoutput=1

\usepackage{amssymb, amsmath, subfigure, epsfig}
\usepackage{morefloats}
\usepackage{algorithm, algorithmic}
\usepackage[normalem]{ulem}
\usepackage[inner=3cm,outer=3cm]{geometry} 

\DeclareMathOperator{\Ker}{\mathcal{K}} 
\newcommand{\aKer}{\tilde{\Ker}}

\DeclareMathOperator{\bigO}{\mathcal{O}}

\newcommand{\reals}{\mathbb{R}}

\newcommand{\aK}{\tilde{K}}
\newcommand{\aA}{\tilde{A}}

\newcommand{\subK}{K^\prime} %

\newcommand{\subaK}{\tilde{\subK}}

\newcommand{\subA}{A^\prime} 

\newcommand{\ns}{{s}} 

\newcommand{\coherence}{\gamma} 

\newcommand{\lmin}{\lambda_{\textrm{min}}}
\newcommand{\lmax}{\lambda_{\textrm{max}}}
\newcommand{\mmin}{\mu_{\textrm{min}}}
\newcommand{\mmax}{\mu_{\textrm{max}}}

\newcommand{\expect}{\mathbb{E}}
\newcommand{\prob}{\mathbb{P}}

\DeclareMathOperator{\tr}{tr}

\title{Far-Field Compression for Fast Kernel Summation Methods in High Dimensions}

\author{William B. March\thanks{Corresponding Author, \textit{email:} march@ices.utexas.edu.} \and George Biros}

\begin{document}

\maketitle

\begin{abstract}
We consider fast kernel summations in high dimensions:
given a large set of points in $d$ dimensions (with $d \gg 3$) and a
pair-potential function (the {\em kernel} function), we 
compute a weighted sum of all pairwise kernel interactions for
each point in the set. Direct summation is equivalent to a (dense)
matrix-vector multiplication and scales quadratically with the number of
points. Fast kernel summation algorithms reduce this cost to log-linear
or linear complexity.

Treecodes and Fast Multipole Methods (FMMs) deliver tremendous speedups
by constructing approximate representations of interactions of points
that are far from each other. In algebraic terms, these
representations correspond to low-rank approximations of blocks of the
overall interaction matrix. Existing approaches require an excessive
number of kernel evaluations with increasing $d$ and number of points
in the dataset. 

To address this issue, we use a randomized algebraic approach in which
we first sample the rows of a block and then 
construct its approximate, low-rank interpolative
decomposition.  We examine the feasibility of this approach
theoretically and experimentally.  We provide a new theoretical result showing 
a tighter bound on the reconstruction error from uniformly sampling rows  
than the existing state-of-the-art. We demonstrate that our sampling approach 
is competitive with existing
(but prohibitively expensive) methods from the literature. We also
construct kernel matrices for the Laplacian, Gaussian, and polynomial
kernels -- all commonly used in physics and data analysis.  We
explore the numerical properties of blocks of these matrices, and show
that they are amenable to our approach.  Depending on the data set,
our randomized algorithm can successfully compute low rank
approximations in high dimensions. We report results for data sets with 
ambient dimensions from four to 1,000.
\end{abstract}



\pagestyle{myheadings}
\thispagestyle{plain}

\section{Introduction}

Given $n$ source points $x_j$ with densities $q_j$, $m$
target points $y_i$, and a kernel function $\Ker$,
we seek to evaluate the \emph{kernel sum}
\begin{equation}\label{e:sum}
u_i = \sum_{j = 1}^n \Ker(y_i, x_j) q_j = \sum_{j=1}^n K_{ij}q_j
\end{equation}
for each target $y_i$, with $K_{ij}=\Ker(y_i,x_j)$.  Computing
$u\in\mathbb{R}^m$ is equivalent to a matrix-vector multiplication,
$u=Kq, \ K\in\mathbb{R}^{m\times n}$, and it requires $\bigO(n
m)$ work. It is prohibitively expensive for large $m$ and $n$.  Fast
kernel summation algorithms (also known as generalized N-body
problems) aim to provide an approximate solution with guaranteed error
in 
$\bigO(n + m)$
time. They do so by identifying and
approximating blocks of $K$ that have low-rank structure.

Fast kernel summations are a fundamental operation in computational
physics. They are related to the solution of 
partial differential equations in which $\Ker$ is the
corresponding Green's function.  Examples include the
3D Laplace potential (reciprocal distance kernel) and the heat
potential (Gaussian kernel). 

Kernel summations are also fundamental to non-parametric statistics
and machine learning tasks such as density estimation, regression, and 
classification.  Linear
inference methods such as support vector
machines \cite{suykens1999least} and 
dimension reduction methods such as principal components
analysis \cite{mika1998kernel} can be efficiently generalized to non-linear
methods by replacing inner products with kernel evaluations 
\cite{bishop2006pattern}. Problems in statistics
and machine learning are often characterized by very high-dimensional inputs.

Existing fast algorithms for the kernel summation problem hinge on the
construction of efficient approximations of interactions\footnote{We use the 
term \emph{interaction} between two points to refer to 
the value of the kernel $\Ker$.}
between groups of sources and targets when these groups are far apart
or \emph{well separated} (see section \ref{sec_approach}).  In the physics/PDE
community, they are known as far-field approximations. From a linear
algebraic point-of-view, they correspond to low-rank decompositions of
blocks of the matrix $K$.
These approximations can be roughly grouped in three categories: {\bf
analytic, semi-analytic, and algebraic}.  

In analytic methods, Taylor or
kernel-dependent special function expansions are used to approximate the
far-field. The Fast Multipole Method (FMM)
~\cite{greengard1987fast} is one of these. Semi-analytic methods rely only on
kernel evaluations, but the low-rank constructions use the
analytical properties of the underlying kernels. For example, the
kernel-independent fast multipole method~\cite{ying2004kernel}
requires that the underlying kernel is the Green's function of a PDE.
Finally, algebraic methods
(\emph{e.g.}~\cite{martinsson2007accelerated}) also only use kernel 
evaluations, but the only necessary condition is the existence a low-rank block 
structure for $K$.

In high dimensions, most existing methods
fail.  There are two main reasons for the lack of scalability of
analytic and semi-analytic  methods. 
The first reason is that all existing schemes require too many terms
for the kernel approximation.  Analytic and semi-analytic schemes can
deliver approximations to arbitrary accuracy (in practice all the way to
machine precision) in $\bigO(n+m)$ time, but the constant can be very large. 
For $p$ terms in the series expansion, they require $p=c^d$ or $p=c^{d-1}$
terms to deliver error that decays exponentially in $c>1$.  Variants
that can scale reasonably well beyond three dimensions scale as $p=d^c$
and deliver error that decays algebraically in $c$. For sufficiently
large $d$ and $c>1$, either of  these methods become too expensive~\cite{griebel2013fast}.

The second reason for lack of scalability of existing schemes is that
they do not take advantage of any lower-dimensional structures
that may be present in the data. 
For example, the data may be embedded in a low-dimensional manifold.
This is mostly relevant in data analysis
applications in which often the important dimension is not
the \emph{ambient} one but instead an \emph{intrinsic} dimension
that depends on the distribution of the source and target points.

Algebraic approximations \cite{martinsson2007accelerated} are a promising 
direction for scalable methods in high dimensions.
These approximations are
based on the observation that Equation~\ref{e:sum} is a matrix-vector
product and certain blocks of the matrix have low-rank
structure. Algebraic methods are useful only if the approximation can be
computed efficiently. Efficient methods for low dimensions do exist,
but in high-dimensions they fail because the number of 
kernel evaluations required exceeds the cost of the direct summation.

Beyond scalability requirements, let us also mention the need to
support several different kernels in a block-box fashion. Analytic or
semi-analytic methods depend significantly on the type or class of
kernel.  Although there has been extensive work on these methods for 
classical kernels like the Gaussian, new kernel functions have been
developed for a wide variety of data types, such as
graphs \cite{kondor2002diffusion} and strings \cite{lodhi2002text}.
Also, adaptive density estimation methods use kernels with variable
bandwidth~\cite{silverman1986density}.
This observation further motivates the use of entirely algebraic
acceleration techniques for Equation~\ref{e:sum}.

\subsection{Contributions}

In this paper, we make the ideas discussed above more precise. First, we explore
the low-rank structure of the far-field of several widely-used kernels in 
high-dimensions,
and then we propose a new scheme that uses randomized sampling to construct
interpolative decompositions~\cite{martinsson2011randomized} of the
far-field.  
Our goal is to design far-field approximations that 
do not scale
exponentially with the ambient dimension of the input, do not require 
analytic information about the kernel function, and 
require a number of kernel evaluations that is smaller than the cost of 
the direct summation (in the case that the kernel is compressible).

In particular, our contributions are
the following:
\begin{itemize}
\item We examine the approximability of the far-field for the Gaussian, 
Laplacian, and polynomial 
kernels in high dimensions. In particular we look at the 
structure of blocks of the matrix $K$, and 
we carefully study the effects of dimensionality 
and bandwidth.

\item We propose a new sampling scheme, summarized in Figure 
\ref{fig_algebraic_sampling}, which can be combined with the
interpolative decomposition scheme~\cite{martinsson2007accelerated} to 
construct approximations of the far field. 

\item We provide empirical results that show the effectiveness of our method 
for compressing
general kernels for higher-dimensional data without prior knowledge of
the structure of the kernel or any low-dimensional structures in the
data. We show results for data sets with high ambient but low
intrinsic dimension.
Also, we explore kernel matrices for data sets from the UCI machine learning 
repository \cite{Bache+Lichman:2013}.

\item We show a new theoretical analysis of the reconstruction error of
sampling columns of a matrix uniformly at random. We show a factor of 
$\sqrt{m/s}$ improvement over the existing best result 
\cite{gittens2011spectral} 
for $m$ columns and $s$ samples. 

\item We explore the use of heuristic approximations to theoretically optimal
but prohibitively expensive sampling distributions. We show that in many
cases of interest, a computationally-inexpensive distribution based on 
nearest-neighbor information is as effective as one based on statistical 
leverage scores \cite{mahoney2009cur}.
\end{itemize}

\subsection{Limitations}

First, here we only explore the feasibility of our far-field
compression method. We do not integrate our work with a fast summation
algorithm, such as a treecode or FMM. This integration will be reported 
elsewhere \cite{siscaskit, ipdpsaskit}.

Second, our experiments cover a range of kernel functions, parameters,
and input distributions.  However, these are not
comprehensive. Further experimentation, particularly on data from real
application domains, would be informative.

\subsection{Related work}\label{s:related}

This paper builds on two largely distinct bodies of existing work: fast
kernel summation methods and randomized algorithms for linear
algebra. We briefly survey existing results.

\subsubsection{Kernel summation 
Methods\label{sec_related_work_kernel_summations}}

Broadly, fast kernel summation methods group the points using a 
space-partitioning tree, then approximate the interactions between distant 
groups of points. 
These methods can be categorized based on the method used to approximate groups 
of interactions. We group our survey of related work into analytic, 
semi-analytic, and algebraic methods. We describe several of these methods in
greater detail in section~\ref{sec_approach}.

\textbf{Analytic.} The most effective methods use analytic series expansions to
approximate these interactions. This approach has its roots in the
work of Barnes and Hut \cite{barnes1986hierarchical},
Appel \cite{appel1985efficient}, and Greengard and
Rokhlin \cite{greengard1987fast}. These algorithms have been applied
to the Laplace kernel up to three dimensions. The Fast Gauss
Transform \cite{greengard1991fast, yang2003improved, lee2006dual,
griebel2013fast} is a variant of the FMM for the Gaussian kernel.
Similar approaches have been applied to solving the kernel summation
problem for the Helmholtz \cite{darve2000fast, darve2000fast2} and
Maxwell equations \cite{chew2001fast}. 

\textbf{Semi-analytic.} This approach avoids the explicit use of series expansions.  The
contribution of a group of points can be approximated as the
contribution of a carefully chosen, smaller, group of equivalent
source points along with corresponding
densities~\cite{anderson1992implementation, berman1995grid}. These
ideas have been extended to the kernel independent fast multipole
method (KIFMM)~\cite{ying2004kernel} and the black-box fast multipole
method~\cite{messner-darve-e12}. Another kernel independent method
that works well in high dimensions is discussed
in~\cite{rahimi-recht07} in the case where the kernel is diagonalizable in 
Fourier space. We also mention kernel-independent methods that only require
the existence of bounds on the kernel as a function of distance 
\cite{gray2001n, lee2012distributed}.

\textbf{Algebraic.} Given a tree data structure that can be used to define the near and
far fields, numerical linear algebra methods can be
used to approximate the far field.  One set of algorithms uses the
truncated singular value decomposition (SVD) to directly compute an
approximation to the kernel sum \cite{kapur1998ies3, kapur1997fast}.
Several methods compute an approximate singular value decomposition of
the kernel operator \cite{yarvin1998generalized, yarvin1999improved,
gimbutas2003generalized}, and employ this approximation in the context
of an FMM scheme.  An alternative to the SVD is the interpolative
decomposition~\cite{martinsson2011randomized}, which uses columns of the 
matrix as basis vectors.

\subsubsection{Randomized linear algebra}

There is a rich literature on randomized algorithms for linear algebra
that attempt to construct low-rank approximations of matrices. 
We briefly highlight some of the
results with the most bearing on our work.  For a more comprehensive
review of randomized low rank approximations,
see \cite{halko2011finding, mahoney2012randomized}.

\textbf{Random projections.}
One approach employs the Johnson-Lindenstrauss lemma and the
observation that a randomly chosen subspace of $\mathbb{R}^d$ will
capture most of the action the kernel
interactions \cite{johnson1984extensions,
dasgupta2003elementary}. These \emph{random projection} methods, first
introduced by Sarlos \cite{sarlos2006improved}, have been successfully
applied to the construction of low-rank
decompositions \cite{martinsson2011randomized, woolfe2008fast}.
However, these methods require the application of a projection
operator to the matrix.  This scheme ends up being at least as expensive as
a matrix-vector multiply, making it inappropriate for our problem.

\textbf{Subsampling.}
Alternatively, one can use sampling to build an approximation. 
These methods vary the sampling distribution and which parts of the matrix to 
sample. The
question what kind of sampling to use and whether we sample columns,
rows or both.  
One approach samples individual entries of the matrix to obtain a sparse 
representation~\cite{achlioptas2001fast, achlioptas2007fast}. 
Other methods construct a distribution over rows or columns.  Frieze \emph{et 
al.}~\cite{frieze2004fast}
sample entire rows of the original matrix using a probability
proportional to the row Euclidean norm. Extensions of this work use
probabilities proportional to the volumes spanned by sets of
vectors~\cite{deshpande2006matrix,rudelson2007sampling}.

\textbf{Statistical leverage scores.}
Other papers utilize the concept of statistical leverage scores to
form an importance sampling
distribution \cite{drineas2004clustering,drineas2006subspace,
drineas2006subspace2, drineas2006polynomial,drineas2006fast2,
drineas2006fast3,drineas2005nystrom, talwalkar2010matrix}. Importance
sampling distributions based on the magnitudes of rows of the matrix
of right or left singular vectors provide excellent theoretical guarantees for 
matrix approximations, and are also effective in practice.
Broadly, these
algorithms show that we can achieve high accuracy
from a small ($\bigO(r \log r)$ for a rank $r$ matrix) number of
samples.  Related algorithms have been developed for the column-subset
selection problem \cite{boutsidis2009improved}, fast matrix-matrix
multiplication~\cite{drineas2006fast}, and least-squares solutions to
over-determined systems~\cite{drineas2011faster}.

\textbf{Nystrom methods.}
Another line of work in the machine learning community is Nystrom methods
\cite{williams2001using}. 
Broadly, these methods attempt to approximate a positive semi-definite 
matrix by sampling a subset of its columns. These approaches use 
uniform distributions \cite{talwalkar2010matrix, jin2011improved, 
gittens2011spectral}, and more 
complex distributions \cite{drineas2005nystrom, zhang2008improved, 
gittens2013revisiting}. These methods generally require the entire kernel matrix
to be low-rank, while treecodes only require the presence of low-rank 
sub-blocks.

\textbf{Compressed sensing.}
Another line of research relevant to our problem is compressed
sensing~\cite{candes2006robust, candes2007sparsity, candes2009exact}.
While not directly relevant to low-rank approximations, the theoretical
machinery developed in this context is used in our work.  In our case,
since we want to approximate the matrix-vector product, we cannot use
a method that touches all the entries the matrix. Also, we cannot
compute sampling probabilities, they are too expensive. As we will see,
the cost is too high even if we just sample some full rows (or
columns). 

\textbf{Other methods.}
We also mention one other randomized method for the evaluation of 
kernel summations \cite{lee-gray08}. This method directly samples the 
far-field interactions, which can lead to large error and does not exploit 
the low-rank structure of the matrix. 

In conclusion, all existing methods that are general enough for
high-dimensions require an excessive number of kernel evaluations. A
new scheme is required. 

\subsection{ASKIT}

We have incorporated the ideas in this paper into a treecode scheme, called 
\emph{ASKIT}. We provide details of serial \cite{siscaskit} and parallel
\cite{ipdpsaskit}
versions of this algorithm elsewhere.  In the present paper, we focus on a 
theoretical and experimental study of the underlying structures of kernel 
matrices and the sub-blocks exploited by treecodes.

\begin{table}[t!]
  \centering
\caption{Notation used throughout the paper.}
\begin{tabular}{|c|l|}
\hline
\multicolumn{2}{|c|}{Data Parameters} \\ \hline
$d$ & dimension of input \\
$x_j, \tilde{x}_j$ & source point and equivalent source or skeleton point \\
$y_i, \tilde{y}_i$   & target point and equivalent or subsampled target \\
$q_j$ & charge or density on a source point \\
$u_i = u(y_i)$ & potential at target point $i$ \\ \hline
\multicolumn{2}{|c|}{Kernel Functions and Matrices} \\ \hline
$\Ker, \aKer$  & kernel function and approximate kernel function \\
$K, \aK$ & kernel matrix (in $\mathbb{R}^{m \times n}$ with entries $\Ker(y_i, x_j)$) and a low-rank approximation of $K$ \\
$m, n$ & number of targets (rows of $K$) and sources (columns of $K$), with $m \gg n$ \\
$\subK$ & subsampled kernel matrix (in $\mathbb{R}^{s \times n}$ for $s \leq m$) \\
$\gamma^{(r)}(K)$ & coherence of matrix $K$ with respect to rank $r$ (Equation~\ref{e:gamma_def})\\
$\sigma_i(K)$ & $i^{\textrm{th}}$ singular value of matrix $K$ \\ \hline
\multicolumn{2}{|c|}{Experiment Parameters} \\ \hline
$N$ & total number of points sampled in experiments \\
$\xi$ & separation parameter between sources and targets \\
$h$ & kernel bandwidth \\
$s$ & number of samples / interpolation points / sampling parameter in experiments section \\
$r$ & rank of a matrix approximation / number of skeleton points \\
$\kappa$ & rank tolerance in experimental setup \\
$K_S, K_N, K_F$ & self, nearest neighbor, and far-field interactions (Equation~\ref{eqn_interation_splits})\\
$\epsilon$ & rank tolerance parameter used in experiments \\	
\hline
\end{tabular}
\end{table}

\subsection{Organization}
In section~\ref{sec_approach}, we give a brief outline of existing
methods for constructing low-rank approximations for kernel summation
and we highlight where these methods break down for high-dimensional
data.  We then describe our approach. We prove basic results in
section~\ref{sec_theory} and provide numerical experiments
illustrating the feasibility of our approach in
section~\ref{sec_experiment}. We provide proofs in the appendix 
(section~\ref{sec_appendix}).


\section{Overview of outgoing representations\label{sec_approach}}

Let $S=\{x_j\}_{j=1}^n$ be a set of $n$ sources with charges
$\{q_j\}_{j=1}^n$ and $T=\{y_i\}_{i=1}^m$ be a set of $m$
targets. Computing the potential $u_i=u(y_i)$ for all $i$ is equivalent to a
dense matrix-vector multiplication $u = K q$ and requires $\bigO(n m)$ 
work to compute exactly.
Many fast summation schemes construct an approximate kernel function
$\aKer_S(y_i)$ such that
\begin{equation}\label{e:farfield}
\aKer_S(y_i) \approx \sum_{j=1}^n \Ker(y_i, x_j) q_j, \quad \forall y_i \in T.
\end{equation}

For methods based on analytic expansions, a low-rank approximation of
$\Ker$ is constructed by finding functions $\phi_k, \psi_k$ such that
$\aKer_S(y_i) = \sum_{k=1}^p \psi_k(y_i) \phi_k(x_j) q_j$ with an error that
depends on $p$ and $\|y-x\|$. Once such representation is found, the
quantity $z_k = \sum_j \phi_k(x_j) q_j$ can be precomputed and used in
$\aKer(y_i) = \sum_k \psi_k(y_i) z_k$. When $p \ll n$, a
substantial speedup can be observed by replacing $\Ker$ with
$\aKer$.  Finding such low rank approximations (in the example
we just discussed, computing $z_k$ and $\psi_k$) is also referred to as
constructing the \emph{outgoing representations} of the source points
$S = \{x_j\}$.  

For many kernels, this approach also requires that the sets of sources and 
targets be \emph{well separated}. We require that  
\begin{equation}\label{e:separated}
\min_{y_i \in T} \min_{x_j \in S} \|y_i - x_j\|_2 > \delta,
\end{equation}
where $\delta$ is a tolerance that depends on the type of the treecode
used, the kernel, and the set of approximation functions used. In cases where 
the sets are not required to be well separated (such as the Gaussian kernel), 
we let $\delta = 0$.

\begin{figure}
        \centering
\subfigure[\textbf{Well-separated sources and targets.}\label{fig_well_sep_1}]{\includegraphics[width=0.4\textwidth]{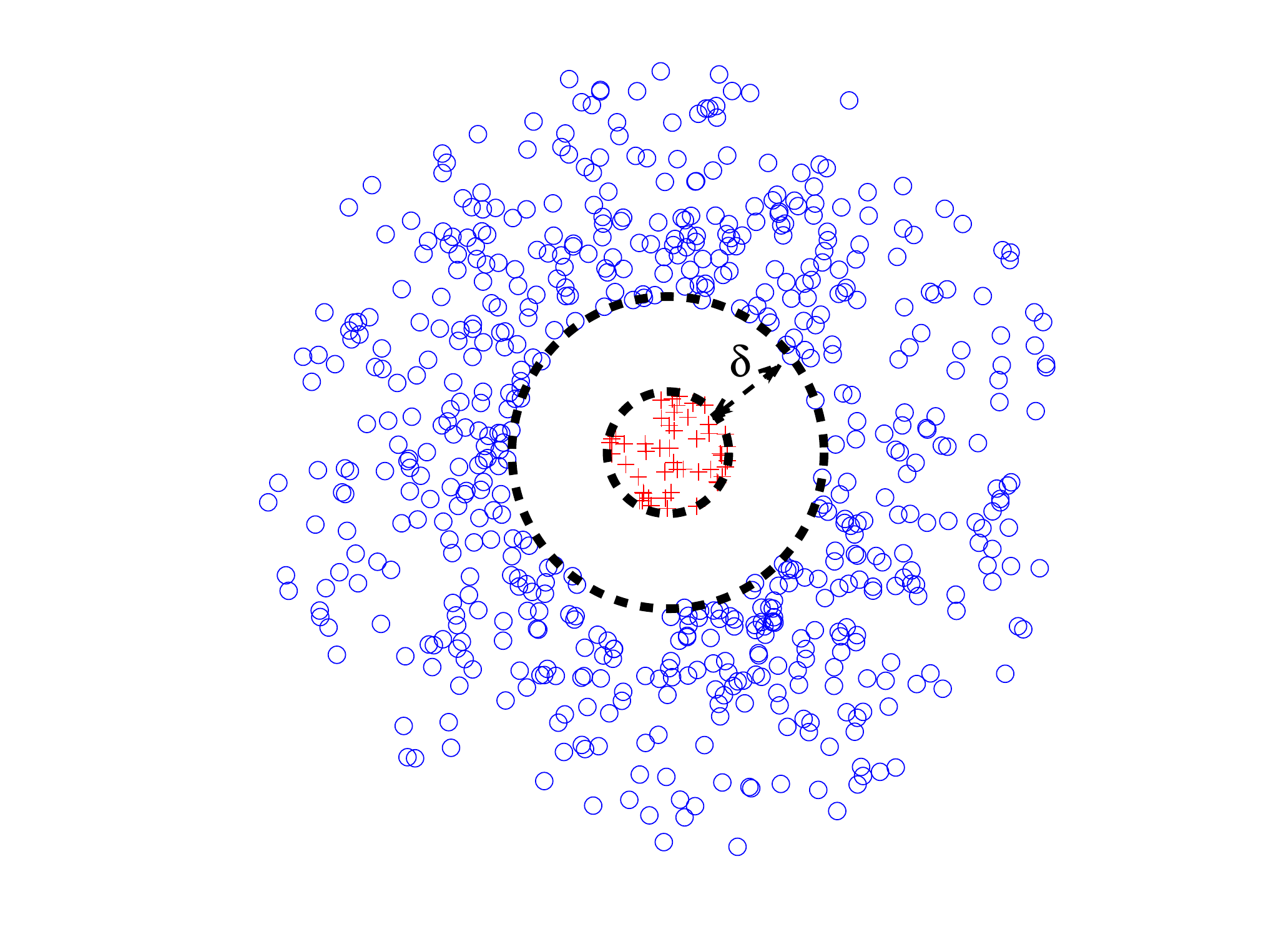}}
\subfigure[\textbf{Identifying well-separated sets.}\label{fig_well_sep_2}]{\includegraphics[width=0.4\textwidth]{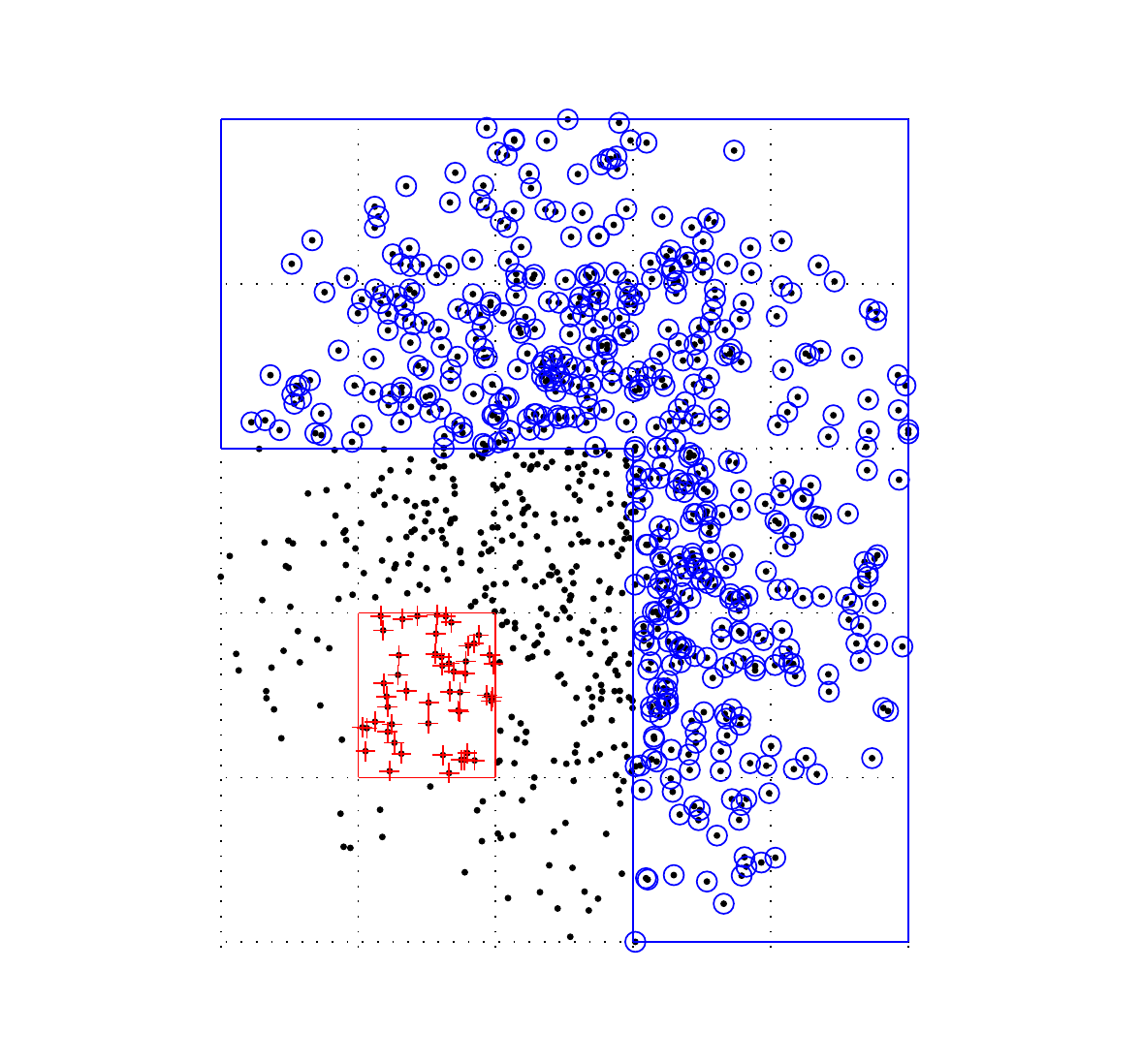}}
\caption{We illustrate the concept of well-separatedness used in 
fast kernel summation algorithms. We show the set $S$ of sources in red and
set $T$ of targets in blue. In Figure~\ref{fig_well_sep_1}, we show the sets
with the separation parameter $\delta$. In Figure~\ref{fig_well_sep_2}, we 
show the use of a space-partitioning tree to identify well-separated sets. 
The sources in the tree node highlighted in red are well separated from all of 
the target points in the nodes highlighted in blue. The fast kernel summation
literature typically refers to the well-separated
targets in blue as the \emph{far-field} and the remaining points (black and 
red) as the \emph{near-field}.
  \label{fig_well_separated}}
\end{figure}

For general source and target inputs and $\delta > 0$, this condition will not
hold. 
A fast kernel summation scheme can overcome this problem by using hierarchical
groupings of sources and targets (see Figure \ref{fig_well_sep_2}). 
Once such groups have been identified, for
each target point, we split interactions into near-field (those points which 
are not well separated) and the far-field (which are well separated from the 
target). We can then compute the near-field interactions directly, and 
efficiently approximate the far-field using an outgoing representation.

These hierarchical groupings are typically done using a
spatial data structure, such as a $d$-dimensional octree or a
$kd$-tree. Given such a tree, we perform two traversals. 
First, we construct an outgoing representation for each leaf. Then, we perform 
a preorder traversal, 
constructing an outgoing representation of each node by combining the 
representations of its children. Then, to evaluate the potential for each 
target point, we perform a postorder traversal, starting at the root.
At a node, we bound the error
due to applying our outgoing representation to approximate the potential
at the target. If the error is small
enough to satisfy some user-specified tolerance, we apply the
approximation.  Otherwise, we recurse, and evaluate the potential at
leaves directly if necessary. 


As described, the algorithm results in
$\bigO((n + m) \log n)$ complexity and is commonly referred to as
a \emph{treecode}. The Fast Multipole Method \cite{greengard1987fast}
extends this idea by also constructing an \emph{incoming
representation} which approximates the potentials due to a group of distant
sources at a target point; it results in $\bigO(n + m)$ complexity.


For the remainder of the paper, we strictly focus our attention on the 
construction of outgoing representations. 
Exactly the same process can be used to build incoming representations.
Our method's integration with a treecode and an FMM 
will be presented elsewhere. 

Next, we discuss the main techniques for
constructing the low-rank outgoing representations and their
shortcomings when applied to high dimensional data.
In this discussion, we fix a set $S$ of $n$ sources and a set $T$ of $m$ 
targets. The sources and targets will be assumed to 
be well-separated, where the precise value of $\delta$ will depend on the 
context and will be made explicit if needed. 


\subsection{Types of outgoing representations}

We have outlined the basic structure of fast summation schemes, but we
have left out the central detail -- constructing the
outgoing representation of a group of sources. To facilitate the discussion, we 
classify these methods into three groups:\footnote{This is by no means a widely accepted classification. We use it here to facilitate the discussion.}
\begin{itemize}
\item \emph{analytic} -- based on kernel-dependent series
expansions;
\item \emph{semi-analytic} -- based on
approximating the kernel at analysis-based target points;
\item \emph{algebraic} -- based on approximating 
 blocks of the kernel matrix directly.
\end{itemize}
 We now review each of these methods in turn and highlight how each
one scales poorly with the dimensionality of the problem. These methods 
are illustrated in 
Figure~\ref{fig_outgoing_rep_illustrations}.

\subsubsection{Analytic methods}  
 The potential at a target point $y$ sufficiently distant from a set of
sources is expanded around a point $x_c$ (generally the centroid of the
sources) as:
\begin{equation}
\aKer(y) = \sum_{k = 0}^\infty \psi_k \left(\|y - x_c \| \right) z_k
\end{equation}
for some coefficients $z_k$ and expansion basis $\psi_k$.
The approximation is constructed by truncating the
expansion after $p$ terms. 
Bases that deliver exponential convergence have been
constructed for the Laplace \cite{greengard1987fast},
Helmholtz \cite{darve2000fast, darve2000fast2}, 
Maxwell \cite{chew2001fast}, and Gaussian 
\cite{greengard1991fast, yang2003improved, lee2006dual, morariu2009automatic, griebel2013fast} kernels. Efficient approximations
have also been carried out using the SVD of the
kernel function \cite{hrycak1998improved, gimbutas2003generalized} and
in a basis of Chebyshev polynomials \cite{dutt1996fast,
fong2009black}.

In low dimensions, these expansions are optimal in terms of accuracy
and cost.  But the number of basis functions required generally scales
unfavorably with the dimension $d$. For instance, the fast Gauss
transform~\cite{greengard1991fast} requires $\bigO(c^d)$ terms for a tensor
product expansion, for a value $c>1$ which is related to the
convergence order of the series expansion to the exact solution. This
result has been improved to $d^c$ (using so-called sparse grid expansions),
but it is still expensive~\cite{yang2003improved,griebel2013fast} in
high dimensions. Furthermore, analytic expansions cannot take
advantage of the presence of nonlinear, lower dimensional structures
in the distribution of points.
Finally, they are
kernel specific and their stability and optimal performance can be
difficult to achieve. 
\begin{figure}[tp]
        \centering
\subfigure[\textbf{Analytic.}\label{fig_series}]{\includegraphics[width=0.3\textwidth]{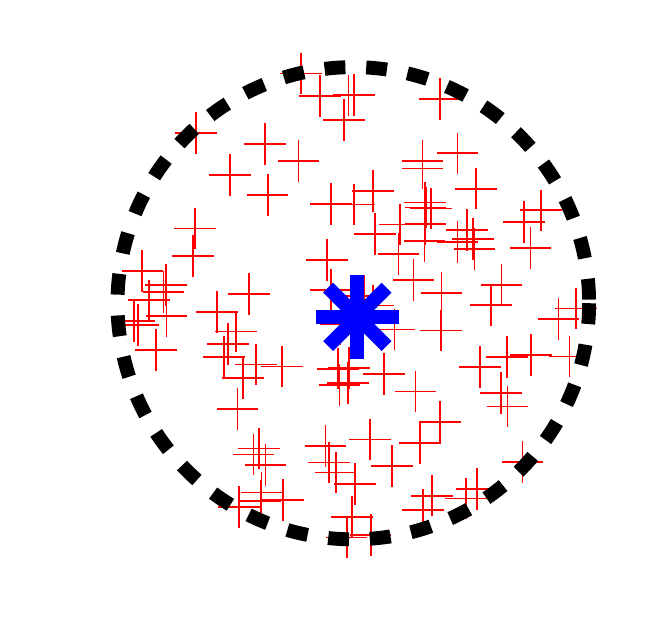}}
        \subfigure[\textbf{Semi-Analytic.}\label{fig_interpolation}]{\includegraphics[width=0.3\textwidth]{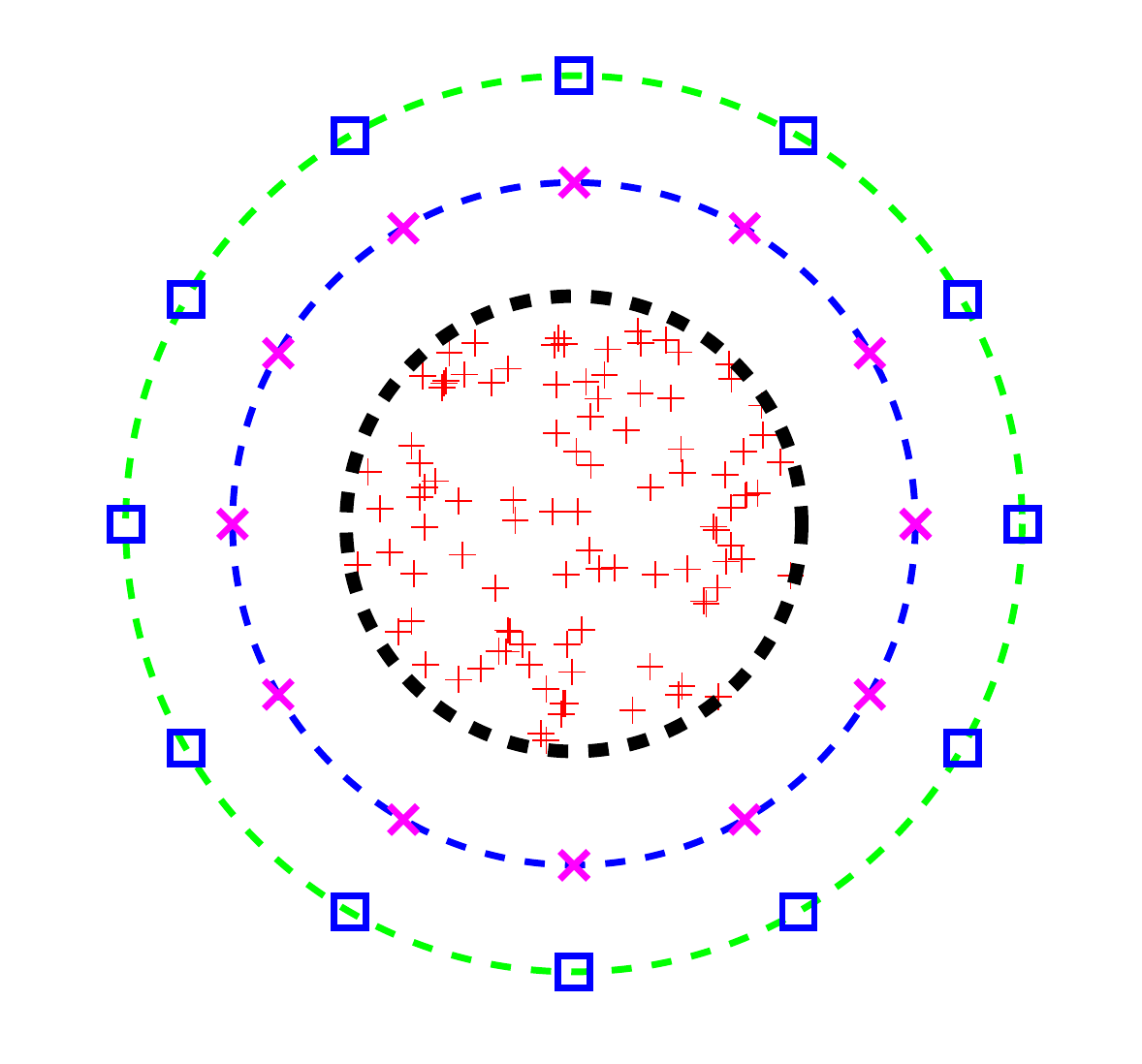}}
        \subfigure[\textbf{Algebraic.}\label{fig_skeletonization}]{\includegraphics[width=0.3\textwidth]{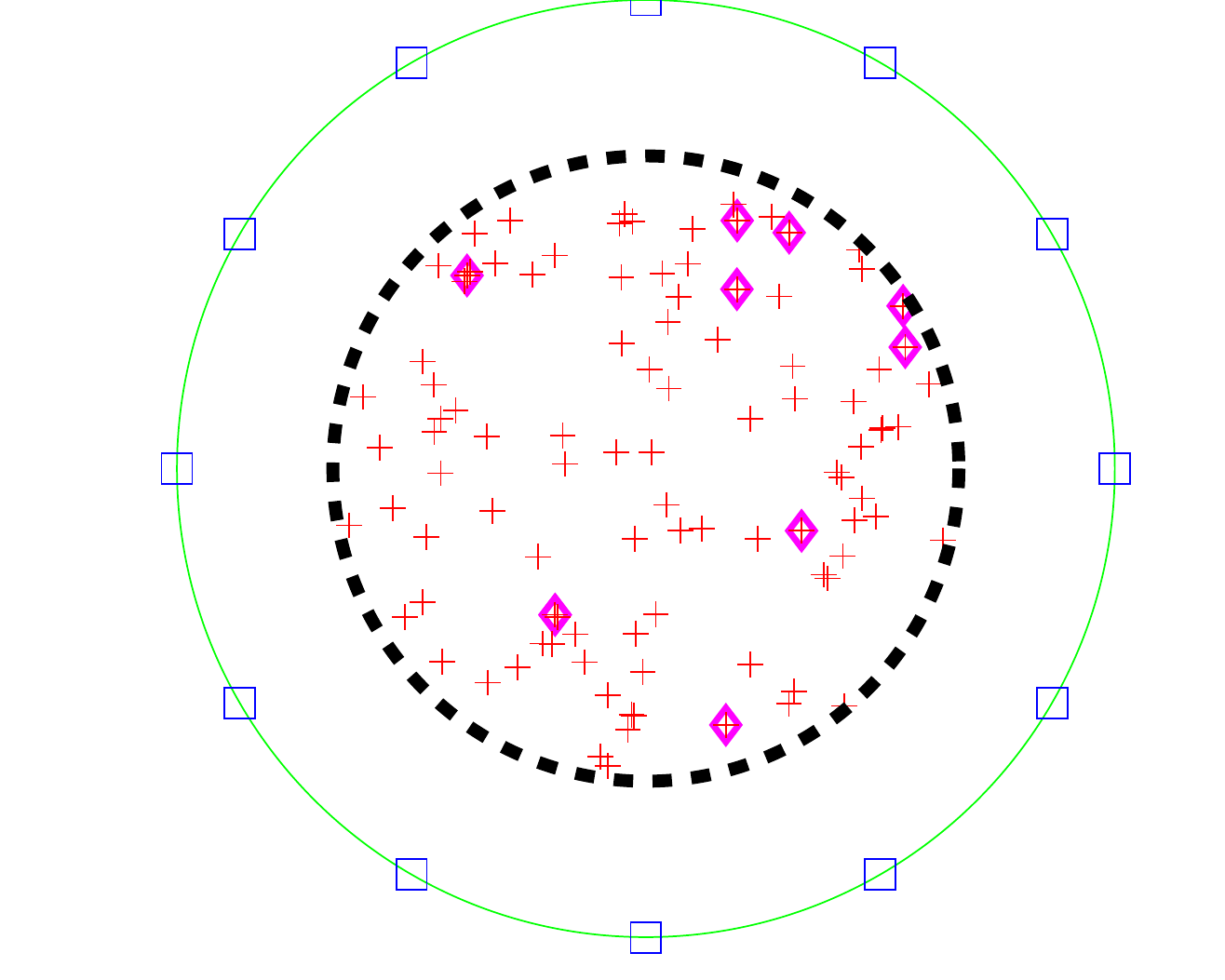}}
\caption{We illustrate three methods for computing an outgoing
        representation of the red source points. In
        Figure~\ref{fig_series}, we illustrate an analytic, single
        term expansion: the points are represented by their centroid.
        Higher order approximations can be viewed as Taylor expansions
        around this point and require a number of terms that grows significantly
        with the dimension $d$. In Figure~\ref{fig_interpolation}, we
        show a method based on placing equivalent sources and finding
        equivalent densities that can approximate the far
        field (Equation~\ref{e:kifmm}).  An outgoing representation is
        constructed so that the far field due to the true sources (red
        points) is reproduced by equivalent sources (magenta ``X'').
		The charges on the equivalent sources are determined from interactions 
		with fictitious check points (blue squares). As
        the dimension of the input increases, the number of equivalent
        sources and check points required grows quickly, since they must cover 
		the surface of a bounding sphere or cube in $d$ dimensions. In
        Figure~\ref{fig_skeletonization}, we illustrate the
        skeletonization-based approach.  
        Using the interactions between the sources and fictitious targets (blue
        squares), the method computes
        an interpolative
        decomposition and chooses some skeleton sources (magenta points)
        to represent the far field. The number of skeleton points needed
        depends on the local intrinsic dimensionality and the
        kernel. However, with existing techniques, the number of targets
        needed can grow with the ambient dimension.
\label{fig_outgoing_rep_illustrations}}
\end{figure}


\subsubsection{Semi-analytic methods} 
This class of methods approximates the potential due to a collection
of sources using additional fictitious source points, which,
following~\cite{ying2004kernel}, we term \emph{equivalent sources}.  We
focus our discussion on the KIFMM~\cite{ying2004kernel}, while noting
that a similar method has been applied in other
settings \cite{anderson1992implementation,
berman1995grid,messner-darve-e12}.  The idea is to find $p \ll n$
equivalent sources $\{\tilde{x}_j\}_{j=1}^p$ and \emph{equivalent
densities} $\{\tilde{q}_j\}_{j=1}^p$ such that Equation~\ref{e:farfield}
becomes
\begin{equation}\label{e:kifmm}
\aKer_S(y) =  \sum_{j=1}^p \Ker(y,\tilde{x}_j) \tilde{q}_j
\end{equation}
In the KIFMM, the equivalent sources are placed on a convex surface
(typically either a sphere or a cube) surrounding the true sources.
The positions correspond to surface quadrature rules, and their number
$p$ scales as $c^{d-1}$, where $c > 1$ depends on the target
accuracy. 

To obtain the equivalent densities $\tilde{q}_j$, we solve a least-squares 
problem that minimizes the mismatch between the far field of
the equivalent sources and the far field of the true sources at a set
of target points. In the KIFMM, these are referred to as the \emph{check points}
and they are also fictitious. The check points are placed on a convex
surface surrounding both true sources and equivalent sources.  In the
KIFMM, their positions correspond to surface quadrature rules and their
number scales as $\bigO(p)$.

The advantage of semi-analytic methods is that they only require
kernel evaluations and fairly general assumptions about the kernel function
(e.g. that the far-field decays and that the kernel is a Green's
function).  We call them semi-analytic, because the positions of
equivalent sources and check points are chosen using arguments from
analysis. Once these positions are chosen, we no longer require any information
about the kernel other than how to evaluate it. 
While semi-analytic methods are effective in low
dimensions, they share the same scalability issues with analytic
methods: the number of equivalent sources scales poorly with
increasing dimension. By sacrificing accuracy, sparse grids that scale
as $p=\bigO(d^{c-1})$ could be used to push these techniques to higher
$d$, but for large $d$ and $c > 1$,
this approach also becomes too costly.


\subsubsection{Algebraic approximations}

Both analytic and semi-analytic approximations make use of analytical 
properties of the kernel function. On the other hand, algebraic approximations 
work directly with the kernel matrix-vector product. They use  
methods from linear algebra to construct the outgoing representation.

Recall that Equation~\ref{e:sum} can be viewed as the product of an 
$m \times n$ 
matrix $K$ with an $n$-vector $q$. We (conceptually) construct an 
approximate matrix $\aK$ such that the product $\aK q$ can
be efficiently computed. 
One common construction uses the 
\emph{truncated singular value
decomposition} \cite{stewart1993early}:
\begin{definition}
\textbf{Truncated Singular Value Decomposition.} For any matrix $K \in 
\reals^{m \times n}$, its singular value decomposition consists of orthonormal 
matrices $U \in \reals^{m \times m}$ and $V 
\in \reals^{n \times n}$ and a diagonal matrix $\Sigma$ such that
\begin{equation}
K = U \Sigma V  
\end{equation}
and
$\Sigma$ has non-negative 
entries $\sigma_1, \ldots, \sigma_n$ such that $\sigma_i > \sigma_{i+1}$ for 
all $i$. The columns of $U$
($V$) are referred to as the left (right) singular vectors, and
the $\sigma_i$ are the singular values.
\end{definition}

For a given rank $r$, the truncated SVD consists of the first $r$
columns of $U$ (denoted $U_r$) and $V$ ($V_r$) along with the first $r$ 
singular values ($\Sigma_r$).
Furthermore, it provides the following error guarantees, which are optimal 
among any rank $r$ approximation:
\begin{equation}
\left\|K - U_r \Sigma_r V^T_r \right\|_2 = \sigma_{r+1},
\quad \quad
\left\|K - U_r \Sigma_r V^T_r \right\|_F = \sum_{k = r+1}^n \sigma_k
\end{equation}
If $K$ has rank $r \ll (m+n)$, then we can compute $U_r \Sigma_r V^T_r
q$ in $\bigO(r (m + n))$ time.

Another possible decomposition is the Interpolative Decomposition
(ID), utilized in the context of the FMM by Martinsson and
Rokhlin~\cite{martinsson2007accelerated}.
\begin{definition}
\textbf{Interpolative Decomposition.} Given a $m \times n$ matrix $K$,
the rank $r$ interpolative decomposition consists of matrices
$C \in \mathbb{R}^{m \times r}$ and $P \in \mathbb{R}^{r \times n}$
such that 
\begin{equation}
K \approx C P
\end{equation}
and 
\begin{enumerate}
\item The columns of $C$ are a subset of the columns of $K$
\item $P$ has the $r \times r$ identity matrix as a submatrix.
\end{enumerate}
We refer to the
column indices of $K$ chosen to make up $C$ as the \emph{skeleton} and
$P$ as the \emph{projection matrix}. 
\end{definition}
Note that some definitions differ slightly in the literature. 

The ID can be computed by a rank-revealing QR
factorization \cite{gu1996efficient}. 
\begin{theorem}
(\cite{cheng2005compression}.) We can form a rank $r$ interpolative 
decomposition $C P$ of an $m \times n$ matrix $K$ such that
\begin{equation}
\|K - CP \|_2 \leq \sqrt{1 + n r (n - r)} \sigma_{r+1}(K).
\label{eqn_id_error}
\end{equation}
\label{thm_compute_id}
\end{theorem}

The ID can be used to form an outgoing representation 
\cite{martinsson2007accelerated}. Since $C$ is
a subset of the columns of $K$, $C_{ij} = \Ker(y_i,\tilde{x}_j)$ where
$\tilde{x}_j$ is one of the $r$ skeleton points. Given the original source
charges $q_j$, we compute equivalent skeleton charges by
$\tilde{q} = P q$ where 
$\tilde{q} \in \mathbb{R}^r$.
Then, the potential $u(y_i)$ at any source $y_i$ in $T$ due to the
charges in $S$ can be recovered as
\begin{equation}\label{e:id}
u_i = K q \approx PCq = C \tilde{q}
= \sum_{j=1}^r \Ker(y_i, \tilde{x}_j) \tilde{q}_j. 
\end{equation}
The representation takes $\bigO(n r)$ work to compute the equivalent charges 
and $\bigO(r)$ kernel computations between the target and skeleton points. 
The approximation error satisfies
\begin{equation}\label{e:id-error}
|u_i - \tilde{u_i} | < \bigO\left(\sqrt{1 + n r (n - r)} \sigma_{r+1}(K)\right). 
\end{equation}
If $K$ is numerically rank $r$, then this error term will be
negligible.  
 
The method sketched here has the advantage that it does not 
require any prior knowledge of the analytic structure of the kernel.  As long 
as we are able to partition sources and targets so that the matrix $K$ is 
numerically low rank, this scheme will work.

However, any method based on the SVD or ID will have to overcome the 
high cost of computing the decomposition. A direct SVD or QR factorization of 
$K$ will require $\bigO(m n^2)$ work, which is greater than the direct 
evaluation of the kernel summation. Although more efficient algorithms can 
compute the factorization in $\bigO(m r^2)$ time, this is still too expensive
for use as the basis for an outgoing representation. Therefore, algebraic 
methods require a smaller matrix that does not depend on $m$.

%
%

Note that Equation~\ref{e:id} resembles Equation~\ref{e:kifmm}. 
In some sense, the
skeleton points correspond to the equivalent sources of the
KIFMM and $P q = \tilde{q}$ corresponds to the equivalent densities. 
These methods differ in the way the equivalent source positions are chosen and 
the way the equivalent densities are computed. Rather than constructing an ID 
of the entire matrix $K$, ID-based approaches construct a smaller matrix $\subK$
using some carefully chosen fictitious targets, similar to the check points 
used in the
KIFMM \cite{martinsson2007accelerated}. That is, we place $s \ll m$ 
fictitious targets on a surface that encloses the source region.
Then we form the dense $s \times n$ interaction matrix $\subK$ with these 
fictitious targets and compute its ID, from which we extract $r$ skeleton
points and, using $P$, compute their equivalent densities (see Figure~\ref{fig_skeletonization}). 
Existing ID-based outgoing representations successfully use this method
\cite{martinsson2007accelerated}.

SVD and ID algebraic decompositions have been successfully
demonstrated in one~\cite{yarvin1999improved,
martinsson2007accelerated} and two~\cite{gimbutas2003generalized}
dimensions for a variety of kernel functions.  
As we mentioned, existing approaches
suffer from the same problem in higher dimensions 
as the KIFMM: the number of interpolation points needed scales exponentially 
with $d$.

\subsection{Our approach}

\begin{figure}
  \centering
  \subfigure[\textbf{Fictitious targets.}\label{fig_fict_targets_id}]{\includegraphics[width=0.4\textwidth]{figures/illustrations/algebraic_interpolation.pdf}}
\subfigure[\textbf{Subsampling (new).}\label{fig_subsampling_id}]{\includegraphics[width=0.4\textwidth]{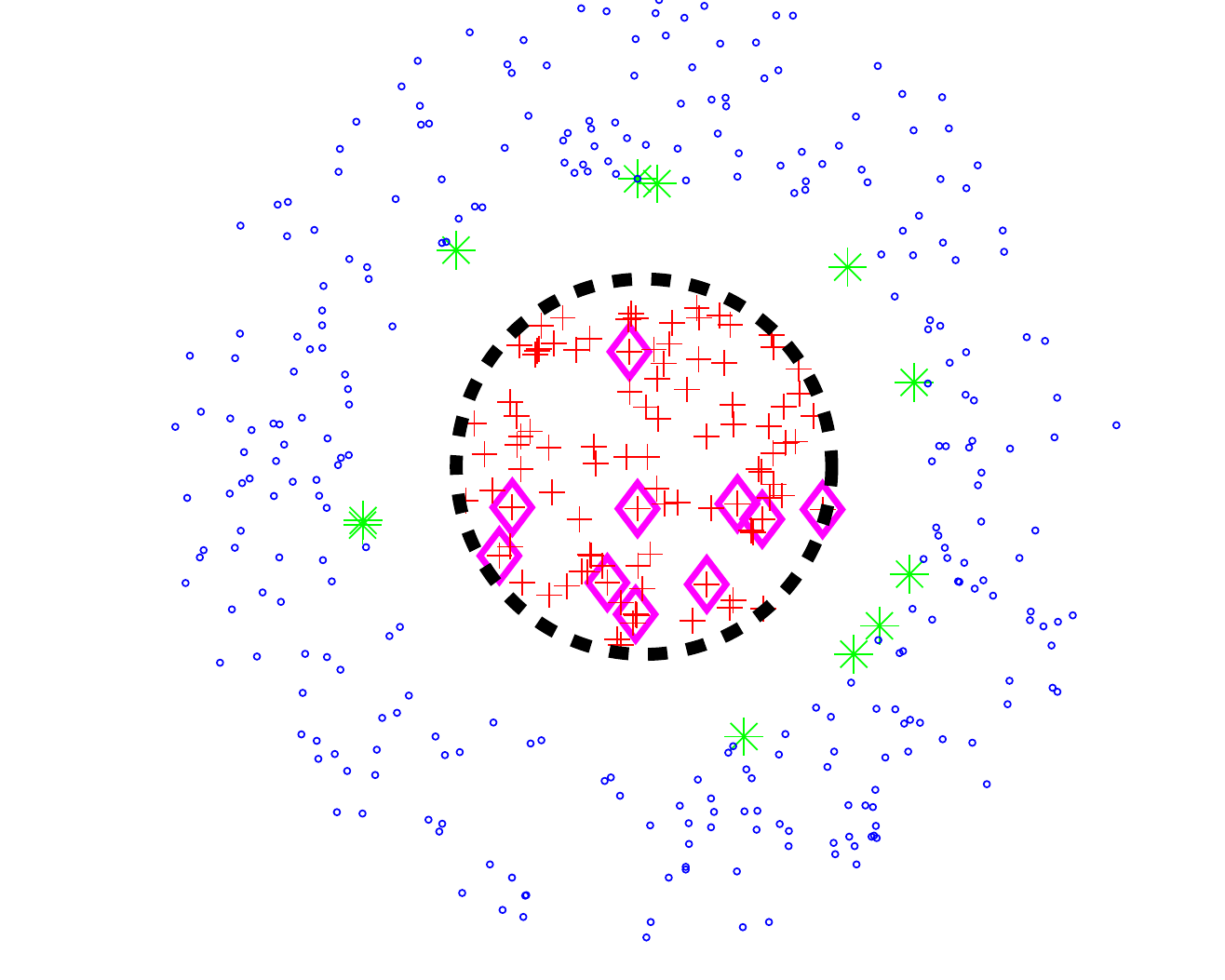}}
   \caption{\textbf{Two approaches for computing an algebraic outgoing representation.} In both cases, we are interested in computing an outgoing 
   representation of the red source points. One method 
   \cite{martinsson2007accelerated} (Figure 
   \ref{fig_fict_targets_id}), places
   a set of fictitious targets covering a ball or box surrounding the sources. 
   It computes the matrix of interactions between targets and sources, and 
   computes its ID. 
  Our approach (Figure \ref{fig_subsampling_id}) subsamples $s$ of the 
  well-separated target points (shown in green). We then compute the ID of the 
  resulting $s \times n$ matrix $\subK$. 
  \label{fig_algebraic_sampling}}
\end{figure}

We propose an alternative approach to the formation of interpolative
decompositions for outgoing representations.  
Rather than choosing fictitious target points according to 
quadratures chosen via an analytic approach, we choose a sub-sample of the 
target points themselves. 
We pick $s$ points, either randomly or deterministically, from the entire 
set of targets $T$. We use these target points to construct the subsampled 
matrix $\subK \in \reals^{s \times n}$.  We then construct an ID of this matrix 
and use it to obtain skeleton points and compute effective charges. 

This approach has several possible advantages over existing methods. Methods 
using interpolation are not able to take the intrinsic dimensionality of 
the points into 
account when choosing interpolation points. This leads to their poor scaling
with dimension, even if the data set occupies a (probably unknown) 
low-dimensional manifold. Since our method chooses points from the data
set, this potentially allows our scheme to achieve accuracy with a number of 
samples that does not depend on $d$, but only on the intrinsic dimension
of the set.

Furthermore, there are situations where creating new points is not 
straightforward. For instance, kernel-based learning methods are frequently 
applied to non-metric data such as strings, graphs, or documents. 
It is not always clear how to create a new string or graph
for the purposes of interpolation.

We now explore the possibility of using a sample of the target points in order 
to compute interpolative decompositions for outgoing representations. 
We examine the error due to using a 
subsample of targets both theoretically and experimentally.




%
%


\section{Theoretical results\label{sec_theory}}

Given an $m \times n$ matrix $K$, we will sample $s$ of its rows to
form $\subK$. Then, we compute a low-rank decomposition (such as the ID) 
of $\subK$ and use it to approximate the original matrix. 
This leaves several outstanding questions. We need to determine a sampling 
distribution over rows that is efficiently computable. 
Given this, we must understand the approximation error
due to the sampling and the number of samples needed to achieve a given 
error.

\textbf{A problem case.}
Clearly, the most straightforward approach is to sample 
rows uniformly at random.  However, for arbitrary matrices, this will not 
work. Consider a rank two matrix $K$ in which the first row of $K$ is 
$v_1^T$ and the other $m-1$ rows are copies of $v_2^T$ for some
orthogonal vectors $v_1$ and $v_2$.
The row space will be spanned
by $v_1$ and $v_2$.  However, if we sample less than $m$ rows of $K$,
we will likely capture only the part of the row space spanned by $v_2$.
This example shows that in general, it is not possible to succeed by 
uniformly sampling a small number of rows. Instead, we will either require a 
more sophisticated method of choosing rows or some restriction on the 
input rows to eliminate difficult cases like the one above.

We begin by discussing a way to formalize the ``difficulty'' of a matrix, like 
in our example.
We require a measure of the contribution of each row
to the total action of the matrix. In the example above, the first row $v_1^T$
is more significant to the row space of the matrix than any other single row. 
There are several ways to measure how ``concentrated''
the row space of a matrix is along a few of its rows or columns. 
Following previous work, we use the concepts of statistical
leverage scores and concentration.

\textbf{Sampling columns vs.~rows.} Most results in the literature on 
subsampling from matrices deal with sampling 
columns instead of rows. Clearly, sampling rows of $K$ is equivalent to 
sampling columns from $K^T$. In order to follow the results from the literature 
more closely, we switch to the consideration of columns
for this discussion. We discuss a general 
matrix $A \in \reals^{n \times m}$ with $m \geq n$ 
and discuss the construction of a subsampled
or projected matrix $\subA \in \reals^{n \times s}$. This can be thought of as 
the transpose of the matrix $K$ from the remainder of the paper.

\begin{definition}
\textbf{Statistical Leverage \cite{donoho2001uncertainty}.} Let 
$A \in \mathbb{R}^{n \times m}$ be any matrix
with $m > n$. Fix a rank $r \leq n$. 
Let $A = U \Sigma V^T$ be the singular value decomposition of $A$ and let 
$U_r \Sigma_r V_r^T$ be the optimal rank $r$ approximation of $A$ formed 
from this decomposition. 
Then, the \emph{statistical leverage
scores} of $A$ \emph{with respect to the best rank $r$ approximation} are 
given by 
\begin{equation}
  \ell_j^{(r)} = \|V_r^T e_j\|_2^2
\end{equation}  
where $e_j$ is a standard basis vector. In other words, the scores are the 
squared norms of rows of $V_r$. 

The \emph{coherence} of $A$ with respect to the rank $r$ is given by 
\begin{equation}
  \gamma^{(r)} = \max_j \ell_j^{(r)}
  \label{e:gamma_def}
\end{equation}
\end{definition}

Note that the coherence with respect to rank $r$ is bounded by 
\begin{equation}
\frac{r}{m} \leq \coherence \leq 1
\end{equation}

%

These definitions attempt to formalize the concern raised in our 
example above: a small number of rows may have a disproportionate effect 
on the row space. This in turn can increase the number of 
samples required to achieve a given accuracy. Existing methods center around 
two main approaches to overcome this obstacle. One approach uses
some pre-processing of the matrix to make the leverage scores more 
uniform or reduce the concentration 
before constructing the smaller matrix $\subK$.  
The other approach constructs an importance sampling distribution which samples 
rows with probability proportional to their norm or leverage score. This will 
preferentially select ``difficult'' rows like $v_1$ in our example. 
Next, we briefly review some of the main results regarding these two
approaches.

\subsection{Sampling strategies and main results}

We examine two successful strategies for constructing submatrices: random 
projections and importance sampling distributions.

\textbf{Random projections.}
Rather than directly sampling columns of
 $A$, these methods project the matrix $A$ onto some smaller space. 
A typical result for random projections is
from~\cite{halko2011finding} (pp 226). 
\begin{theorem} 
Let $r$ be the target rank, and choose an
oversampling parameter $s = r + \ell$ for some $\ell \geq 4$.  Let $\Omega$
be an $m \times (r + \ell)$ matrix with iid Gaussian entries and let $C =
A \Omega$. Then, with probability at least $1 - 3\ell^{-\ell}$:
\begin{equation} 
 \|(I - \Pi) A\|_2 \leq (1+9\sqrt{(r + \ell)m}) \sigma_{r+1}
\end{equation}
where $\Pi$ projects onto the span of $C$.
\end{theorem}

The problem with this approach is that computing $C$ costs $\bigO(m n
r)$ work. The complexity can be improved to $\bigO(m n \log r)$ using a more
sophisticated $\Omega$ \cite{tropp2011improved}. In either case, the cost 
exceeds the cost of applying $A$ to a vector, so it cannot be used in our 
context.

\textbf{Importance sampling.}
The other major approach considers a more sophisticated way 
to choose rows. 
We begin with the gold standard for sampling rows or columns directly
from a matrix: using an importance distribution based on leverage
scores. The following result is from~\cite{mahoney2009cur}.
\begin{theorem}
Let $\epsilon > 0$ and $s = \bigO(r \log r / \epsilon^2)$.  Draw $s$
columns from an importance sampling distribution where the
 probability of choosing a column is proportional to its leverage
 score. Then, with high probability,
\begin{equation} \|A - \Pi A\|_F \leq (1 + \epsilon/2) \sum_{j = r+1}^m \sigma_j(A)
\end{equation}
where $\Pi$ is the projection of $A$ onto the space spanned by the selected
columns. 
\end{theorem}

This is only a factor of $(1 + \epsilon)$ worse than the optimal rank
$r$ approximation obtained from the SVD. A
similar result exists for the spectral norm \cite{boutsidis2009improved}.  
Unfortunately, computing the leverage scores requires a basis for the left 
singular vectors of $A$.  Computing this will in turn require $\bigO(n^2 m)$
work and requires access to the entire matrix $A$.\footnote{Randomized methods
can approximate these scores, but they still require $\bigO(r^2 m)$ work and
the ability to compute the product of $A$ with a vector.} 

Other sampling-based approaches utilize simpler importance 
distributions. For instance, a result due to Frieze~\emph{et al.} 
\cite{frieze2004fast} samples columns with a probability 
proportional to their Euclidean norm. 
\begin{theorem}
Sample $s$ columns of $A$ with probability proportional to their Euclidean norms
with replacement. Let $\Pi$ be the projection onto the best rank-$r$ subspace 
of the sampled columns. Then, with high probability:
	\begin{equation}
		\|A - \Pi A \|_F^2 \leq \sum_{k=r+1}^n \sigma_{k}^2(A) + \frac{10 r}{s} \|A\|_F^2.
	\end{equation}
  \label{thm_frieze_bound}
\end{theorem}
An importance sampling distribution based on row norms is easier to compute 
than one based on leverage scores.  Later work improves this result with a more 
sophisticated sampling distribution \cite{deshpande2006matrix}. Either 
approach still requires access to the entire matrix and is thus not practical
in our context.

To summarize, the two basic approaches for approximating $A$ do
not work in our context because we
require a scheme that is cheaper than a matrix vector multiplication -- 
\emph{i.e.}~cheaper than $\bigO(mn)$. 

\textbf{Uniform sampling.} One solution is to sample from a predetermined 
distribution, such as
the uniform distribution. These results follow the intuition discussed above: 
if $A$ has low concentration and low rank, then all columns make a roughly
equal contribution to its range. In this case, we intuitively expect that
uniform sampling will work quite well.

Uniform sampling approaches utilize a number of samples that grows with the 
concentration of the matrix $A$.
Previous results have focussed on Nystrom extensions -- 
sampling columns to approximate a positive semi-definite matrix. For instance, 
in \cite{gittens2011spectral}, the authors show that uniform sampling of
$s = \bigO(\mu r \log r)$ columns of 
a positive semi-definite matrix $A$ can 
provide spectral norm error bounded by $\sigma_{r+1} (1 + 2 m/s)$.
Compressed sensing approaches also consider uniform sampling.
In \cite{talwalkar2010matrix}, the authors show that in the case of a matrix
of exactly rank $r$, uniform sampling of a matrix with low coherence results 
in exact reconstruction with high probability. 

We extend these results in several ways. First, we show a spectral norm error 
bound for uniform sampling from a general, rectangular matrix $A$, rather than 
a PSD matrix. Second, our result holds for matrices which are not exactly of 
rank $r$. Third, we improve the result in \cite{gittens2011spectral} by a 
factor of $\sqrt{m/s}$.

\begin{theorem}
  Let $A \in \mathbb{R}^{n \times m}$ be any matrix
  with $m > n$. Let $\subA \in \mathbb{R}^{n \times s}$ be the submatrix
  of $A$ obtained by randomly sampling $s$ columns of $A$ uniformly without
  replacement. Let $\Pi$ be an orthogonal projector onto the space
  spanned by the columns of $\subA$. Let $r$ be a targeted approximation rank.

  Then, for any $\delta \in (0, 1)$ and any $\epsilon \in (0,1)$, the 
  following holds. 
  If 
  \begin{equation}
  \ns \geq m \coherence \log\left( \frac{2r}{\delta} \right) \left[ \log\left(  \frac{(1+\epsilon)^{1+\epsilon}}{e^\epsilon} \right) \right]^{-1}
  \label{eqn_sample_complexity}
  \end{equation}
  then with probability at least $(1 - \delta)$,
  \begin{equation}
  \|(I - \Pi) A \|_2 \leq \sqrt{1 +  \frac{m}{s} (1 + \epsilon) (1 - \epsilon)^{-2}}\, \sigma_{r+1}(A)
  \label{eqn_reconstruction_error}
  \end{equation}
  \label{thm_column_sampling_approx2}
\end{theorem}


We briefly sketch the proof, then provide details in the appendix. 
We introduce a matrix $\Omega$ which carries 
out the sampling -- \emph{i.e.}~$A \Omega = \subA$. We then use two results from
the literature. First, we apply a deterministic 
bound on the quantity $\|(I - \Pi)A\|$ in terms of the singular values of 
$V_r^T \Omega$. 
Second, we bound these singular values (with high probability) using a matrix 
Chernoff inequality. We first prove the theorem for sampling \emph{with} 
replacement using the inequality. We then apply a result due to 
\cite{gross2010note} to show that sampling without replacement does not do 
worse. 

Note that by plugging
in $\epsilon = 1/2$, we get:
\begin{equation}
\left\|(I - \Pi) A \right\|^2  \leq \left(1 + 6 \frac{m}{s} \right) \sigma_{r+1}^2
\end{equation}
as long as $s \geq 10 m \coherence \log(2r/\delta)$.
We use this result where convenient.

\textbf{The error in an outgoing representation.}
Theorem~\ref{thm_column_sampling_approx2} tells us that given a
well-behaved matrix (in the sense of low concentration) then with high
probability our sampled matrix captures the action of the original
matrix. 
We now show how the above error guarantees fit into the overall framework of 
this paper. We bound the total error in our final quantity of interest, the 
matrix-vector product $K q$.

\begin{theorem}
Let $K$ be the $m \times n$ matrix of interactions between sources and targets. 
Sample $s$ columns of $K^T$ under the conditions of 
Theorem \ref{thm_column_sampling_approx2}, and construct a rank $r$ 
interpolative decomposition $\subaK$
of the subsampled matrix $\subK$.
Then, the total error incurred is bounded by:
\begin{equation}
\|K q - \subaK q\| \leq \|q \| \left[ \sigma_{r+1}(K) 
	+ \left( 1 + 
  \left(1 + \epsilon \right) \left(1 - \epsilon \right)^2 \frac{m}{s} \right)^{\frac{1}{2}} \sigma_{r+1}(K)
	+ \left( 1 + n r (n - r) \right)^\frac{1}{2} \sigma_{r+1}(\subK) \right].
\label{eqn_full_error_bound}
\end{equation}	
\label{thm_full_error_bound}
\end{theorem}

\begin{proof}
(\emph{Sketch.}) The proof follows from inserting the best rank $r$ 
approximation of $\subK$, then applying the triangle inequality followed by
a bound on the error due to a subsampled ID \cite{halko2011finding} and 
Theorem~\ref{thm_compute_id}. The full proof is given in the appendix
(section~\ref{sec_appendix}).

%
%
\end{proof}

Note that, while this 
result depends on the singular values of the subsampled matrix $\subK$, these 
are only scaled by $n$. Since we are interested in the case where $m \gg n$, 
this will not dominate the error. Also, since $\subK$ is a submatrix of $K$, 
we know that $\sigma_{r+1}(\subK) \leq \sigma_{r+1}(K)$.

\subsection{Heuristic improvement using geometric information}

The results we have discussed so far hold for any matrix. However, our 
goal is to construct outgoing representations for treecodes. This restriction
provides additional structure which our sampling method can use.
For instance, the data points are typically points in a metric space,
and the kernel function commonly decays with increasing distance between its 
arguments.

In this case, the rows with the largest norm will correspond to the targets
closest to the set $S$. This suggests a heuristic to approximate the Euclidean
norm sampling distribution (Theorem \ref{thm_frieze_bound} \cite{frieze2004fast}).
We can sample targets (i.e. rows) with probability inversely proportional to 
their distance from the source set. 
Note that these distances can be efficiently approximated, for instance with 
the tree structure used in the treecode.
We also expect the largest entries of the matrix to approximate the leverage
scores.
Additionally, we can use nearest-neighbor information to construct an 
approximate importance sampling distribution -- \emph{i.e.}~by deterministically
choosing the targets closest to the source set.

We consider both of these sampling distributions in the
remainder of the paper: sampling from probabilities inversely proportional to 
the distance from the sources, and choosing nearest neighbors deterministically.
While we leave a theoretical analysis of these heuristics to future work, 
in the next section, we explore their performance empirically.

\section{Experimental results\label{sec_experiment}}

In this section, we conduct numerical experiments to demonstrate the
effectiveness of our scheme for several different kernels.

We focus on constructing an outgoing representation for $n$ source
points.  In order for our sampling-based approach to computing
outgoing representations to work, we require two things: first, the
$m \times n$ kernel submatrix $K$ representing the interactions
between the sources and all distant target points needs to be
numerically low rank. If this is not the case, then we will not be
able to construct a cost-effective outgoing representation.  Second,
we require that we can compute an approximation from a few subsampled
rows of $K$.

We wish to investigate the following questions with our experiments:
\begin{itemize}
\item First, how well can we compress the interactions 
due to distant targets in high dimensions? We 
explore the numerical rank of $K$ for a range of kernel parameters and 
properties of the input points (\emph{i.e.}~dimensionality, spatial 
distribution). These
experiments are used to determine the feasibility of using a low-rank
approximation of the far field.

\item Second, how well do different sampling schemes do in
capturing this low-rank far-field approximation? We find that the 
nearest-neighbor sampling works almost as well
as leverage score sampling for kernels which are functions of distances and, 
when combined with the ID,
results in a compression nearly as good as that obtained using an SVD of 
the full matrix $K$.
\end{itemize}
Unlike existing approaches, the effectiveness of the compression of our method
depends only on the intrinsic dimensionality of the dataset and not
the ambient dimension. We provide some examples that demonstrate this
property of our scheme. 

Next, we detail the experimental setup, then discuss each of the Gaussian, 
Laplace, and polynomial kernels. 
For
each kernel, we first discuss the kernel function and our choices of
parameters.  We then explore the numerical rank of kernel submatrices
and test our ability to compute outgoing representations using row
sampling.

\subsection{Experimental Setup}

Throughout, we consider the interactions between a compact set of
sources and a set of distant targets.  We examine the construction of
an outgoing representation to compactly capture the potential due to
these sources at a distant target. We now describe our basic
experimental setup and illustrate it in
Figure~\ref{fig_experiment_illustration}.


\begin{enumerate}
\item Sample $N$ points from a $d$-dimensional distribution.  
\item Choose a center point $x_c$ (typically the origin). Let the source 
set $S$ be the $n$ points closest to $x_c$. Define $\rho$ to be the maximum 
distance between $x_c$ and any source point.
\item Define the target set $T = \{y : \|y - x_c\| \geq \xi \rho \}$.  
The parameter $\xi$ controls the separation between sources and targets. 
Call the number of targets $m$.
\item Compute the $m \times n$ matrix $K$ of all interactions between sources 
and targets. 
\end{enumerate}

\begin{figure}[tp]
\begin{center}
\includegraphics[width=0.5\textwidth]{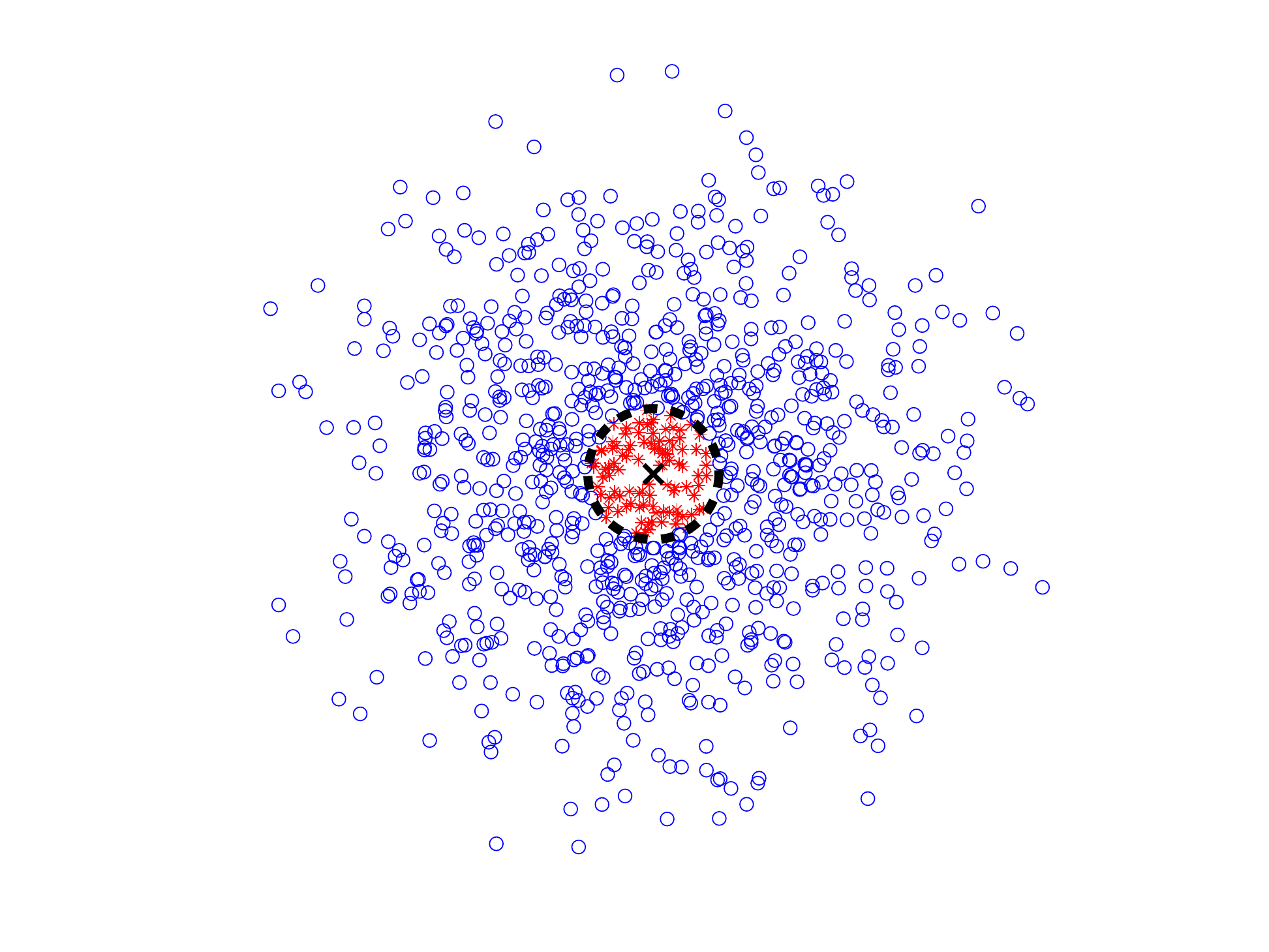}
\caption{\textbf{Experimental setup.} We illustrate our basic experimental 
setup.  
The point $x_c$ is marked with the black ``x''. The $n$ closest points are 
chosen as sources and colored red. We draw a ball of radius $\xi \rho$ 
in black around $x_c$. The points at least $\xi \rho$ from $x_c$ 
are chosen to be targets and marked blue.  Here, we show $\xi = 1$. The matrix 
$K$ will consist of all pairwise interactions between blue targets and red 
sources. 
\label{fig_experiment_illustration}}
\end{center}
\end{figure}
For each set of experiments, the main parameters are:
the choice of input distribution, especially its dimensionality $d$;
the total number of points $N$ drawn from the distribution, along with the center of the sources $x_c$ and the number of source points $n$;
the ``well-separatedness'' parameter $\xi$;
and any parameters for the kernel function, such as a bandwidth $h$.

\textbf{Choice of $d$.} For small values of $d$ (1, 2, 3), existing kernel 
summation algorithms are efficient and accurate. We are primarily interested in 
higher dimensions where these methods fail.  
We show results from four to 64 dimensions for normally 
distributed data and construct low-intrinsic dimensional data sets with ambient 
dimension as large as 1,000.  We also use real data sets with tens of 
dimensions. 

\textbf{Choice of $n$.}  We fix $n = 500$ throughout our experiments.  
Intuitively, we expect that the interaction between sources and
targets has some ``true'' rank, for given locations of points. If this
is the case, as we increase $n$, we should see better and better
compression.  However, this is not a viable strategy in the context of
fast kernel summation methods, since we will still need to compute
direct interactions between sources and themselves. Therefore, we
choose $n = 500$ as an intermediate value, \emph{i.e.}~one that is large
enough for us to see some compression, but small enough so that the
direct interactions between $n$ points are efficiently computable.

\textbf{Choice of $\xi$.} We fix $\xi = 1$ unless otherwise noted. If
$\xi < 1$, then some of the sources are included in the target set.
These self-interactions would be computed directly in a treecode or
FMM, so we are not interested in a compact representation of them. On
the other hand, for large values of $\xi$, most of the $N$ points will no
longer be included in the target set, particularly in high
dimensions. Therefore, we choose $\xi = 1$ as a compromise value for
the Gaussian and polynomial kernels. For the Laplace kernel, we
explore $\xi = 2$.

\subsubsection{Data sets}

We use the following data sets in our experiments:
\begin{itemize}
\item \textbf{Normal.} These data are drawn independently from the
standard multivariate normal distribution in $d$ dimensions. 
These experiments represent a worst-case example where the data truly 
fill out the ambient space. The optimal bandwidth for the kernel density 
estimation task with the Gaussian kernel can also be computed exactly for this 
data set \cite{silverman1986density}, giving us a 
starting point for bandwidth selection in our experiments.
\item \textbf{Low intrinsic dimension.} We draw data from the standard
multivariate normal distribution in $d_i$ dimensions. We pad these
data with zeros so that they live in $d_e$ dimensions, with $d_e \gg
d_i$.  We then apply a random rotation and add small uniform
noise. This artificial example allows us to directly examine if our
approach can successfully capture low dimensional structure in the
data.  
\item \textbf{Real data.} We also use the Color Histogram and Co-occurrence
Texture features sets from the Corel Images data in the UCI ML repository
\cite{Bache+Lichman:2013}.
The properties of these sets are given
in Table~\ref{table_real_data_h_values}. We translate and scale each
set so that it is contained in the unit hypercube.
\end{itemize}


\subsubsection{Distances in high dimensions}

Before we proceed, we take a moment to discuss the consequences of
increasing the dimension of the data set. 
For data in high dimensions, the pairwise distances between points will tend to 
converge around a single distance -- this is an example of the 
\emph{concentration of measure} effect \cite{verleysen2001learning}.
We plot histograms of the 
pairwise distances between points in the source and target sets for the 
$d$-dimensional standard normal distribution in 
Figure~\ref{fig_pairwise_dist_histograms}. These plots illustrate that the 
pairwise distance distributions become increasingly peaked as $d$ 
increases. 

This effect is significant for our choice of the parameter $\xi$ in the 
experiments. Recall that we identify the $n$ points closest to $x_c$ as 
the sources, and call $\rho$ the largest distance from $x_c$ to a source.
The target set then consists of all points at a distance of at least
$\xi \rho$ from $x_c$. As $d$ increases and the distribution of pairwise
distances becomes more peaked, a small increase in $\xi$ can lead to a very 
large fraction of the points being excluded from the target set.
This observation informs our choices of $\xi$ in the experiments.

\begin{figure}[tp]
        \centering
        \subfigure[$d = 2$]{\includegraphics[width=0.3\textwidth]{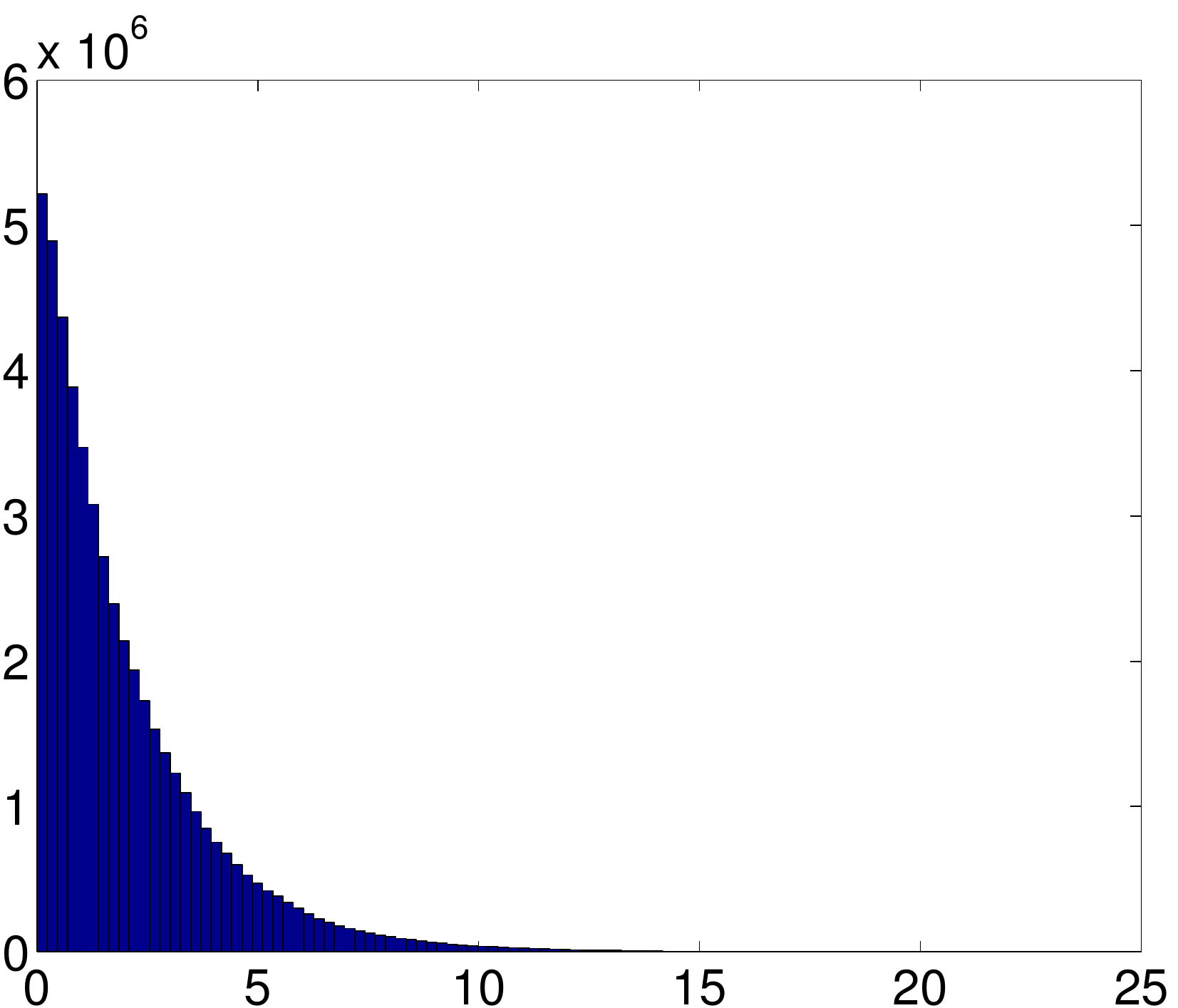}}
        \subfigure[$d = 8$]{\includegraphics[width=0.3\textwidth]{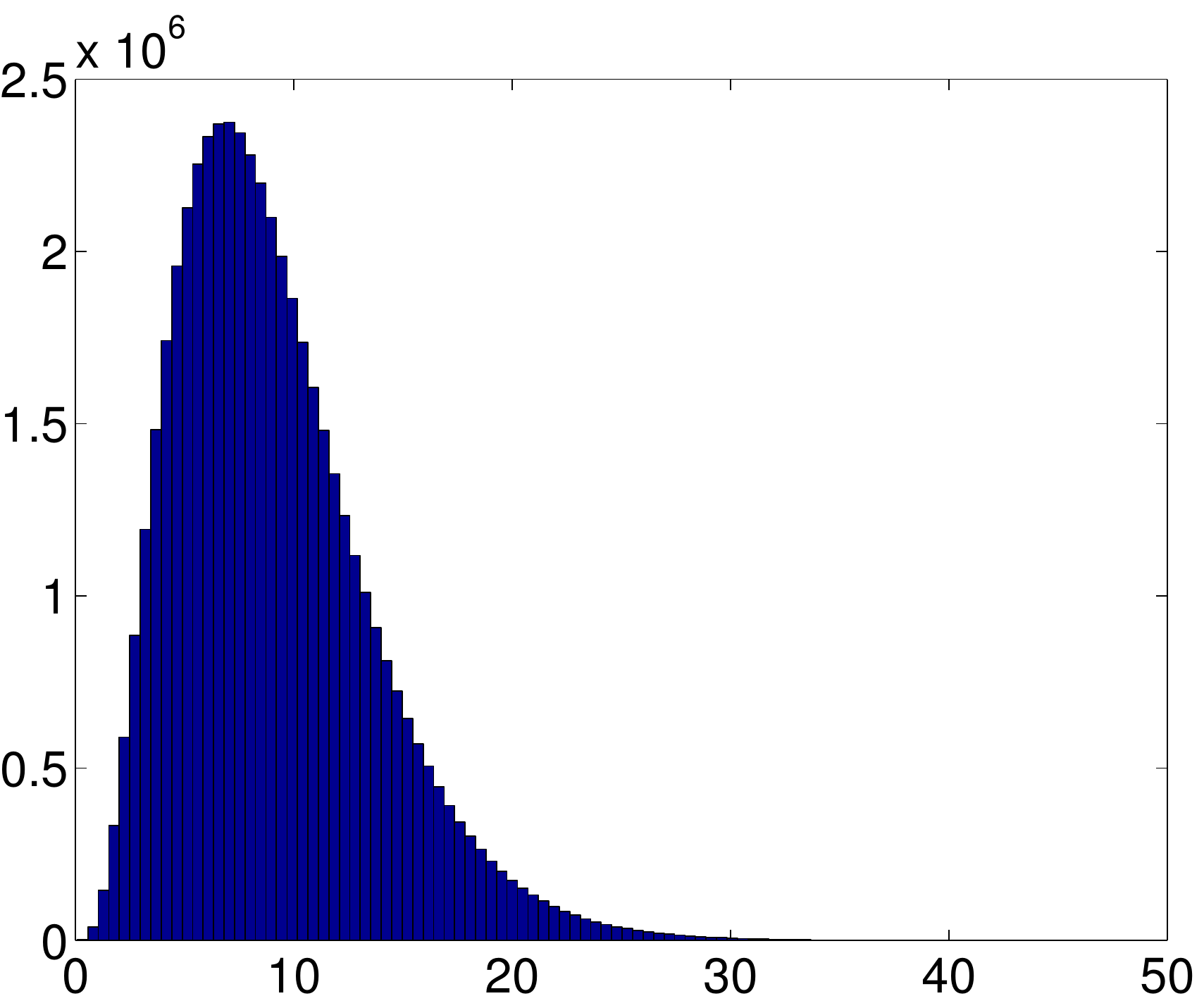}}
        \subfigure[$d = 32$]{\includegraphics[width=0.3\textwidth]{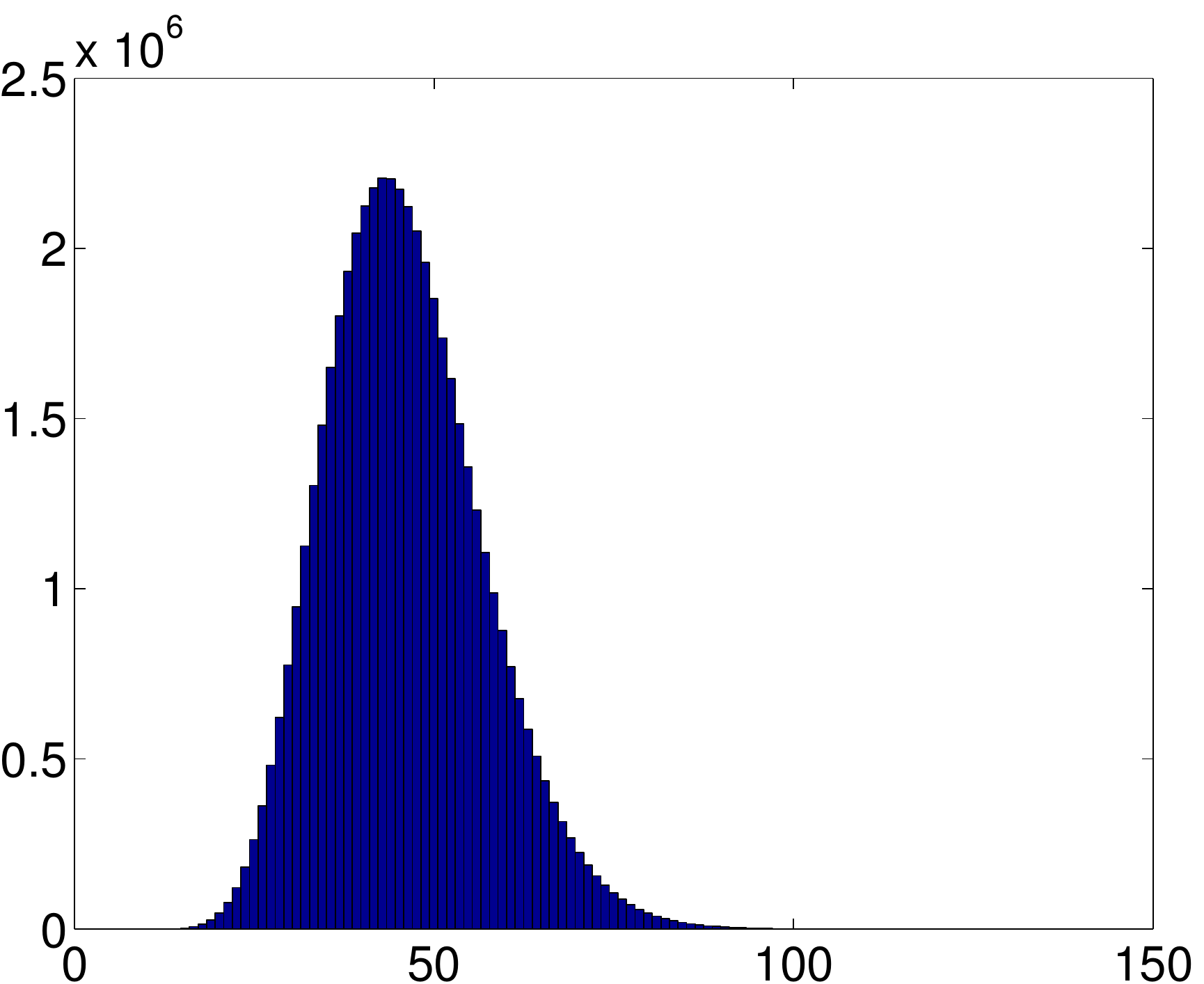}}
\caption{\textbf{Distances in high dimensions.}
Histogram of pairwise distances between sources and targets.  The data 
are drawn from a standard multivariate normal distribution in the
dimensions specified in each caption. Each experiment uses $N = 10^5$, $n = 
500$, $\xi = 1$, and $x_c$ at the origin. 
\label{fig_pairwise_dist_histograms}}
\end{figure}

\subsubsection{Subsampling methods}

To quantify the compressibility and performance of subsampling
methods, we perform the following experiments.
\begin{enumerate}
\item We fix an error tolerance $\epsilon$ and compute the $\epsilon$-rank $r$ 
of the kernel matrix $K$ -- \emph{i.e.}~the 
smallest $r$ such that $\sigma_{r+1}(K)/\sigma_1(K) < \epsilon$.
\item We sample $m_s$ rows of $K$ (according to one of the distributions below)
to form an $m_s \times n$ submatrix
$\subK$.  
\item We compute a rank $r$ interpolative decomposition of of $\subK$.
\item We reconstruct the entire matrix $K$ from this ID, and 
report its relative error. 
\end{enumerate}

For these experiments, we consider the following sampling distributions, each 
characterized by a parameter $s \in \left(0, 1\right]$.   
\begin{itemize}
\item \textbf{Uniform sampling}: Choose a number of rows $m_s = \lceil s m 
\rceil$.  Then, choose a subset of $m_s$ rows uniformly at random, without 
replacement. 
\item \textbf{Distance sampling}: Choose a number of rows $m_s = \lceil s m 
\rceil$.  Then, construct an importance sampling distribution where the 
probability of choosing row $i$ is proportional to the distance between target 
point $i$ and the source center $x_c$. We sample without 
replacement.
\item \textbf{Leverage sampling}: Choose a number of rows $m_s = \lceil s m 
\rceil$.  Then, construct an importance sampling distribution where the 
probability of choosing row $i$ is proportional to its leverage score (See 
section \ref{sec_theory}). We sample without replacement.
\item \textbf{Nearest neighbors}: Deterministically select the $m_s = \lceil s 
m \rceil$ target points that are closest to the source points.  Then, choose the
rows of $K$ corresponding to these targets.  
\end{itemize}

As we have seen, leverage-score sampling should give the best
results, but it is too expensive to be used in our context. We 
use it as the gold-standard for comparison with the other sampling methods. 
On the other hand, the uniform distribution requires no previous
knowledge and is cheap and easy to implement.  The results in 
section~\ref{sec_theory} suggest that it 
will be successful if the concentration of the matrix is small, but it is not 
clear \emph{a priori} if this will be the case.
The distance sampling distribution represents a compromise
between these extremes. While we expect that closer points will have
larger kernel interactions, and thus correspond to 
more significant rows 
of $K$, the distances could be efficiently approximated using
a space-partitioning tree or clustering.
We also use the deterministic selection of nearest neighbors
and compare it to the randomized sampling methods.

It is possible
that most of the interactions between sources and targets are captured
by the nearest neighbors. Examining the results obtained from using
these points deterministically, we can observe whether including
farther points in our approximation is important for an accurate
decomposition. Distance sampling and
nearest neighbors require precomputations that in turn need to be
accelerated using fast methods, since their direct calculation is
$\bigO(m n)$. 
Nearest neighbors can be computed efficiently in low dimensions 
\cite{gray2001n}.  In high dimensions, methods on binary
tree partitions or hashing methods can be used for exact or
approximate schemes, for instance with random projection
trees \cite{dasgupta-freund08}.

\subsection{Gaussian kernel}

We begin with the Gaussian kernel:
\begin{equation}
\Ker(y, x) = \exp \left( -\frac{1}{2 h^2} \|x - y\|^2 \right).
\label{eqn_gaussian_kernel}
\end{equation}
The kernel is characterized by a \emph{bandwidth} $h \in (0, \infty)$.  

\subsubsection{Choice of parameters}

Clearly, the choice of bandwidth is critical to the behavior of the
Gaussian kernel. As $h$ tends to zero, the kernel matrix $K$ will
become increasingly sparse. When the sources are not included in the
target set, the rank of $K$ will become zero.  On the other hand, as
$h$ grows, all entries of $K$ will tend toward one, resulting in a rank
one matrix. While both these cases will compress extremely
effectively, neither is of much practical interest.

Furthermore, we expect the behavior of the kernel to depend on the
simultaneous choice of $h$ and $d$. As we discussed in
Figure~\ref{fig_pairwise_dist_histograms}, the distances between pairs
of sources and targets become increasingly concentrated in high
dimensions. Therefore, a single fixed value of $h$ will demonstrate
very different behavior as $d$ increases.\footnote{In the literature,
when studying the performance of far-field compression, a fixed range of
values of $h$ is typically used, \emph{e.g.}~$h \in \{10^{-3}, \ldots, 
10^{3}\} \times h^*$ for some $h^*$. However, we find that this approach is not very informative
as the range of values of $h$ for which the kernel exhibits interesting
behavior becomes more narrow with increasing dimension.} 

In order to determine a scale of $h$ that will account for this
variation, we consider the choice of bandwidth made in solving kernel
density estimation problems in non-parametric statistics.  Silverman
\cite{silverman1986density} gives the asymptotically optimal (in terms
of expected squared error) choice of $h$ for KDE when the true
underlying distribution is the standard multivariate normal:
\begin{equation}
h_S = \left( \frac{4}{2 d + 1} \right)^{\frac{1}{d+4}} N^{-\frac{1}{d+4}}.
\label{eqn_silverman_bandwidth}
\end{equation}
We use the value $h_S$ (which depends on $d$) as a reference scale in our experiments.  


\begin{table}[tp]
\centering
\caption{Values of $h$ for normally
  distributed data. The first column is $h_S$ for $N = 10^5$. The remaining 
  columns correspond to different rank budgets
  of $\kappa = 1\%$ and $\kappa = 20\%$ of the columns, all given in units of 
  $h_S$. The ``$+$'' and ``$-$'' 
  correspond
  to the larger and smaller values of $h$ where this rank budget is
  achieved.  These values are in units of $h_S$ computed for the given
  dimension. The sample mean over 30 independent 
  realizations of the
  data is given, with the sample standard deviation in parentheses.  
\label{table_normal_h_values}}
\begin{tabular}{|c|c|c|c|c|} \hline
$d$ & $h_S$ & $\kappa = -1\%$ & $\kappa = -20\%$ & $\kappa = +20\%$ \\ \hline
4 & 0.2143 & 0.0587 (0.014) & 0.1656 (0.008) & 1.1719 ($\approx$0) \\
8 & 0.3396 & 0.1879 (0.05) & 0.4082 (0.016) & 2.6367 ($\approx$0) \\
16 & 0.5060 & 0.3708 (0.09) & 0.6700 (0.035) & 3.9062 ($\approx$0) \\
32 & 0.6722 & 0.5398 (0.17) & 0.8999 (0.063) & 3.955 ($\approx$0) \\
64 & 0.8022 & 0.6887 (0.24) & 1.1989 (0.102) & 4.2090 (0.016) \\ \hline
\end{tabular}
\end{table}


Let us emphasize that in practice, the value of $h$ depends on the
algorithm and the application. Commonly, the value is chosen through
cross-validation on some objective function of interest.  This in turn
requires a search over many values of $h$. We suggest $h_S$ as a
starting point for this search, and we explore a range of values. In
exploring this range an additional criterion is the magnitude of the
far field. If the contribution of the far-field becomes too small, the
kernel is too narrow and nearest neighbors can capture the
interactions accurately. On the other hand, if the far field becomes
dominant, the kernel compresses quite well.

\subsubsection{Singular values of $K$\label{gaussian_exp_spectrum}}

Following our intuition above, we expect that
for very small and very large values of $h$, the kernel will compress
easily. For values in between, we expect the singular values to be flatter,
thus implying a greater difficulty in approximating the kernel.   
We would like to know the width of this ``difficult'' region for different 
values of $d$. We empirically measure this range in the following way:
\begin{itemize}
\item We specify a rank tolerance $\epsilon$.
\item We specify a rank budget $\kappa \in \left(0, 1\right]$ as a percentage 
of $n$, the largest possible rank of $K$.  
\item We search over bandwidths $h$ such that the $\epsilon$-rank of
  $K$ is close to $\kappa n$. Note that we expect there to be two ranges of
  $h$ where this occurs, one for small $h$ and one for larger $h$.  
\end{itemize}

For data drawn from the standard normal distribution and $x_c$ at the
origin, we approximately compute these values of $h$ using binary
search. Our results are given in Table~\ref{table_normal_h_values} in
units of $h_S$.

\begin{figure}
        \centering
        \subfigure[$\kappa = 1\%$]{\includegraphics[width=0.4\textwidth]{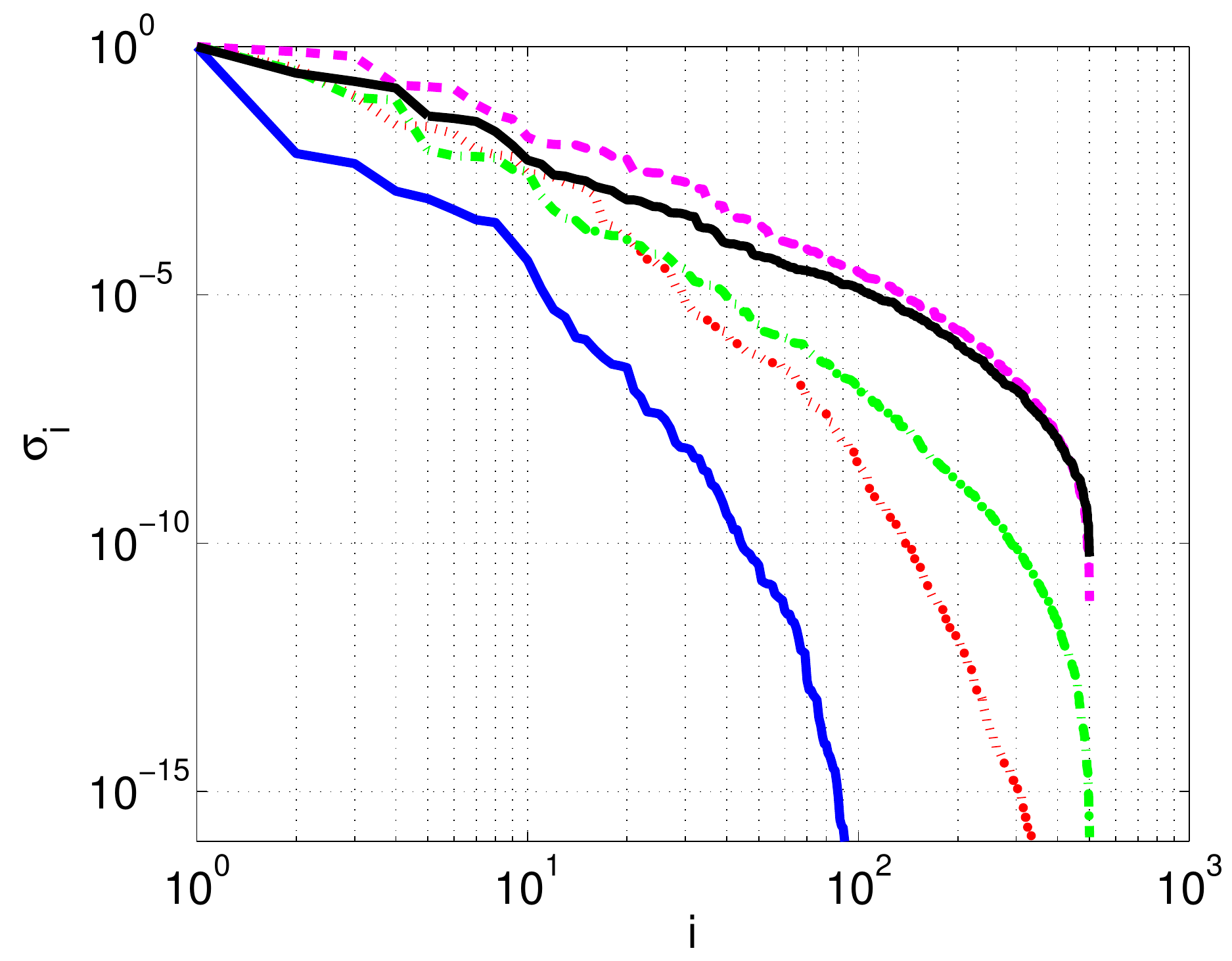}}
        \subfigure[$\kappa = 20\%$]{\includegraphics[width=0.4\textwidth]{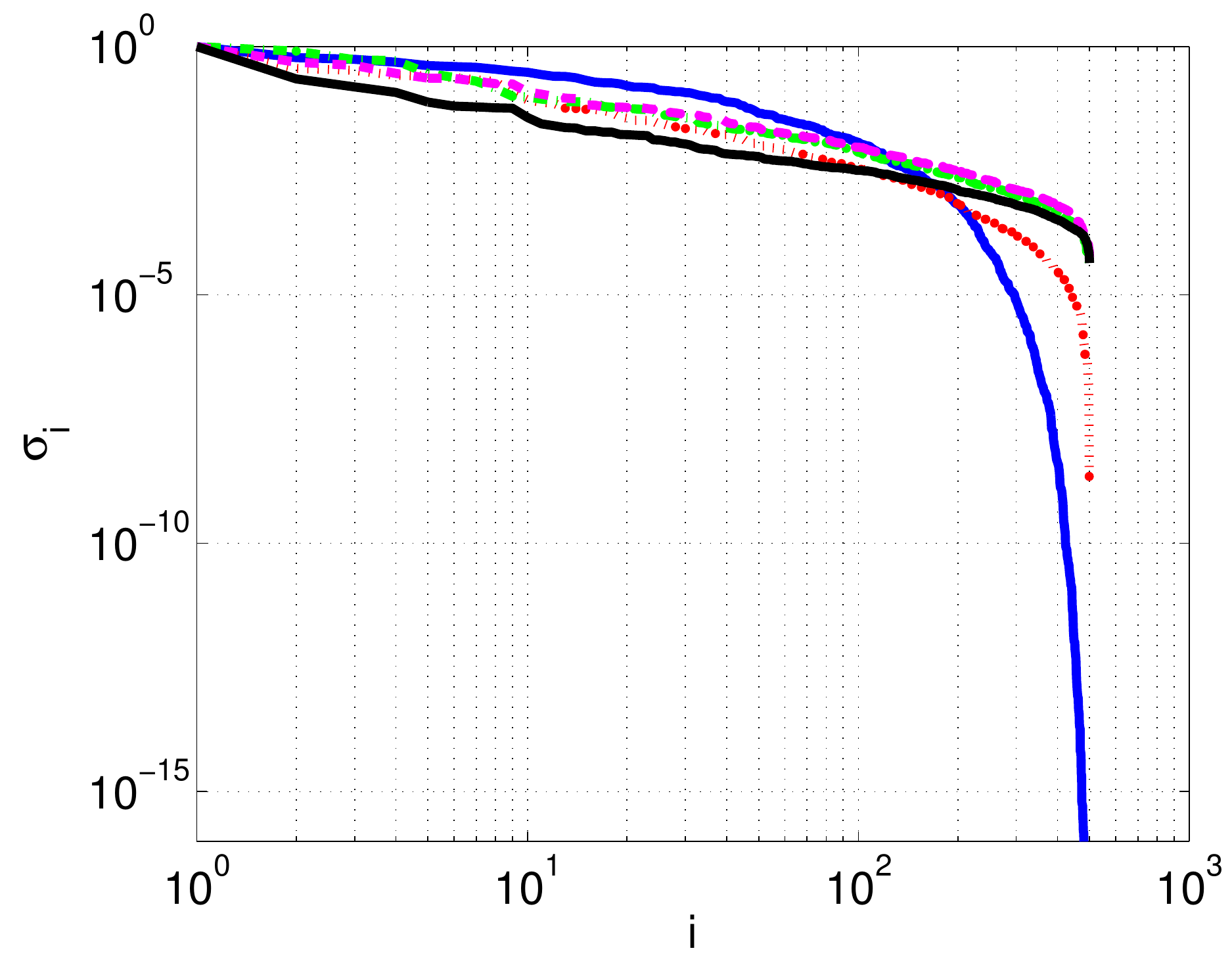}}
        \subfigure[$h = h_S$]{\includegraphics[width=0.4\textwidth]{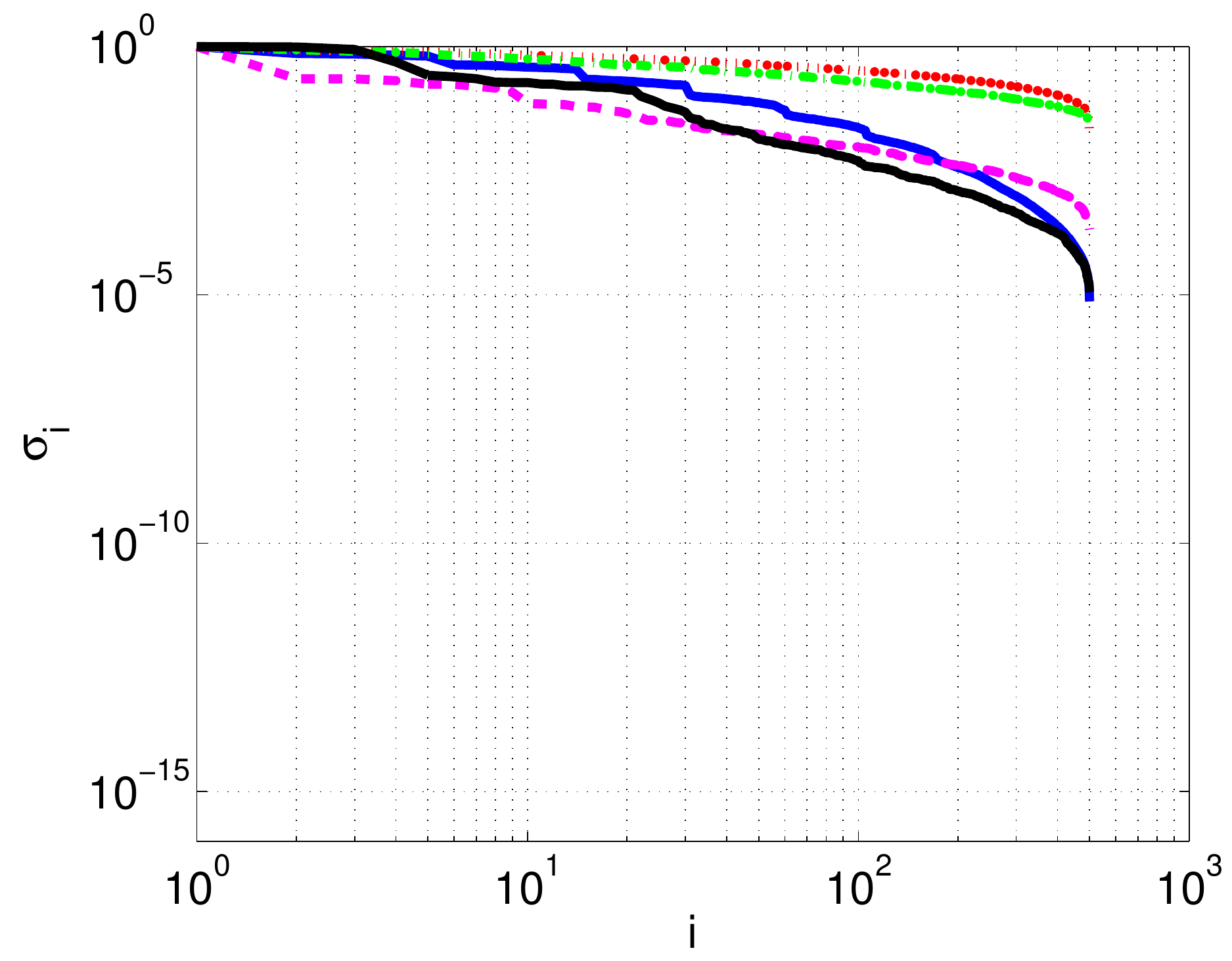}}
        \subfigure[$\kappa = +20\%$]{\includegraphics[width=0.4\textwidth]{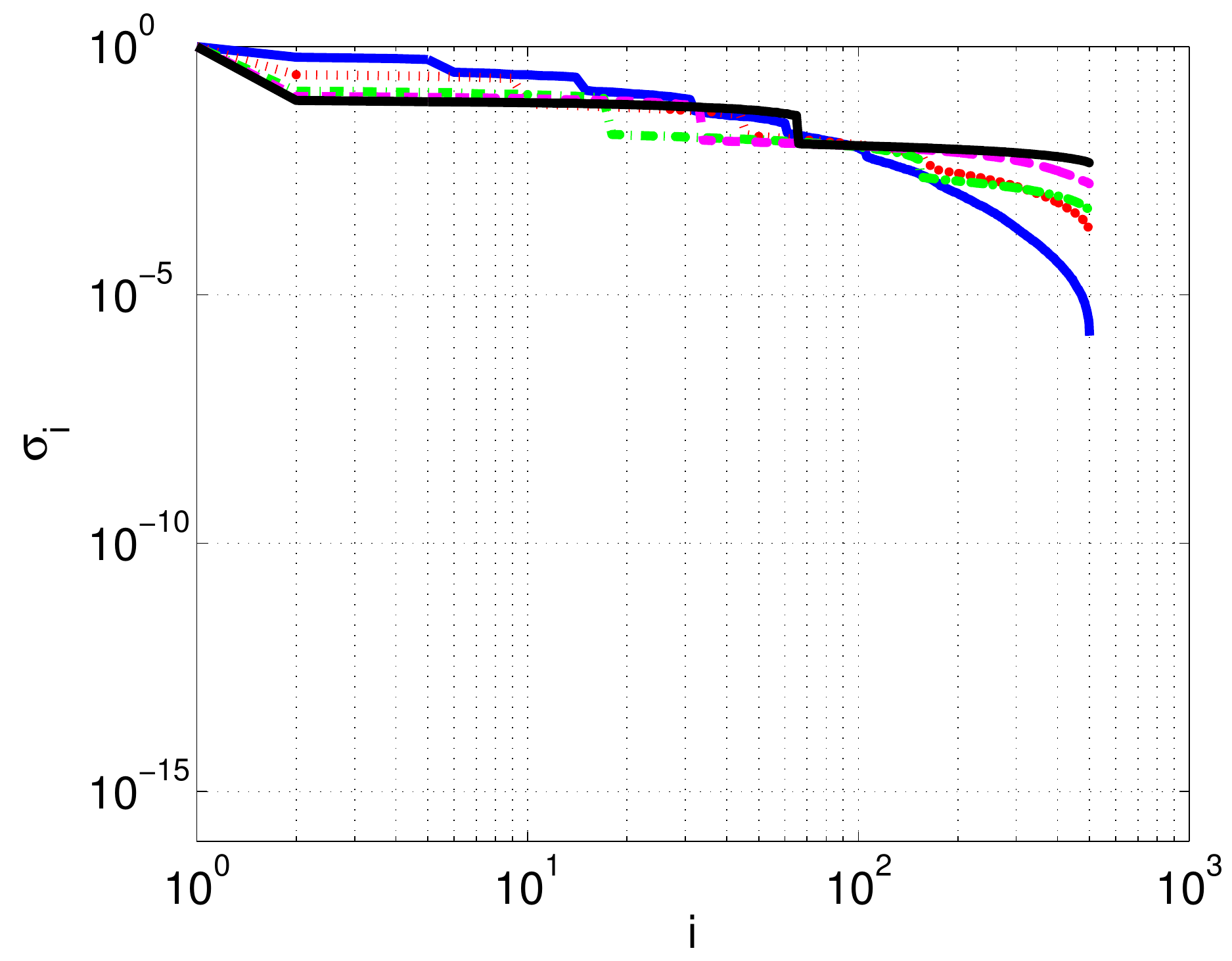}}
\caption{\textbf{Singular values of the Gaussian kernel.} 
We report the compressibility of the far field by computing the  
singular values of $K$ for the Gaussian kernel. We draw $N = 10^5$ points from 
a standard normal distribution. We set $n = 500$,
  $\xi = 1$, and $x_c$ at the origin. The specified values of $\kappa$ correspond to the bandwidths  
  given in Table \ref{table_normal_h_values}. 
  The trend lines show $d = 4$
  (blue), $d = 8$ (red), $d = 16$ (green), $d = 32$ (magenta), and $d
  = 64$ (black).  
\label{fig_gaussian_spectra_scaled}}
\end{figure}

We then plot the singular values of $K$ for these values of $h$ in
Figure \ref{fig_gaussian_spectra_scaled}.  These values suggest that for
$n = 500$, we are not yet in the regime where $K$ can be compressed
with high accuracy.  If computational resources are available, increasing $n$ 
may be preferable. However, here we do not consider the dependence on $n$
since this is a performance optimization in which one balances direct
interactions and far-field interactions.  Also, this amount of
compression is sufficient for lower accuracy, such as is commonly
required in machine learning applications.

These results show that the range of values of $h$ for which $K$ is low rank 
quickly grows as $d$ increases.  As we
demonstrated previously, for truly high-dimensional data, the pairwise
distances between sources and targets become concentrated.  Therefore, for more 
values of $h$, the quantity $\|x - y\|^2 /
h^2$ will be either very small or very large.  This in turn makes the
singular values decay quickly.

\textbf{Influence of nearest neighbors.} 
The Gaussian kernel decays quickly with increasing distance.
Therefore, one possible approximation strategy is to compute the kernel
interactions between nearby pairs of points and truncate the remaining
interactions. In Table~\ref{table_gaussian_interaction_percentages},
we break down the total interactions for target points in order to
demonstrate that the contribution of distant source points can
be significant. In the case that more distant points make a significant 
contribution, an approximation scheme like
ours is necessary to accurately compute the kernel sum.  

We set $\xi = 1$ and $n = 500$ and draw $N$ data points from the standard
multivariate normal distribution. We fix the $n$ points closest to the origin
as the source points, and we consider all $N$ points (including the sources)
as targets. We also define the $n$ next-closest points to the origin as 
the nearest neighbors of the sources. We refer to the remaining $N - 2n$
points as the far-field.

\begin{table}[tb]
\centering
\caption{Fractions of interactions (Equation~\ref{equation_interactions_def}) 
for given values of $d$ and $h$ for the 
Gaussian kernel and standard multivariate normal data. Reported values are 
$\|K_\star\|_2 / \|K\|_2$ as percentages for each matrix described in section 
\ref{gaussian_exp_spectrum}.  All experiments use $N = 10^5$, $\xi = 1$, and 
$n = 500$ with $n$ nearest neighbors. 
\label{table_gaussian_interaction_percentages}}   
\begin{tabular}{|c|c|c|c|c|} \hline
$d$ & $h$ & Self & NN & Far \\ \hline
4 & $-1\%$ & 100.00 & 0.56 & 0.00 \\
4 & $-20\%$ & 100.00 & 26.24 & 0.00 \\
4 & $h_S$ & 84.44 & 40.52 & 38.09 \\
4 & $+20\%$ & 78.38 & 42.41 & 46.74 \\ \hline
8 & $-1\%$ & 100.00 & 0.02 & 0.00 \\
8 & $-20\%$ & 100.00 & 12.48 & 0.01 \\
8 & $h_S$ & 93.34 & 27.48 & 30.48 \\
8 & $+20\%$ & 37.55 & 26.54 & 88.82 \\ \hline
16 & $-1\%$ & 100.00 & 0.00 & 0.00 \\
16 & $-20\%$ & 100.00 & 0.40 & 0.00 \\
16 & $h_S$ & 99.98 & 7.55 & 1.18 \\
16 & $+20\%$ & 19.25 & 15.13 & 96.96 \\ \hline
32 & $-1\%$ & 100.00 & 0.00 & 0.00 \\
32 & $-20\%$ & 100.00 & 0.03 & 0.00 \\
32 & $h_S$ & 100.00 & 0.15 & 0.00 \\
32 & $+20\%$ & 18.48 & 13.50 & 97.35 \\ \hline
64 & $-1\%$ & 100.00 & 0.00 & 0.00 \\
64 & $-20\%$ & 100.00 & 0.00 & 0.00 \\
64 & $h_S$ & 100.00 & 0.00 & 0.00 \\
64 & $+20\%$ & 18.94 & 13.07 & 97.32 \\ \hline
\end{tabular}
\end{table}

We compute the $N \times n$ matrix $K$ of interactions between sources and 
targets. We partition $K$ into  
three submatrices according to the sets identified above: self-interactions, 
nearest neighbor interactions, and far-field interactions.
In other words, we have:
\begin{equation}
K = \left[  
\begin{array}{c}
	K_S \\
	K_N \\
	K_F \\ 
\end{array}
\right]
\quad \quad
\begin{array}{l}
	K_S \in \reals^{n \times n} \,\textrm{-- source-source interactions} \\
	K_N \in \reals^{n \times n} \,\textrm{-- neighbor-source interactions} \\
	K_F \in \reals^{(N-2n) \times n} \,\textrm{-- far field-source interactions} \\
\end{array}
\label{eqn_interation_splits}
\end{equation}
where $K_S$ is the $n \times n$ matrix of interactions between the sources
and themselves, 
$K_N$ is the $n \times n$ matrix of interactions between the $n$ nearest 
neighbors and the sources, and $K_F$ is the $(N-2n) \times n$ matrix of 
interactions between the far field and the sources. 

We are interested in quantifying the contribution of each of the three 
sets to the total action of the matrix $K$. We compute the following 
quantities: 
\begin{equation}
  \begin{array}{ccc}
\textrm{Self} = \frac{\|K_S\|_2}{\|K\|_2},\quad &
\textrm{NN} = \frac{\|K_N\|_2}{\|K\|_2}, \quad &
\textrm{Far} = \frac{\|K_F\|_2}{\|K\|_2}. 	 \\
  \end{array}
  \label{equation_interactions_def}
\end{equation}
We compute these quantities for several values of $d$ and $h$ in 
Table~\ref{table_gaussian_interaction_percentages}. We see that for larger 
values of $h$, an accurate approximation algorithm must take the distant 
targets into account.


\subsubsection{Subsampling}

We have shown that the submatrix $K$ can be meaningfully compressed
for a range of values of $d$ and $h$ and that the far field is significant for 
some of these values.  We now turn to our results on
subsampling rows to build an outgoing representation.  In Figures
\ref{fig_gaussian_subsampling_id_d_4},
through \ref{fig_gaussian_subsampling_id_d_64}, 
we show these {\bf results for 4, 32, 
and 64 dimensional data}.

We select the data and partition them into sources and targets as before. We 
fix $\epsilon = 10^{-2}$ and choose the approximation rank $r$ as the 
smallest $r$ such that 
$\sigma_{r+1}(K) / \sigma_1(K) < \epsilon$.	
With this choice of $r$, the best possible reconstruction error is 
$\sigma_{r+1}(K)/\sigma_1(K)$, even if we were to use the SVD. Therefore, we
do not observe any reconstruction errors better than this ratio even when 
sampling all of the rows.

For small bandwidths (corresponding to $1\%$ and $20\%$ of the possible rank), 
the leverage score and nearest neighbor sampling methods perform very well, 
obtaining the same approximation quality as the full-row decomposition with a 
very small fraction of the rows.  The uniform and distance distributions obtain 
very poor accuracy for even $10\%$ of the rows. This suggests that the nearest 
neighbors account for most of the interaction, and that the quality of the 
decomposition is very sensitive to having these neighbors in the sample. This
fits with the results in Table~\ref{table_gaussian_interaction_percentages}
for smaller values of $h$.
For the larger bandwidths, we see that at about $1\%$ of the rows,  all of our 
row selection methods perform nearly as well as the decomposition of the full 
matrix.

\begin{figure}
        \centering
        \subfigure[$\kappa = 1\%$]{\includegraphics[width=0.4\textwidth]{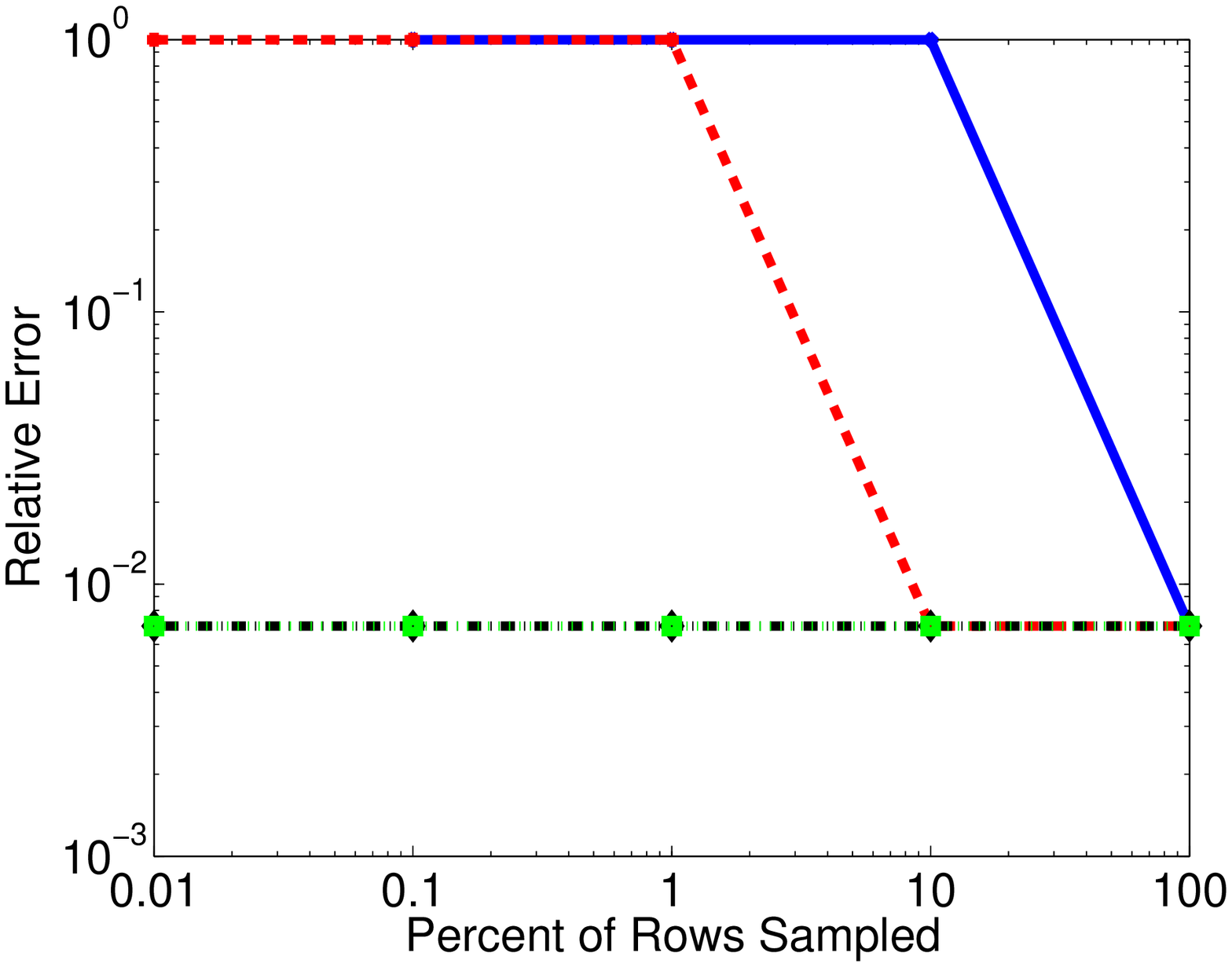}}
        \subfigure[$\kappa = 20\%$]{\includegraphics[width=0.4\textwidth]{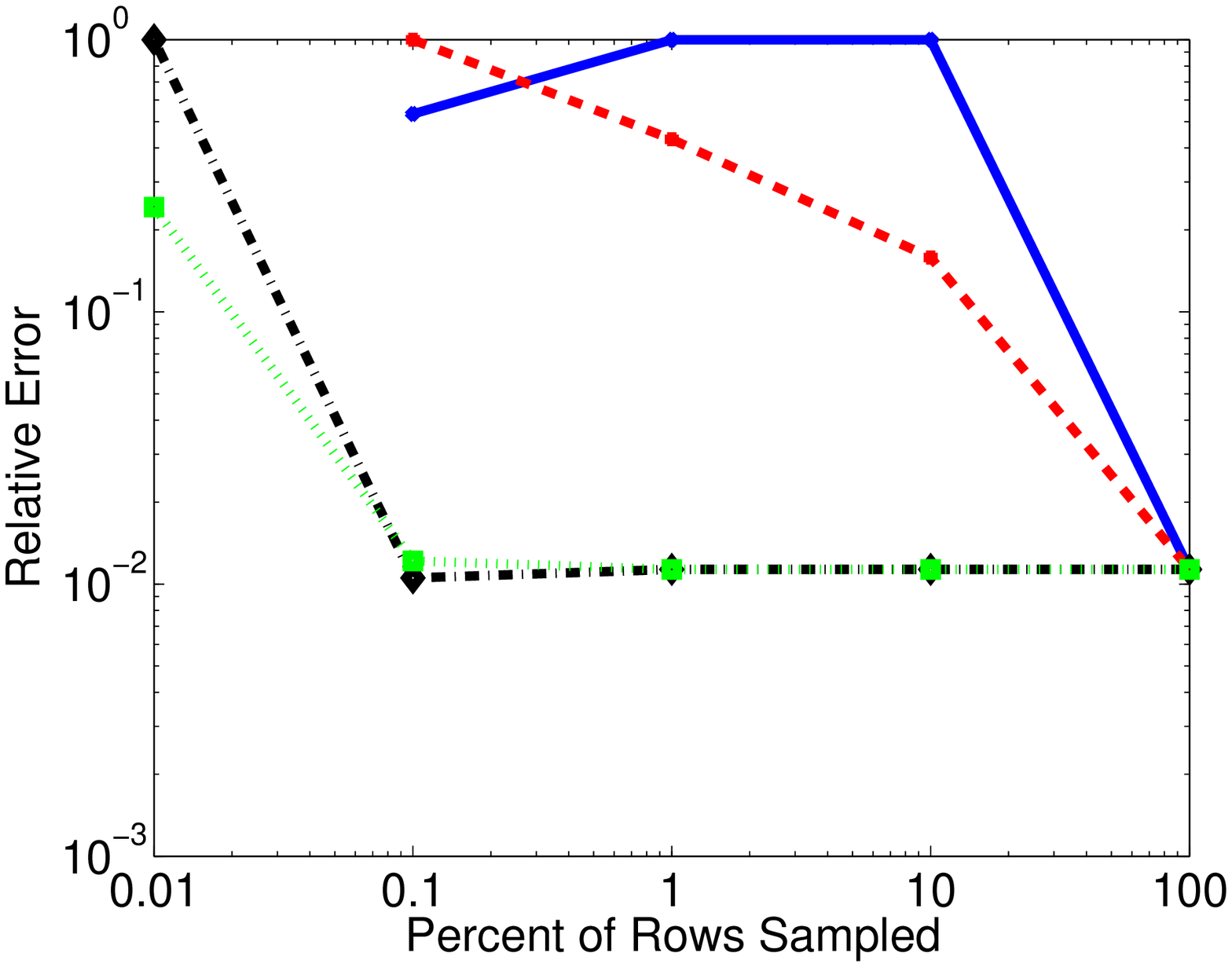}}
        \subfigure[$h = h_S$]{\includegraphics[width=0.4\textwidth]{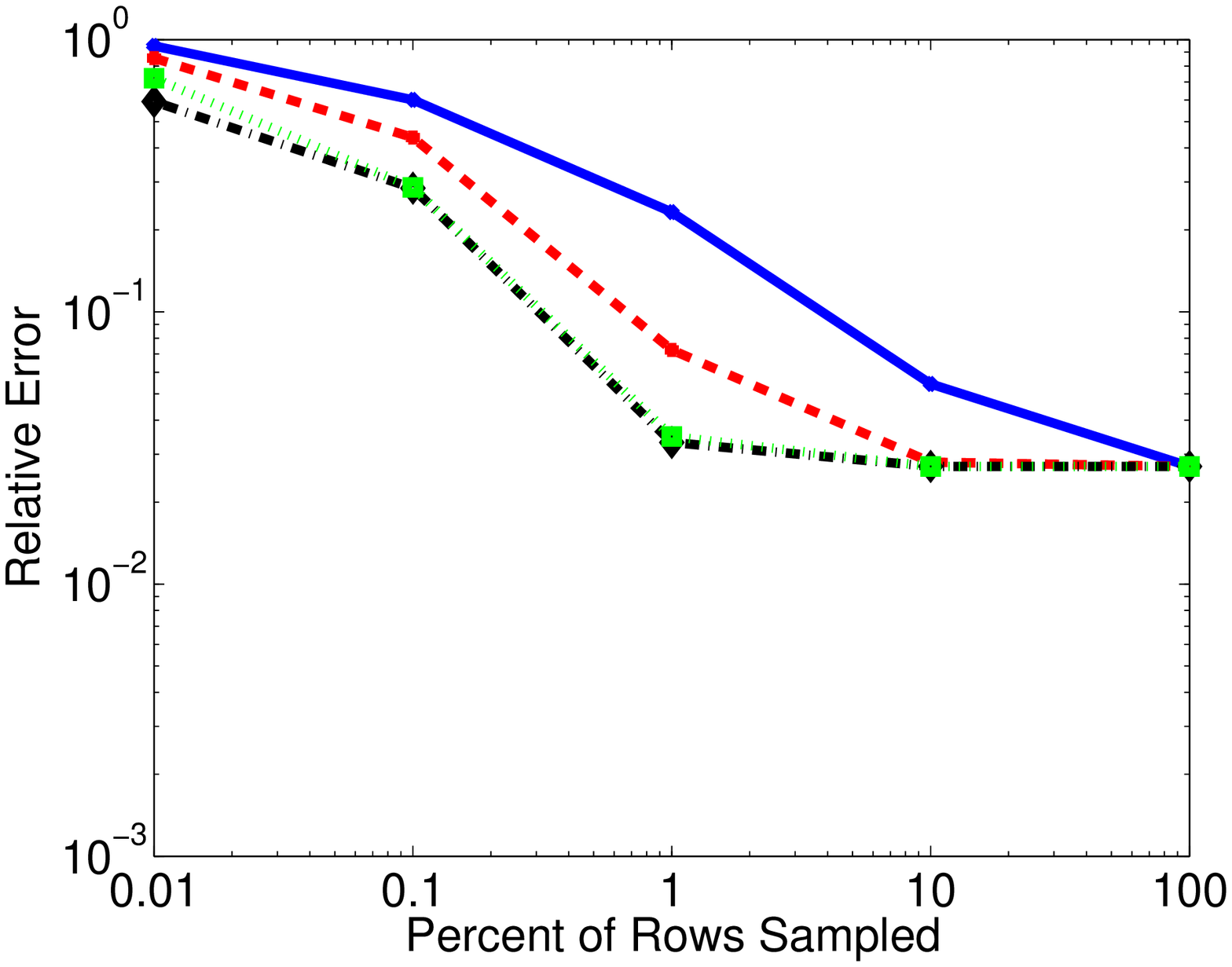}}
        \subfigure[$\kappa = +20\%$]{\includegraphics[width=0.4\textwidth]{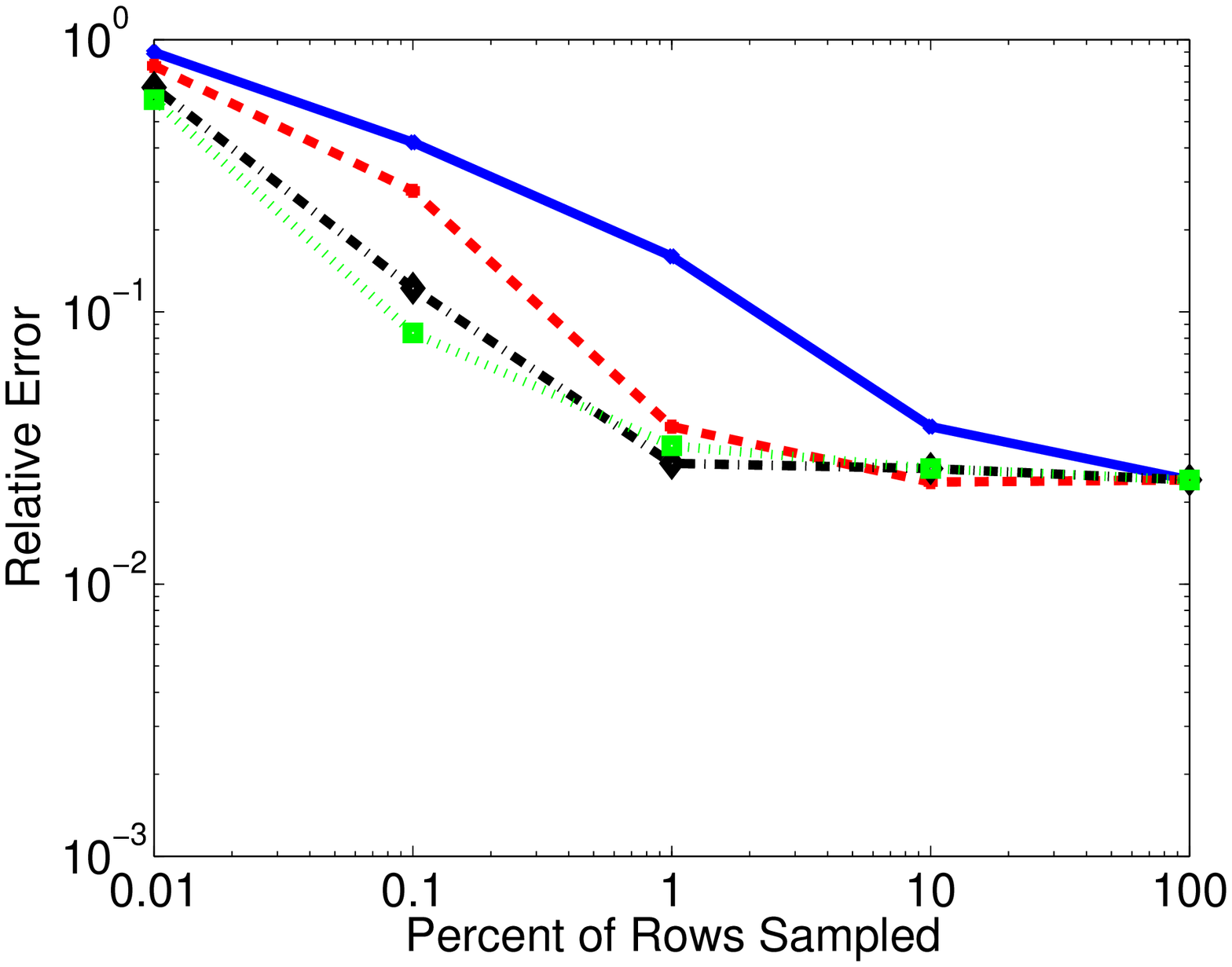}}
\caption{\textbf{ID compression; Gaussian kernel; normal data, \boldmath$d = 4$\unboldmath.}
We show the approximation error of the ID obtained from a subsampled matrix 
$\subK$. 
	We draw $N = 10^5$ points from the 4-dimensional standard normal 	
	distribution and set $n = 500$ and $x_c$ at the origin.
  We use the bandwidth given in the subfigure captions and Table 
 \ref{table_normal_h_values}. We set $\epsilon = 10^{-2}$ and choose 
 the rank $r$ so that it is the smallest $r$ such that 
 $\sigma_{r+1}(K)/\sigma_1(K) < \epsilon$.
  Each trend line represents a
  different subsampling method, with \underline{blue for the uniform
  distribution}, \underline{red for distances}, \underline{black for leverage}, and \underline{green for
  the deterministic selection of nearest neighbors}. 
\label{fig_gaussian_subsampling_id_d_4}}
\end{figure}

\begin{figure}
       \centering
       \subfigure[$r = 1\%$]{\includegraphics[width=0.45\textwidth]{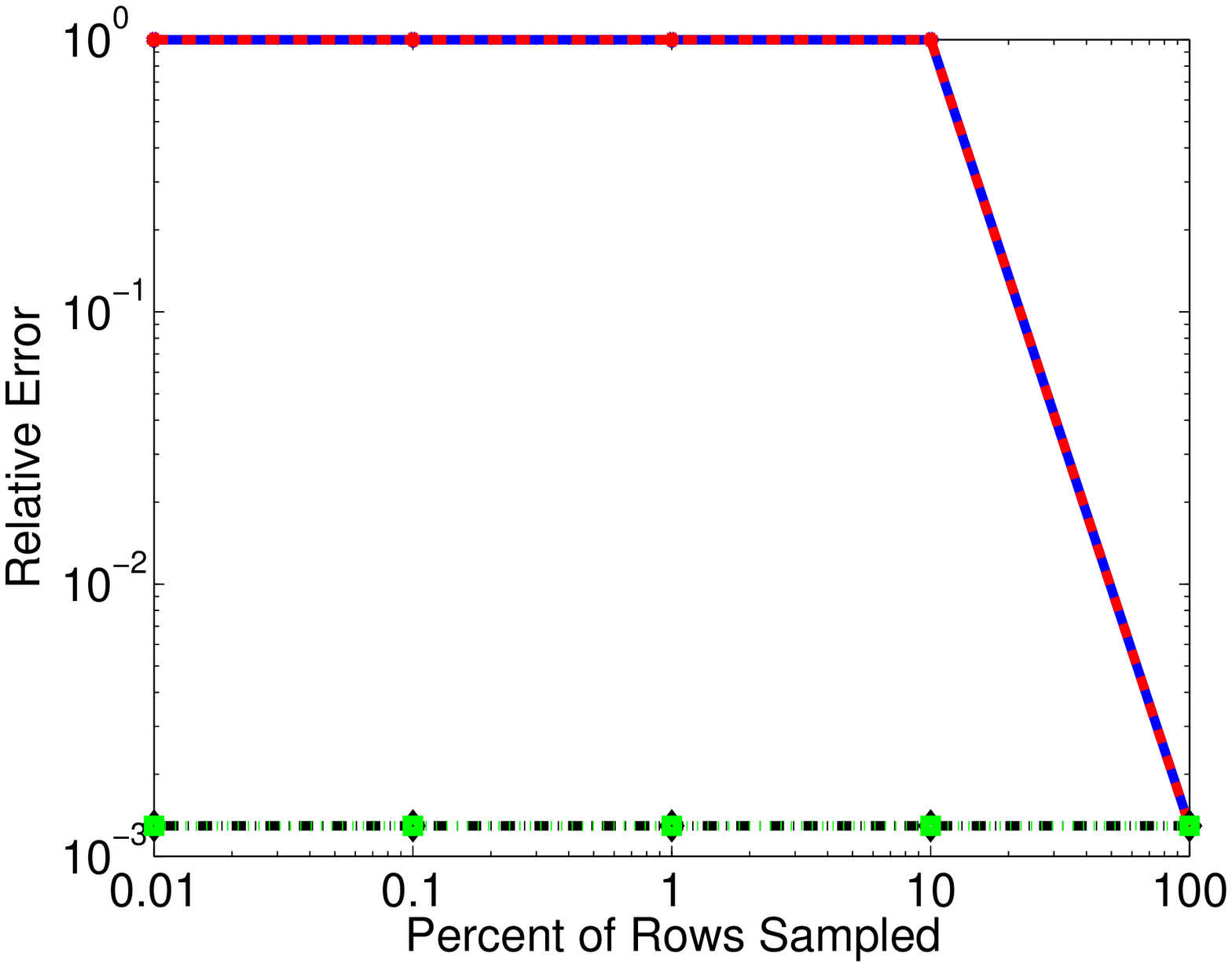}}
       \subfigure[$r = 20\%$]{\includegraphics[width=0.45\textwidth]{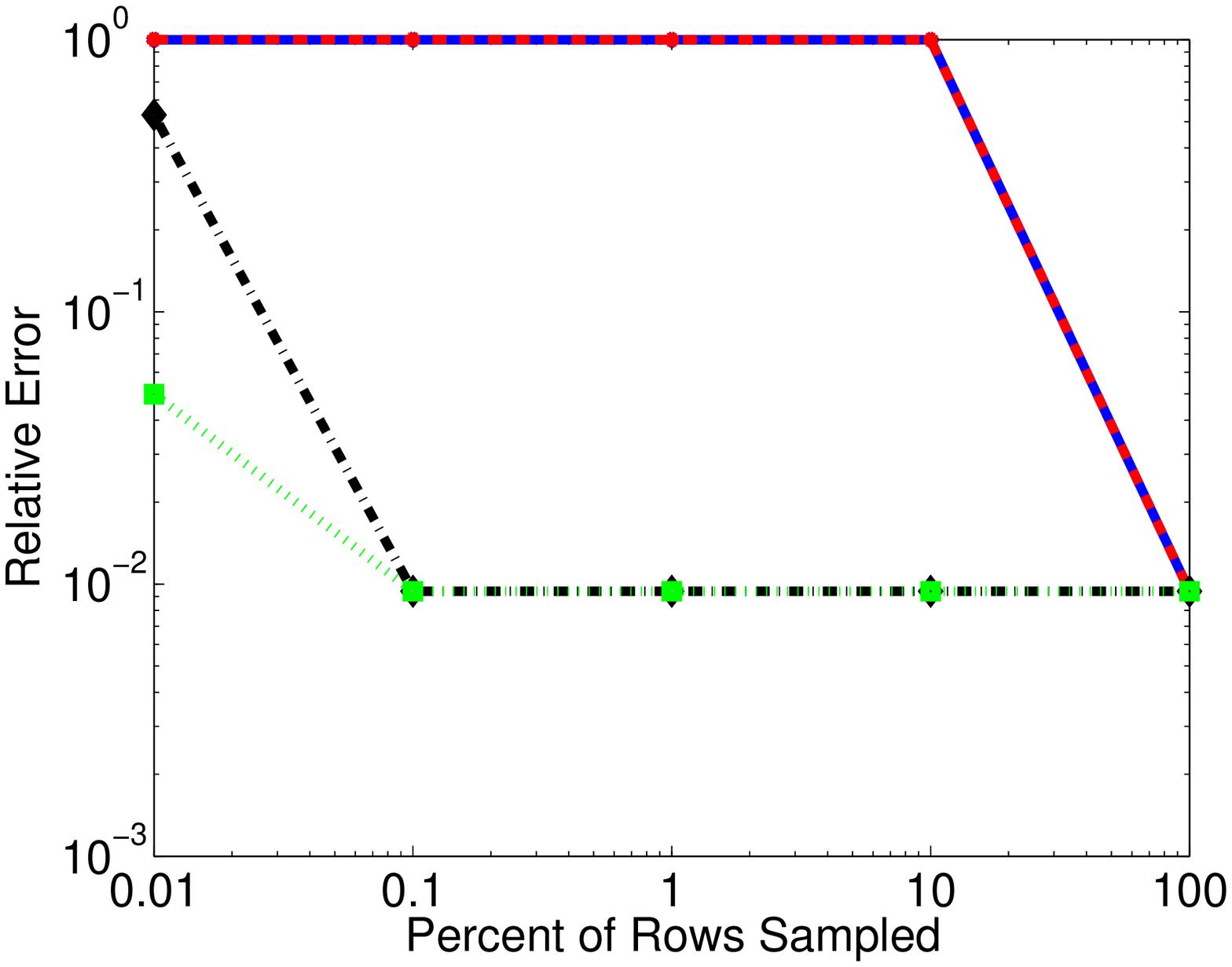}}
       \subfigure[$h = h_S$]{\includegraphics[width=0.45\textwidth]{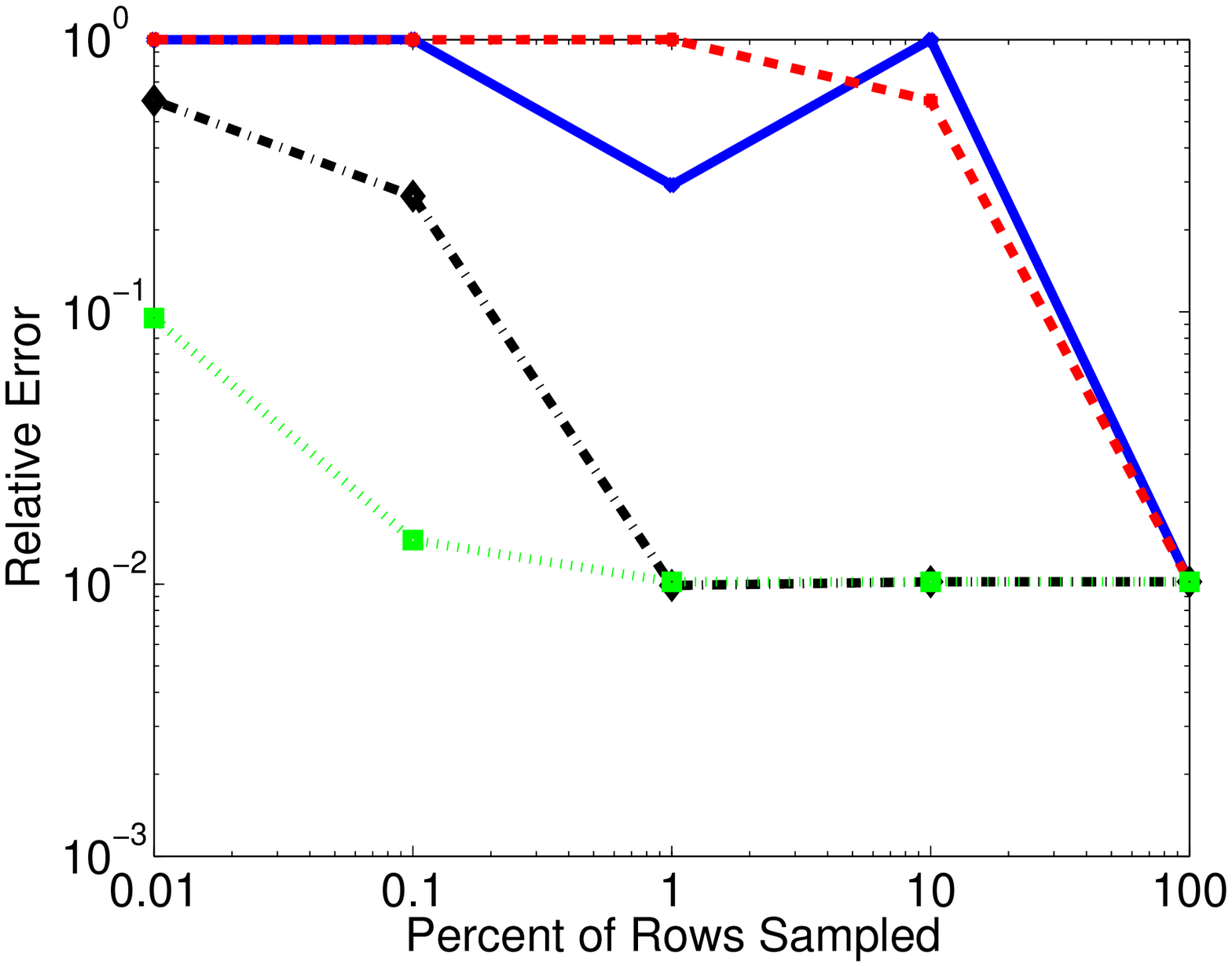}}
       \subfigure[$r = +20\%$]{\includegraphics[width=0.45\textwidth]{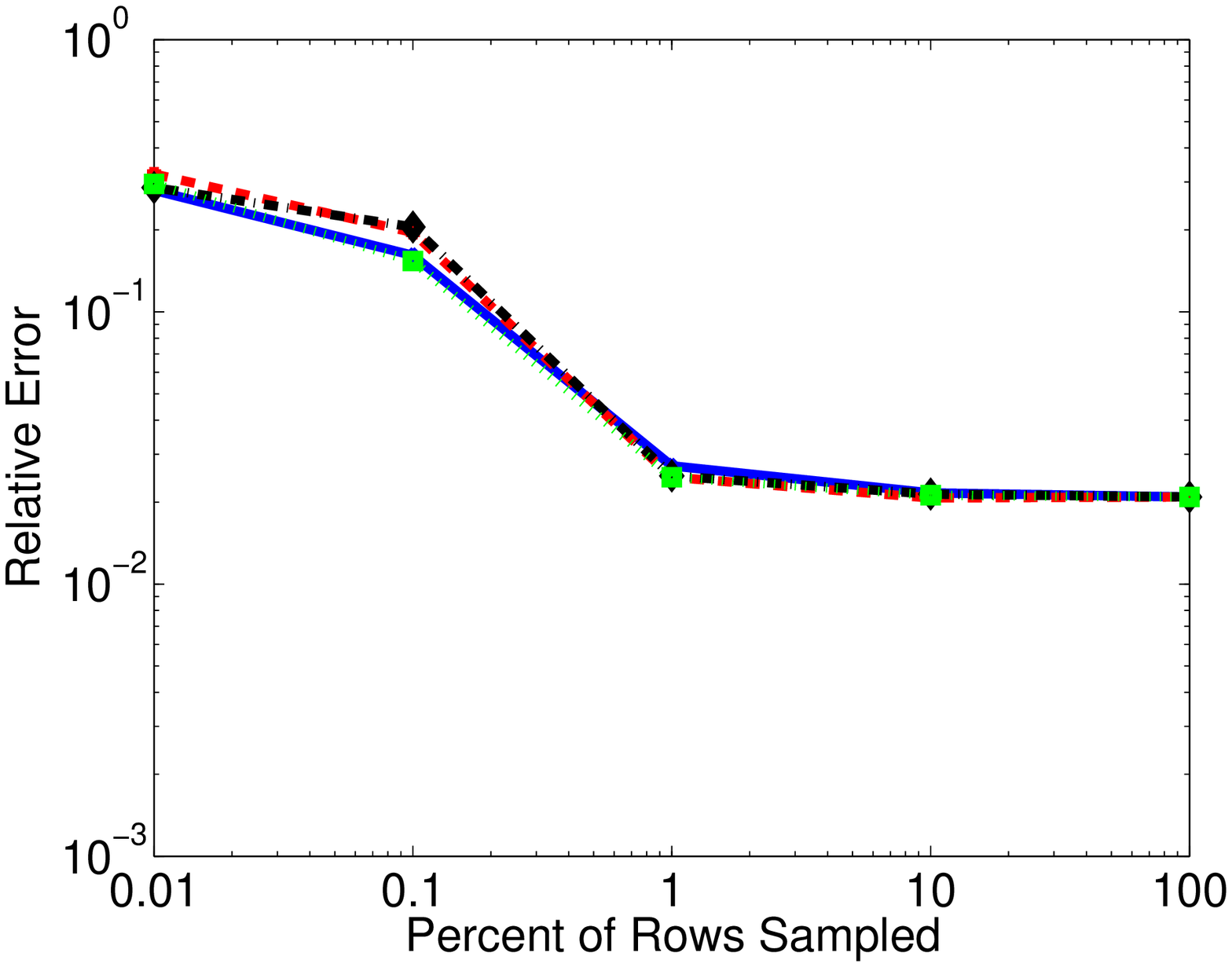}}
\caption{\textbf{ID compression; Gaussian kernel; normal data, \boldmath$d = 32$\unboldmath.}
We show the approximation error of the ID obtained from a subsampled matrix 
$\subK$. 
	We draw $N = 10^5$ points from the 32-dimensional standard normal 	
	distribution and set $n = 500$ and $x_c$ at the origin.
  We use the bandwidth given in the subfigure captions and Table 
 \ref{table_normal_h_values}. We set $\epsilon = 10^{-2}$ and choose 
 the rank $r$ so that it is the smallest $r$ such that 
 $\sigma_{r+1}(K)/\sigma_1(K) < \epsilon$.
  Each trend line represents a
  different subsampling method, with \underline{blue for the uniform
  distribution}, \underline{red for distances}, \underline{black for leverage}, and \underline{green for
  the deterministic selection of nearest neighbors}. 
\label{fig_gaussian_subsampling_id_d_32}}
\end{figure}

\begin{figure}
        \centering
        \subfigure[$\kappa = 1\%$]{\includegraphics[width=0.4\textwidth]{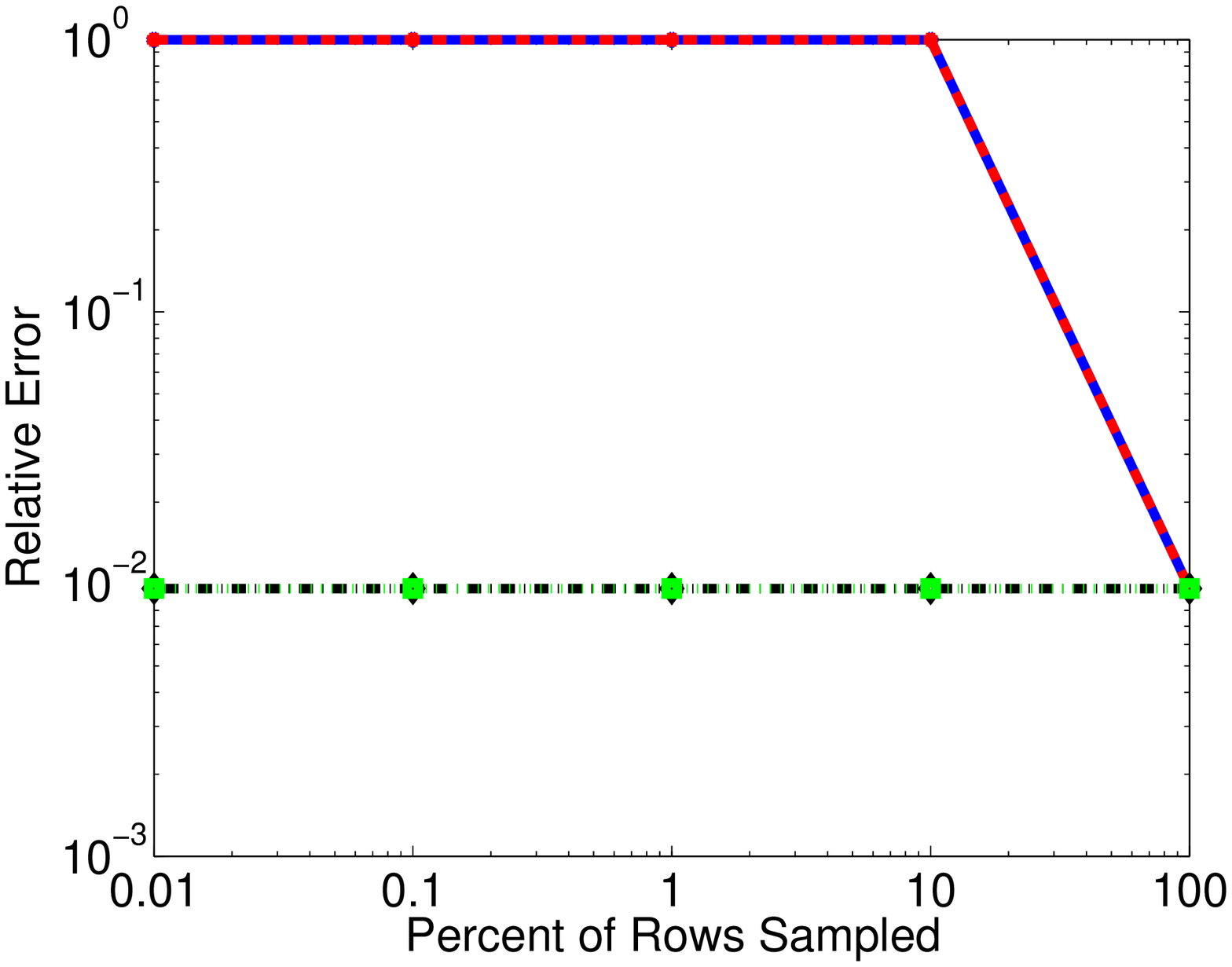}}
        \subfigure[$\kappa = 20\%$]{\includegraphics[width=0.4\textwidth]{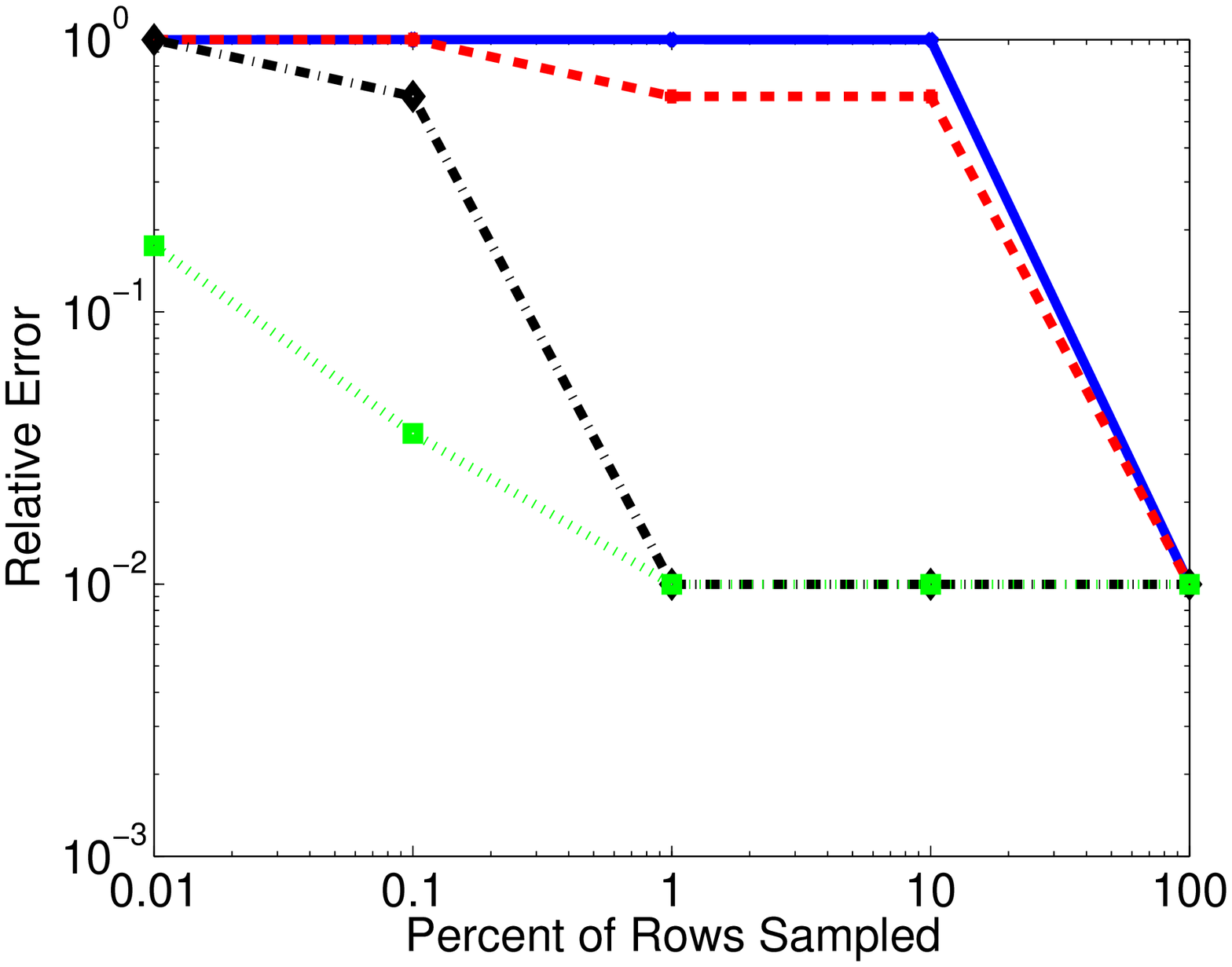}}
        \subfigure[$h = h_S$]{\includegraphics[width=0.4\textwidth]{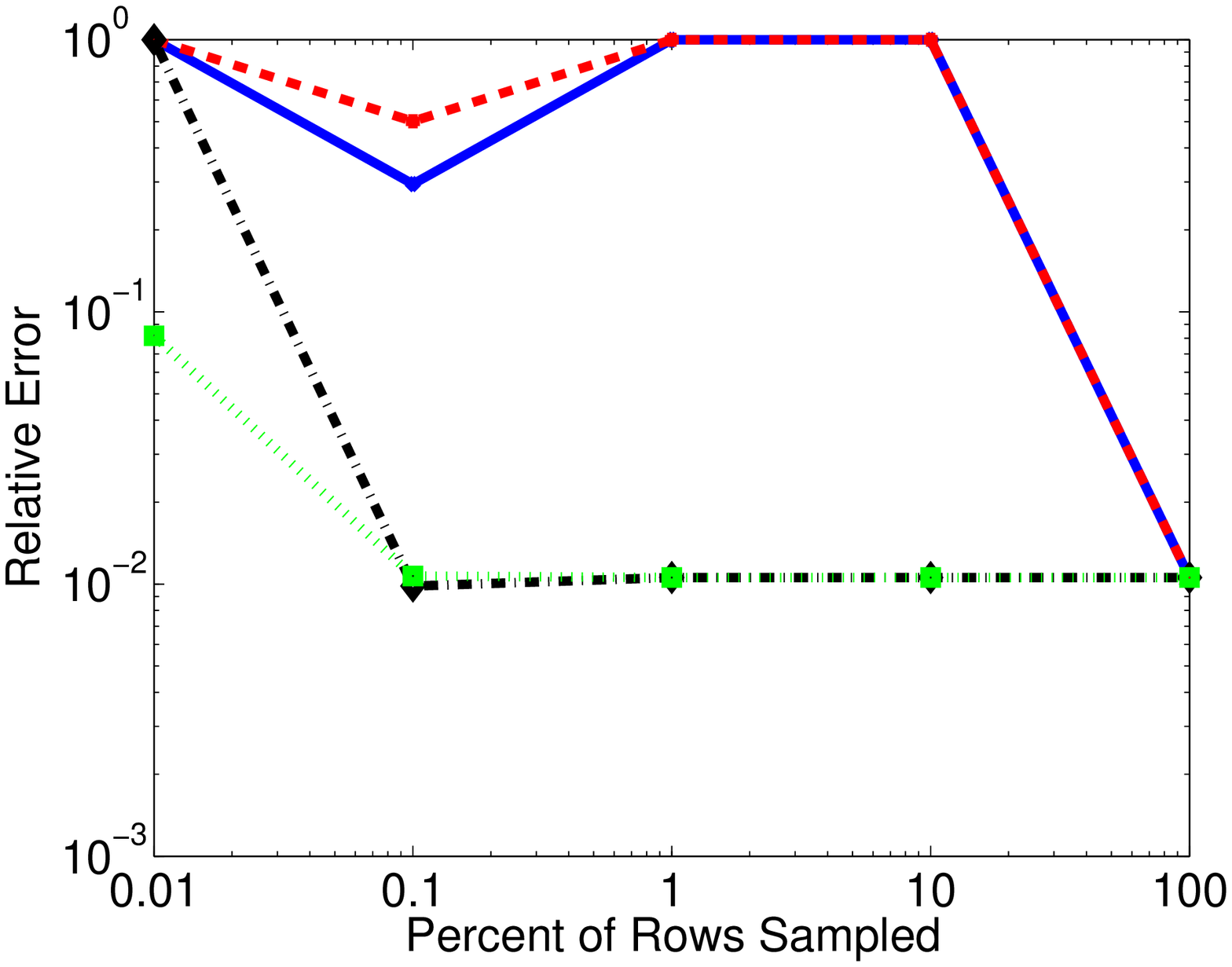}}
        \subfigure[$\kappa = +20\%$]{\includegraphics[width=0.4\textwidth]{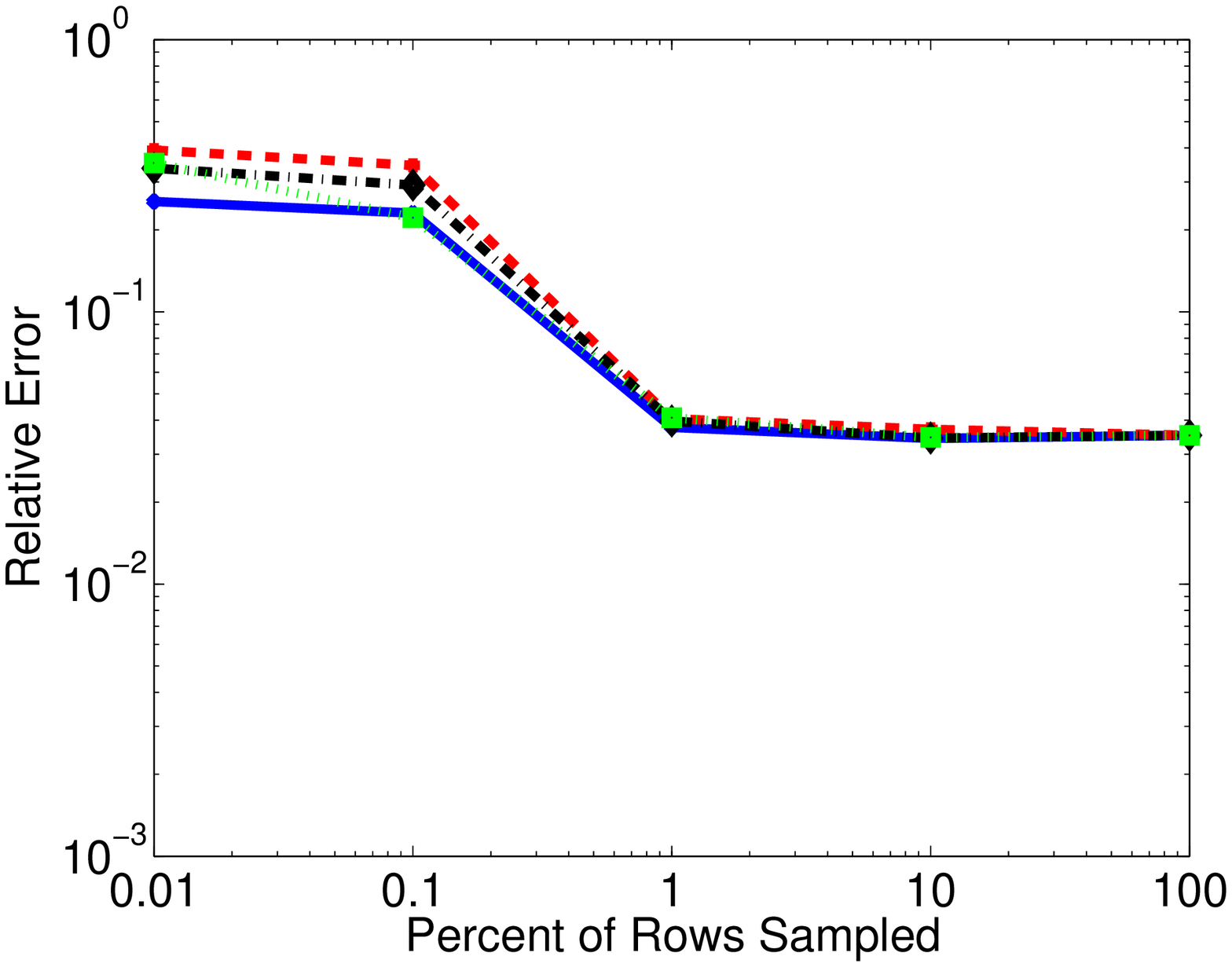}}
\caption{\textbf{ID compression; Gaussian kernel; normal data, \boldmath$d = 64$\unboldmath.}
We show the approximation error of the ID obtained from a subsampled matrix 
$\subK$. 
	We draw $N = 10^5$ points from the 64-dimensional standard normal 	
	distribution and set $n = 500$ and $x_c$ at the origin.
  We use the bandwidth given in the subfigure captions and Table 
 \ref{table_normal_h_values}. We set $\epsilon = 10^{-2}$ and choose 
 the rank $r$ so that it is the smallest $r$ such that 
 $\sigma_{r+1}(K)/\sigma_1(K) < \epsilon$.
  Each trend line represents a
  different subsampling method, with \underline{blue for the uniform
  distribution}, \underline{red for distances}, \underline{black for leverage}, and \underline{green for
  the deterministic selection of nearest neighbors}. 
\label{fig_gaussian_subsampling_id_d_64}}
\end{figure}

\subsubsection{Low intrinsic dimensions}

We show results on our artificial distribution with low intrinsic 
dimension in Figure~\ref{fig_gaussian_subsampling_low_intrinsic}. 
In our experiments, we see that despite the extremely high ambient dimension, 
we are still able to compute an accurate approximation using a subsample of the 
rows.  We also see qualitatively the same behavior as the results in 
Figure~\ref{fig_gaussian_subsampling_id_d_4}. Note that in 1,000 dimensions,  
existing methods will be prohibitively 
expensive. Without any \emph{a priori} information about the low 
dimensional structure, our method is able to efficiently 
compute an outgoing representation.

\begin{figure}
        \centering
        \subfigure[$\kappa = -1\%$]{\includegraphics[width=0.4\textwidth]{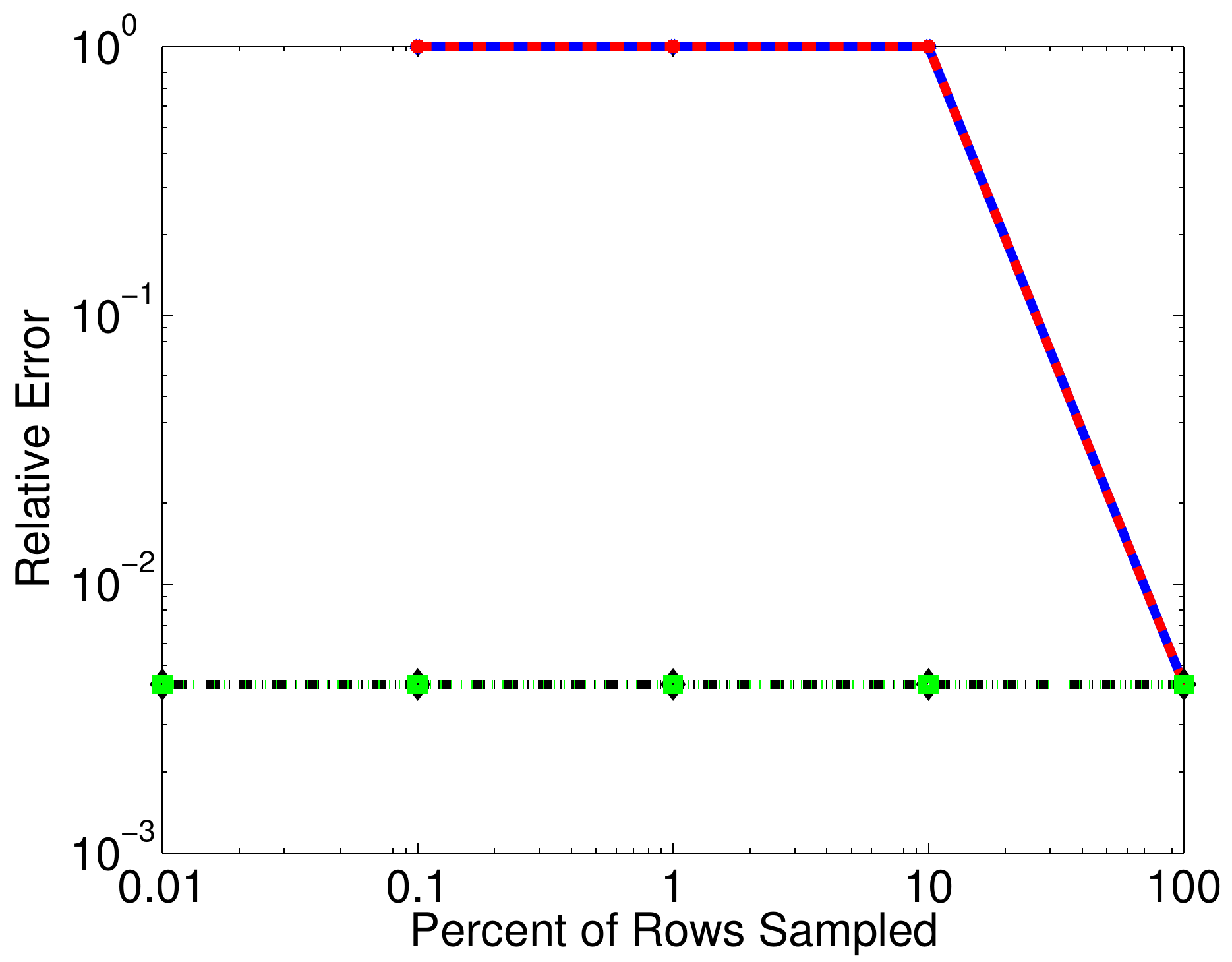}}
        \subfigure[$\kappa = -20\%$]{\includegraphics[width=0.4\textwidth]{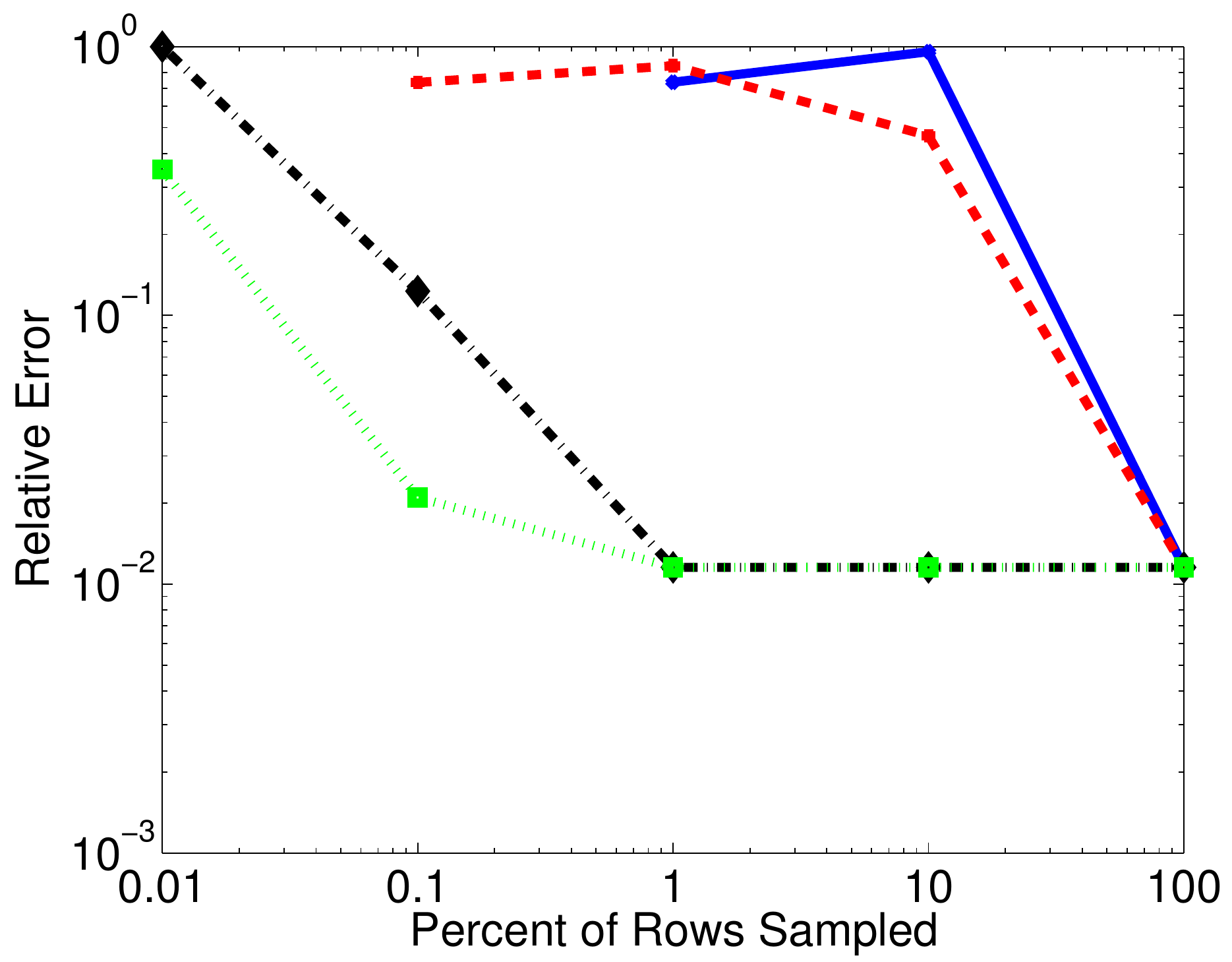}}
        \subfigure[$h = h_S$]{\includegraphics[width=0.4\textwidth]{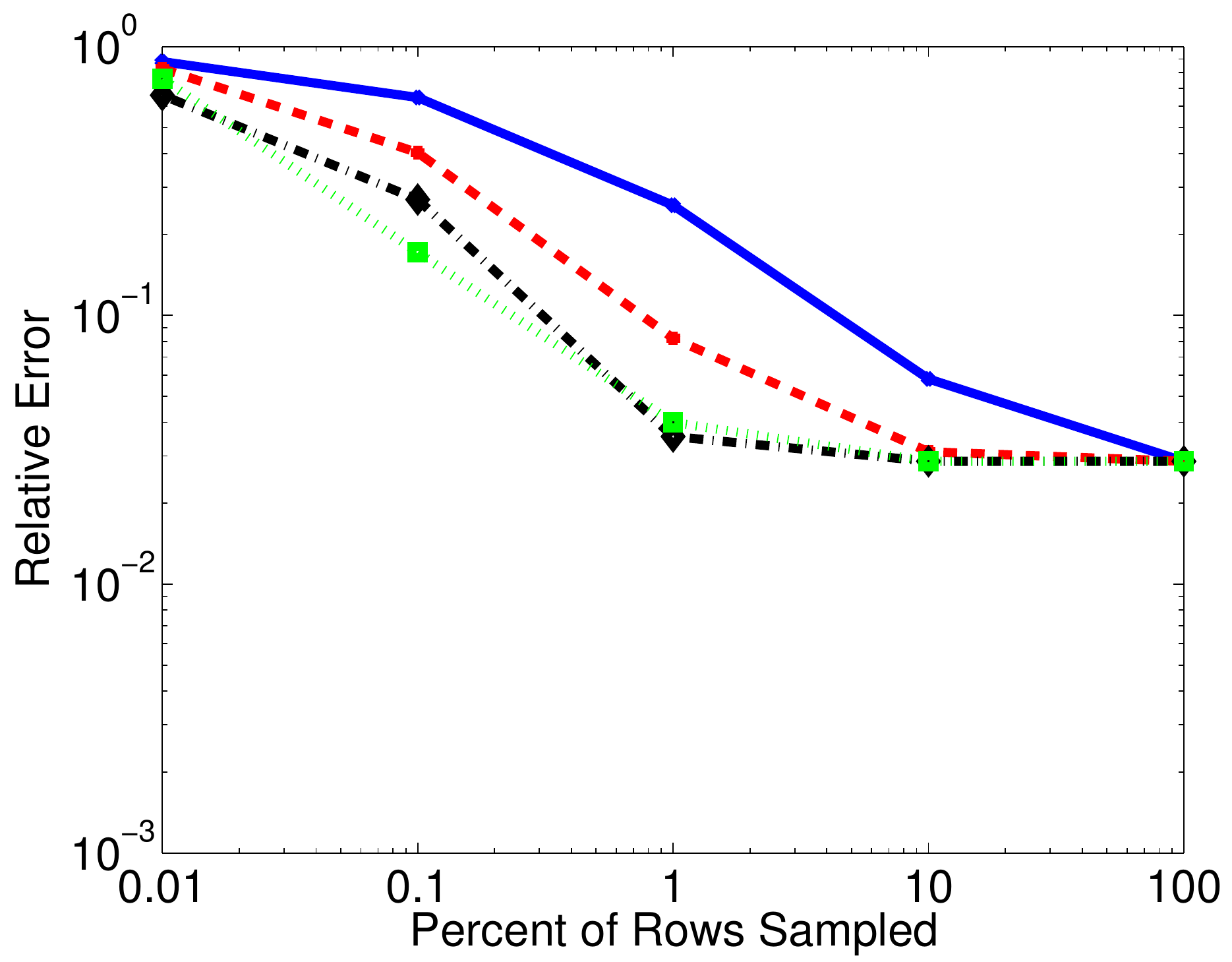}}
        \subfigure[$\kappa = +20\%$]{\includegraphics[width=0.4\textwidth]{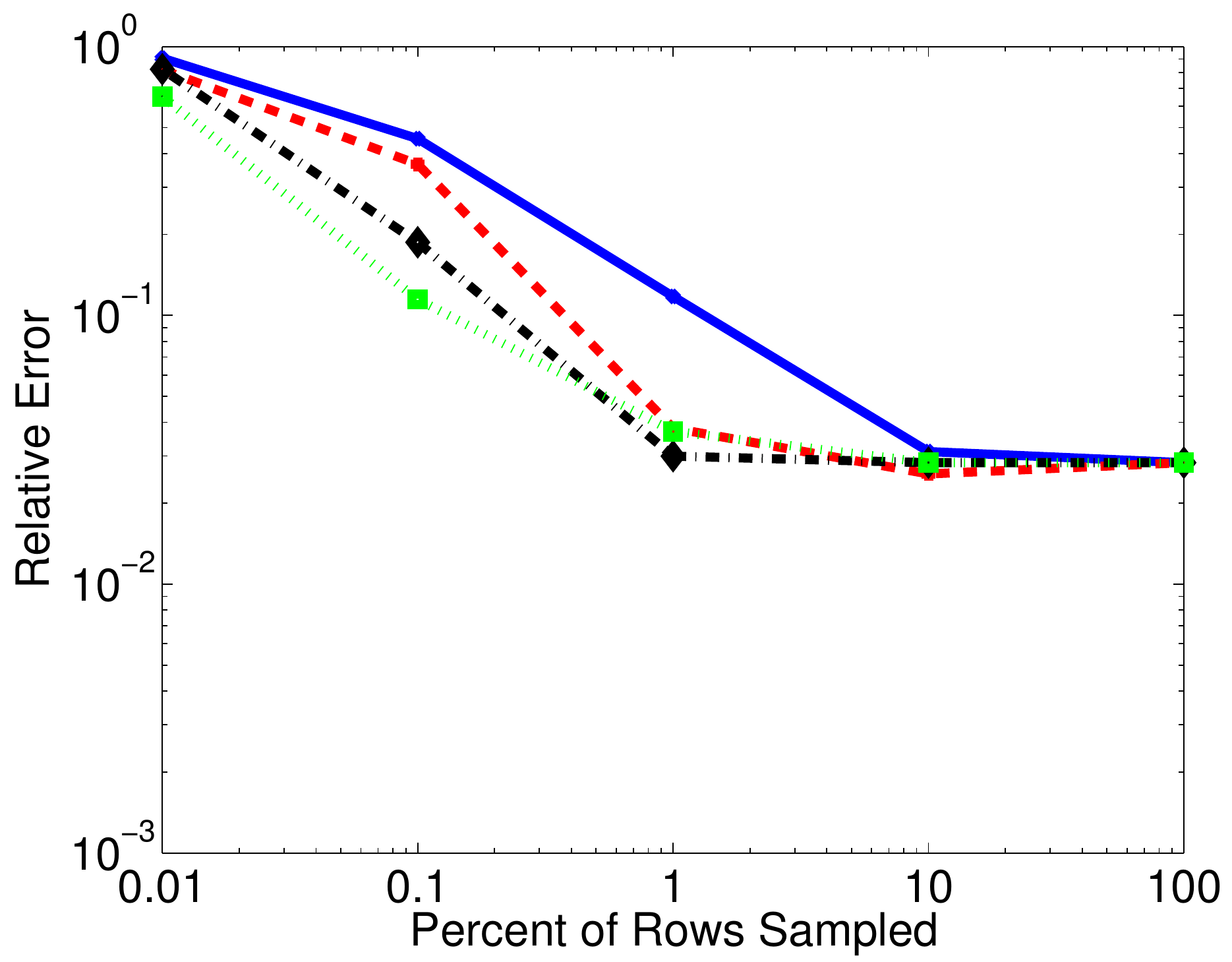}}
\caption{\textbf{ID compression; Gaussian kernel; low-dimensional data, 
\boldmath$d_i = 4, d_e = 1000$\unboldmath.}
We show the approximation error of the ID obtained from a subsampled matrix 
$\subK$. 
	We draw $N = 10^5$ points from our low-dimensional data distribution with intrinsic dimension 4 and ambient dimension 1,000. We set $n = 500$, and $x_c$
  is at the origin.
  We use the bandwidth given in the subfigure captions and Table 
 \ref{table_normal_h_values}. We set $\epsilon = 10^{-2}$ and choose 
 the rank $r$ so that it is the smallest $r$ such that 
 $\sigma_{r+1}(K)/\sigma_1(K) < \epsilon$.
  Each trend line represents a
  different subsampling method, with \underline{blue for the uniform
  distribution}, \underline{red for distances}, \underline{black for leverage}, and \underline{green for
  the deterministic selection of nearest neighbors}. 
\label{fig_gaussian_subsampling_low_intrinsic}}
\end{figure}

\subsubsection{UCI datasets}

We show results for UCI data sets in 
Figures~\ref{fig_gaussian_subsampling_color_hist} 
and \ref{fig_gaussian_subsampling_cooc_texture}.
We determine the bandwidth by direct
experimentation.  As in our other experiments, we choose a targeted
$\epsilon$-rank, then vary $h$ until we achieve this rank.  We select a 
point at random to be $x_c$. Our
results (averaged over independent choices of $x_c$) are shown in 
Table~\ref{table_real_data_h_values}.
Once again, the leverage sampling and nearest neighbor methods can accurately 
reconstruct the matrix from $1\%$ of its rows.  For the smaller bandwidth 
shown, the uniform and distance sampling distributions do not provide an 
accurate reconstruction, but they perform comparably to the other methods for 
larger $h$.

\begin{table}[htb]
\centering
\caption{Properties of data sets from the UCI ML repository \cite{Bache+Lichman:2013}. We give 
values of $h$ in units of $h_S$. We give the sample mean 
over 30 independent choices of $x_c$, with sample standard deviations in 
parentheses. The final column lists the value of $h_S$ for the set's values of 
$d$ and $N$ for comparison.
\label{table_real_data_h_values}}
\begin{tabular}{|l|c|c|c|c|c|} \hline
Data & $d$ & $N$ & $\kappa = -20\%$ & $\kappa = +20\%$ & $h_S$ \\ \hline
Color hist. & 32 & 68040 & 0.0481 (0.0086) & 0.2184 (0.0296) & 0.6794 \\
Cooc texture & 16 & 68040 & 0.0266 (0.0087) & 0.1318 (0.0293) & 0.5159 \\ \hline
\end{tabular}
\end{table}

We see similar behavior as in our synthetic data experiments.  For larger 
values of $h$, all of our row selection methods are effective with roughly 
$1\%$ of the rows. For small values of $h$, the nearest neighbor and leverage 
sampling methods achieve high accuracy, while the other sampling distributions 
do not.

\begin{figure}
        \centering
        \subfigure[$\kappa = 20\%$]{\includegraphics[width=0.4\textwidth]{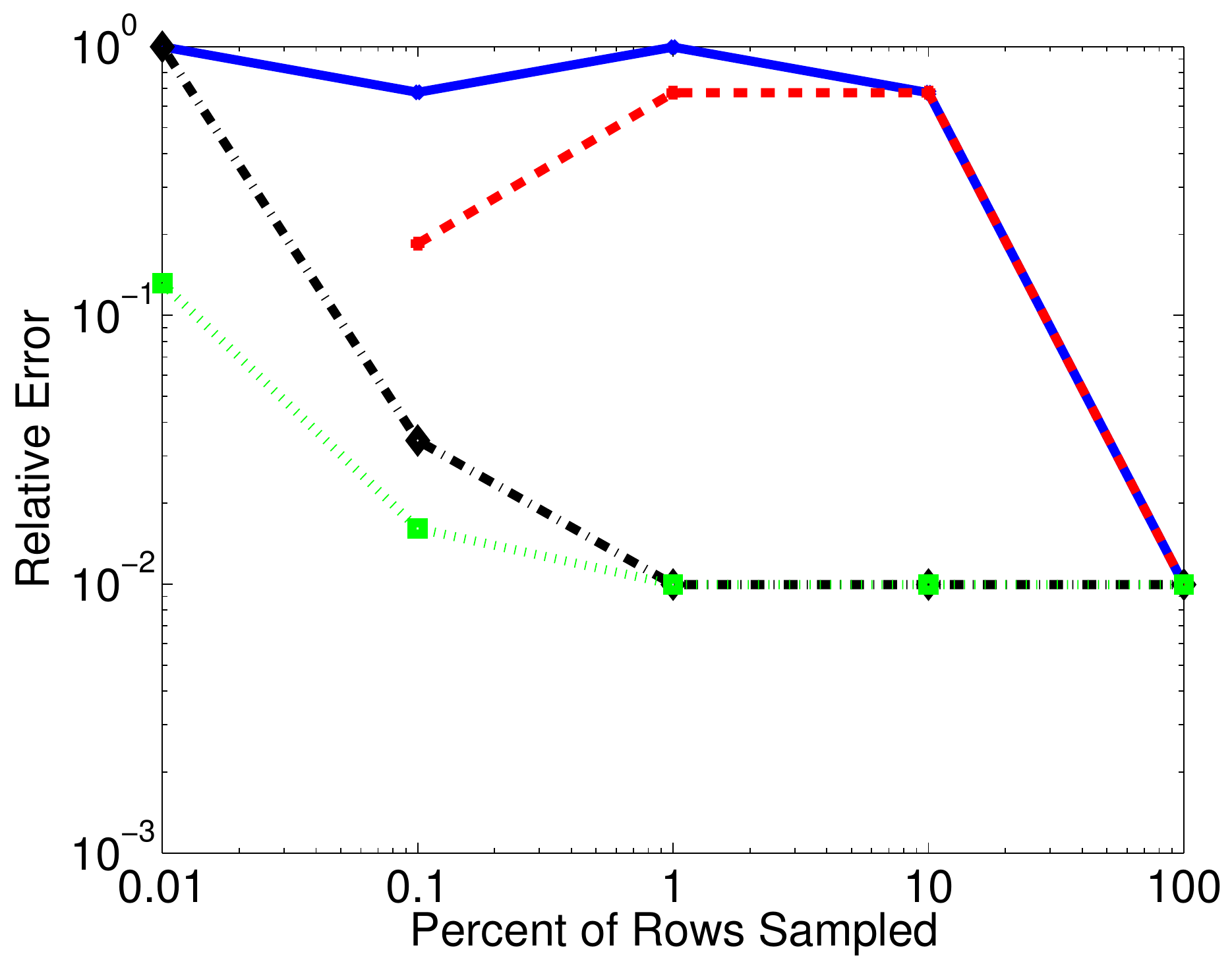}}
        \subfigure[$\kappa = +20\%$]{\includegraphics[width=0.4\textwidth]{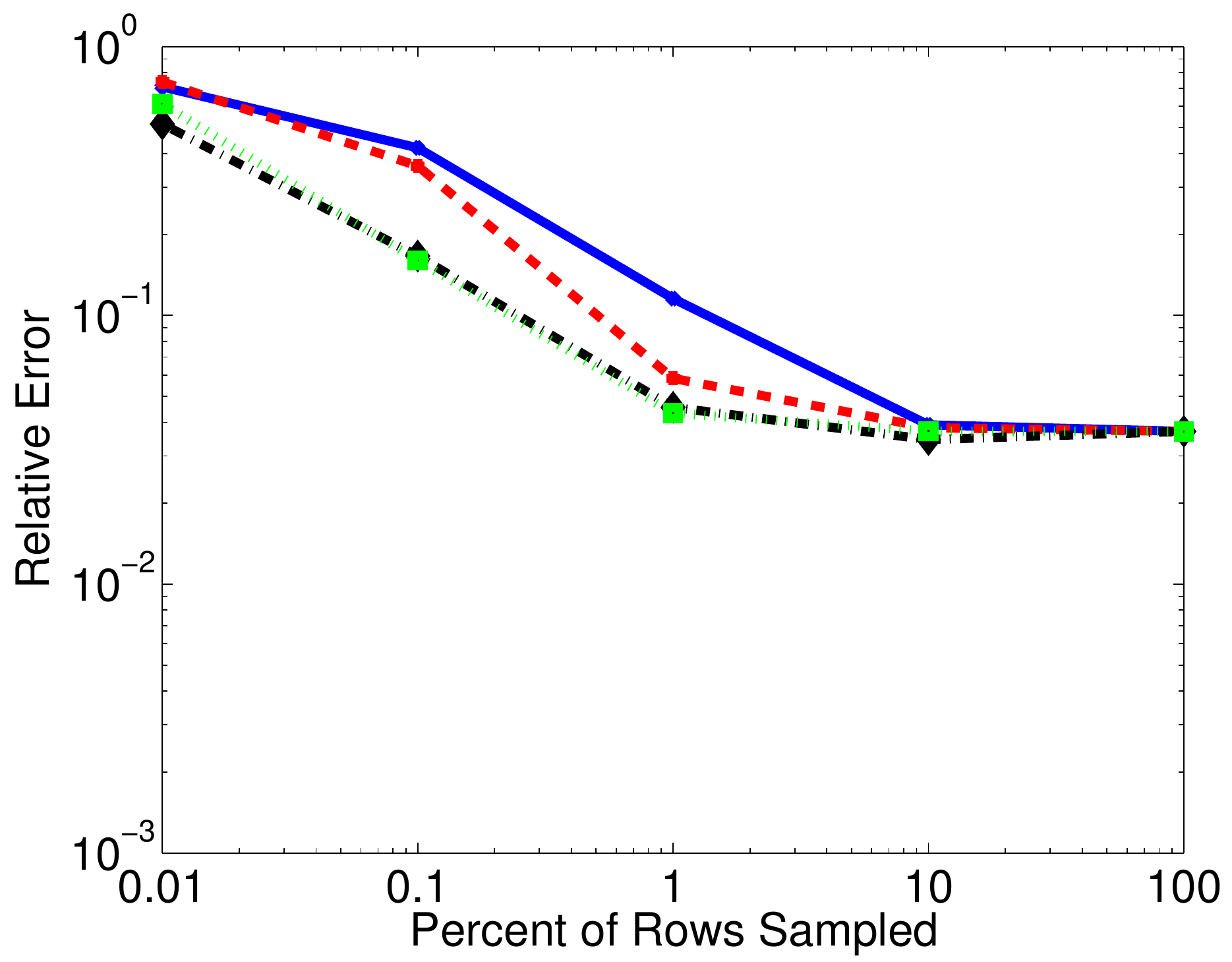}}
\caption{\textbf{ID compression; Gaussian kernel; Color Histogram data.}
We show the approximation error of the ID obtained from a subsampled matrix 
$\subK$. We use the Color Histogram data set, with $N = 68,040$ and $d = 32$. 
We choose $x_c$ uniformly at random and set $n = 500$.
  We use the bandwidth given in the subfigure captions and Table 
 \ref{table_real_data_h_values}. We set $\epsilon = 10^{-2}$ and choose 
 the rank $r$ so that it is the smallest $r$ such that 
 $\sigma_{r+1}(K)/\sigma_1(K) < \epsilon$.
  Each trend line represents a
  different subsampling method, with \underline{blue for the uniform
  distribution}, \underline{red for distances}, \underline{black for leverage}, and \underline{green for
  the deterministic selection of nearest neighbors}. 
\label{fig_gaussian_subsampling_color_hist}}
\end{figure}


\begin{figure}
        \centering
        \subfigure[$r = 20\%$]{\includegraphics[width=0.4\textwidth]{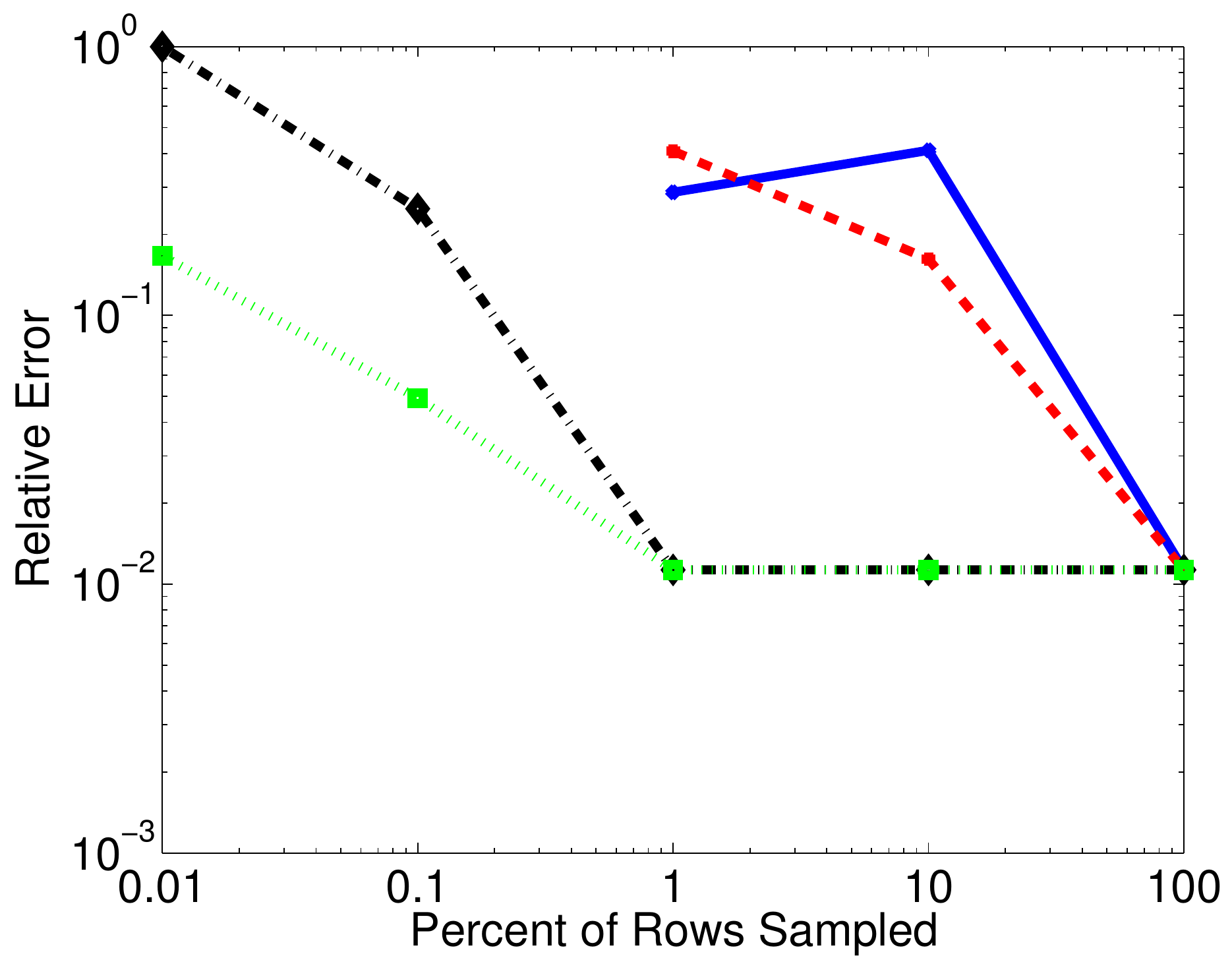}}
        \subfigure[$r = +20\%$]{\includegraphics[width=0.4\textwidth]{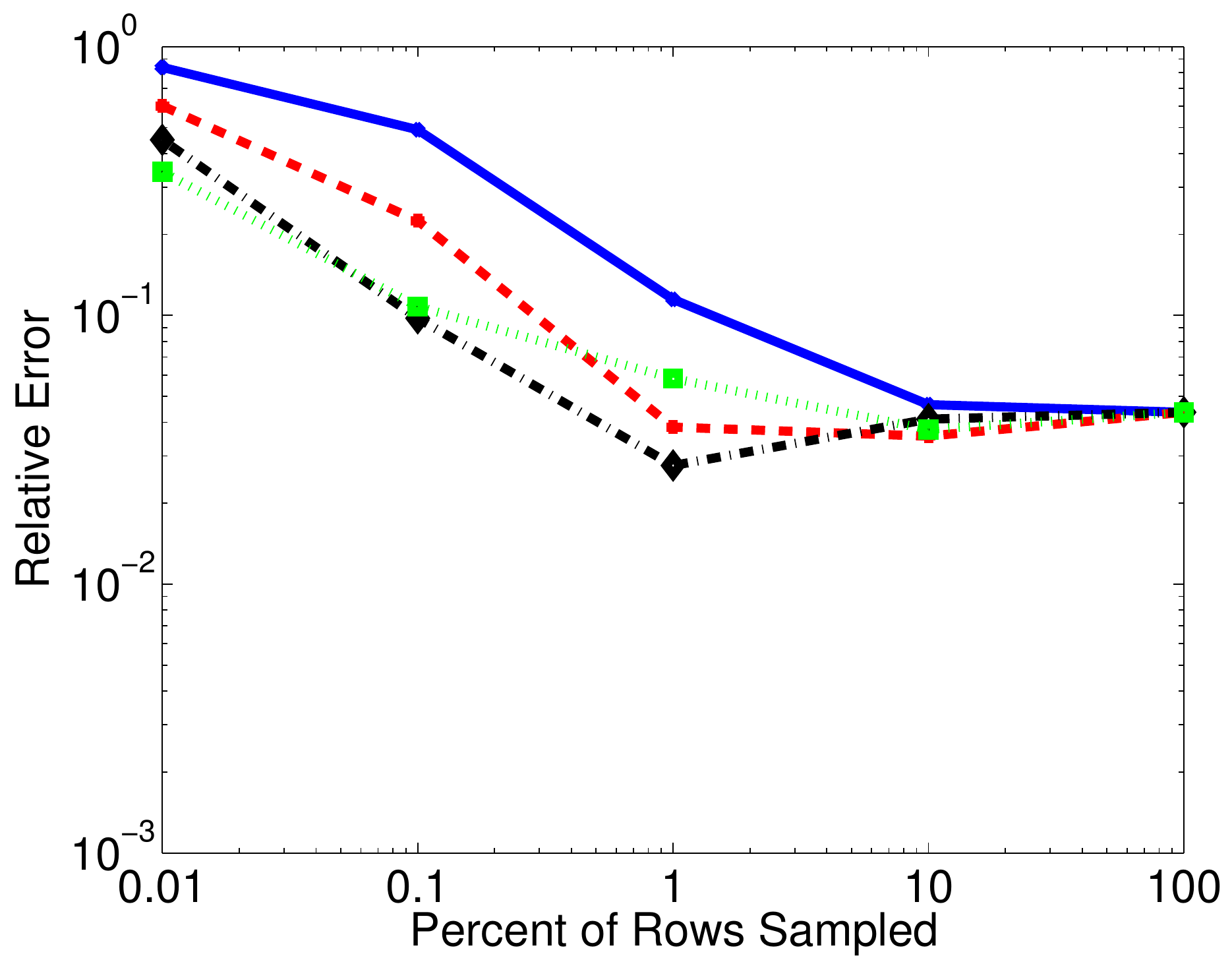}}
        \caption{\textbf{ID compression; Gaussian kernel; Cooc Texture data.}
        We show the approximation error of the ID obtained from a subsampled matrix
        $\subK$. We use the Cooc Texture data set, with $N = 68,040$ and $d =
        16$. We choose $x_c$ uniformly at
        random and set $n = 500$.
          We use the bandwidth given in the subfigure captions and Table
         \ref{table_real_data_h_values}. We set $\epsilon = 10^{-2}$ and choose
         the rank $r$ so that it is the smallest $r$ such that
         $\sigma_{r+1}(K)/\sigma_1(K) < \epsilon$.
          Each trend line represents a
          different subsampling method, with \underline{blue for the uniform
          distribution}, \underline{red for distances}, \underline{black for leverage}, and \underline{green for
          the deterministic selection of nearest neighbors}.
\label{fig_gaussian_subsampling_cooc_texture}}
\end{figure}

 \subsection{Laplace kernel}

The Laplace kernel is given by:
\begin{equation}
\Ker(y, x) = \begin{cases}
\log \|x - y\| & d = 2 \\
\|x - y\|^{2 - d} & d \not= 2.
\end{cases}
\end{equation}
The Laplace kernel lacks any parameters other than the dimension of the
inputs. However, unlike the Gaussian kernel, it has a singularity at $x = y$.

\subsubsection{Choice of parameters}

The only parameter we need for our experiments 
is the well-separateness parameter $\xi$. In 
series-expansion based methods for this kernel, some separation between the 
sources and targets is required for the series to converge.  This is due to the 
singularity in the kernel function as the distance between its arguments goes 
to zero. Therefore, we examine values of $\xi$ that are strictly greater than 
one. 

We also examine smaller values of $d$ for our synthetic data experiments.  This 
is because $r^{-d+2}$ will quickly go to zero for larger values of $d$. 

\subsubsection{Spectrum of $K$}

We explore the compression of $K$ for $\xi = 2$ and $4$ in Figure
\ref{fig_laplace_spectra_normal_data}.
The singular values decay more quickly for $\xi = 4$.  However, for both
values, the Laplace kernel submatrix compresses more effectively than for the
Gaussian kernel.  We see that the use of the ID as an outgoing
representation is feasible for this kernel.

\begin{figure}
        \centering
        \subfigure[$\xi = 2$.]{\includegraphics[width=0.45\textwidth]{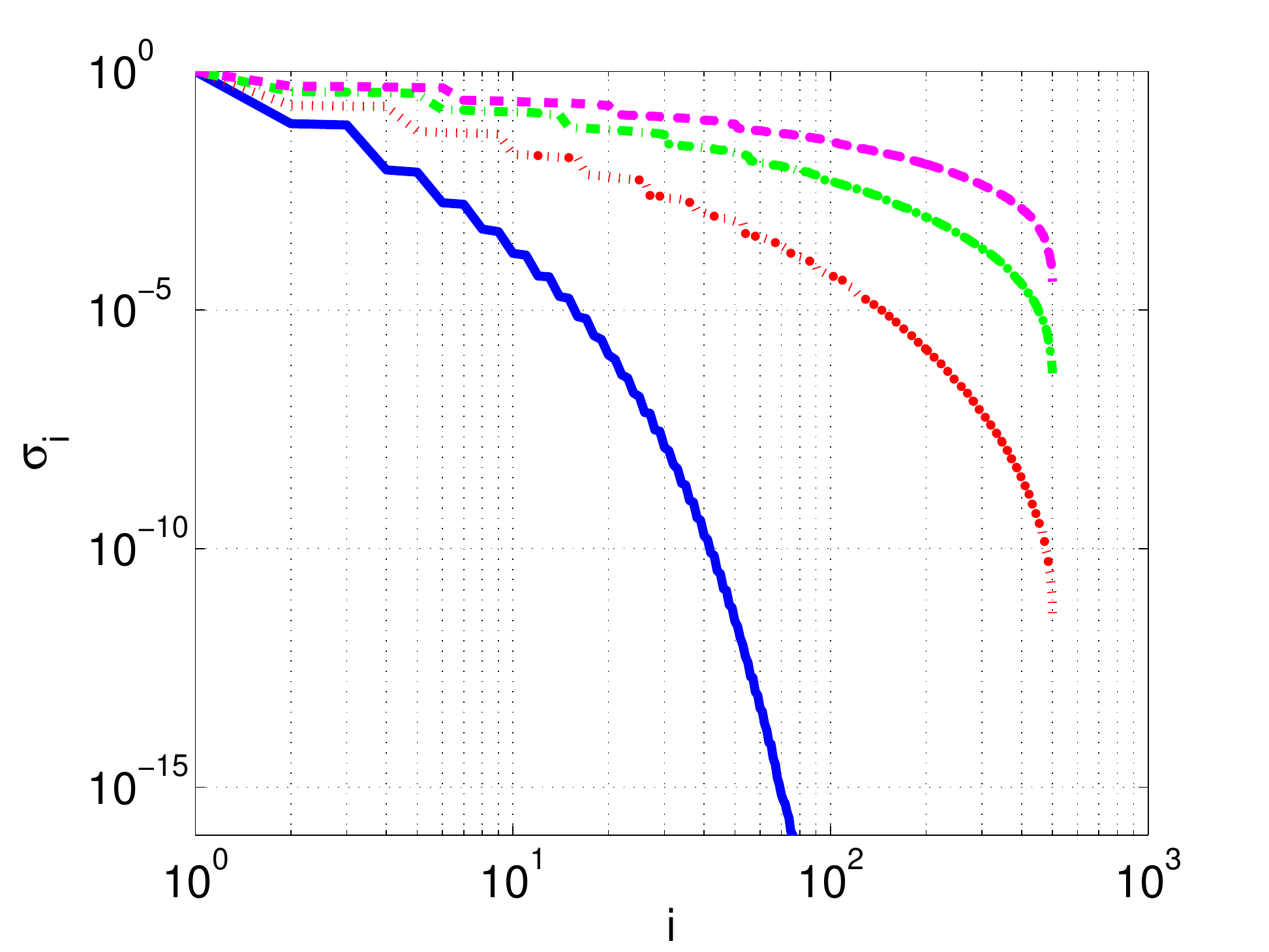}}
        \subfigure[$\xi = 4$.]{\includegraphics[width=0.45\textwidth]{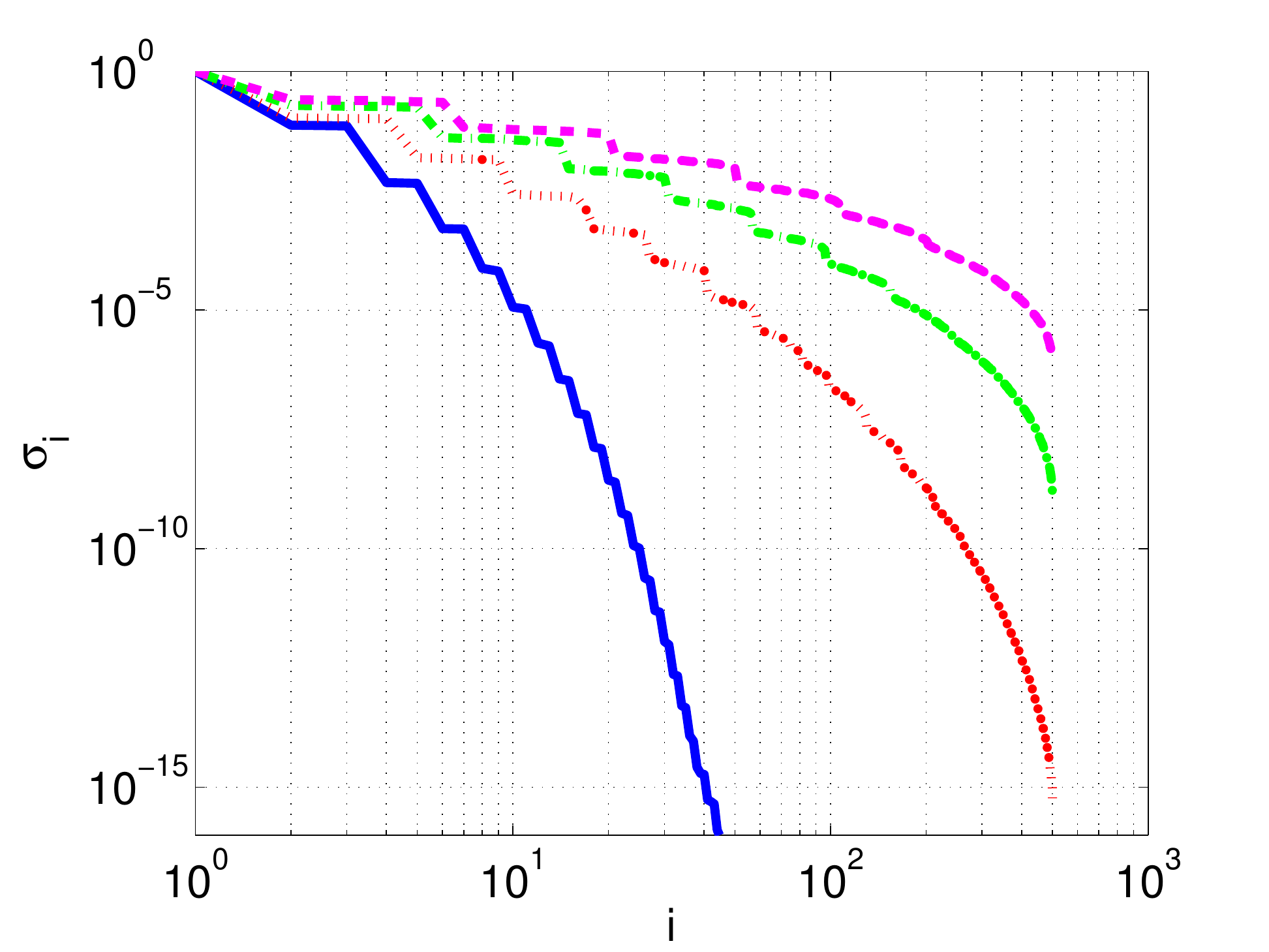}}
\caption{\textbf{Singular values of the Laplace kernel.}
We report the compressibility of the far field by computing the
singular values of $K$ for the Laplace kernel. We draw $N = 10^5$ points from
a standard normal distribution. We set $n = 500$
  and $\xi$ in the subfigure captions.
  The trend lines show $d = 2$ (blue), $d = 3$ (red), $d = 4$ (green), and $d =
  5$ (magenta).
\label{fig_laplace_spectra_normal_data}}
\end{figure}

\subsubsection{Subsampling}

We show subsampling results for the Laplace kernel in Figure 
\ref{fig_laplace_subsampling_ws_2} for $\xi = 2$ and 
Figure~\ref{fig_laplace_subsampling_ws_4} 
for $\xi = 4$.  We note that for all four 
values of $d$, the four methods for selecting rows perform very similarly.  In 
all cases, $1\%$ of the rows are sufficient to form an approximation that is as 
accurate as the decomposition of the full matrix. 

\begin{figure}
        \centering
        \subfigure[$d = 2$]{\includegraphics[width=0.4\textwidth]{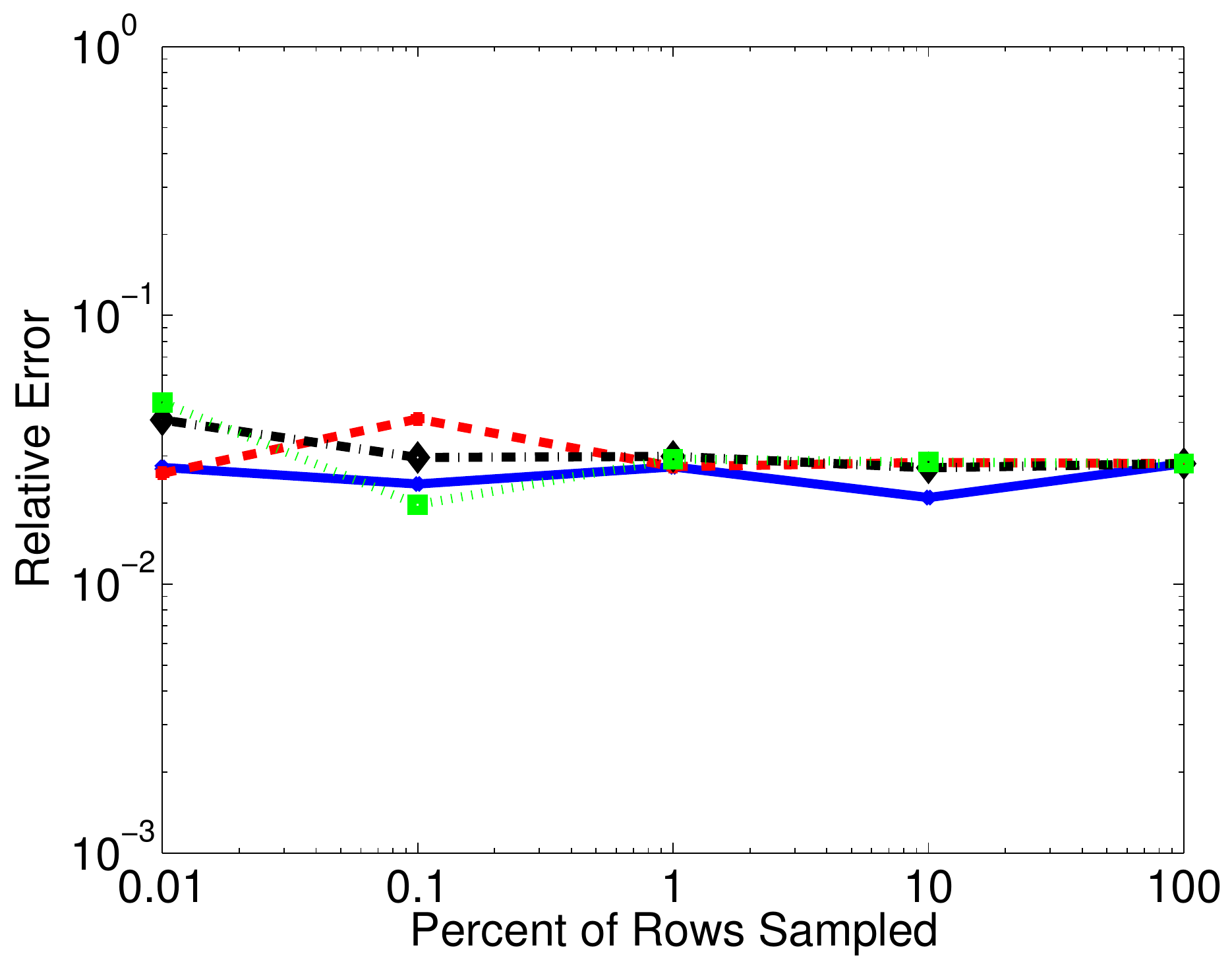}}
        \subfigure[$d = 3$]{\includegraphics[width=0.4\textwidth]{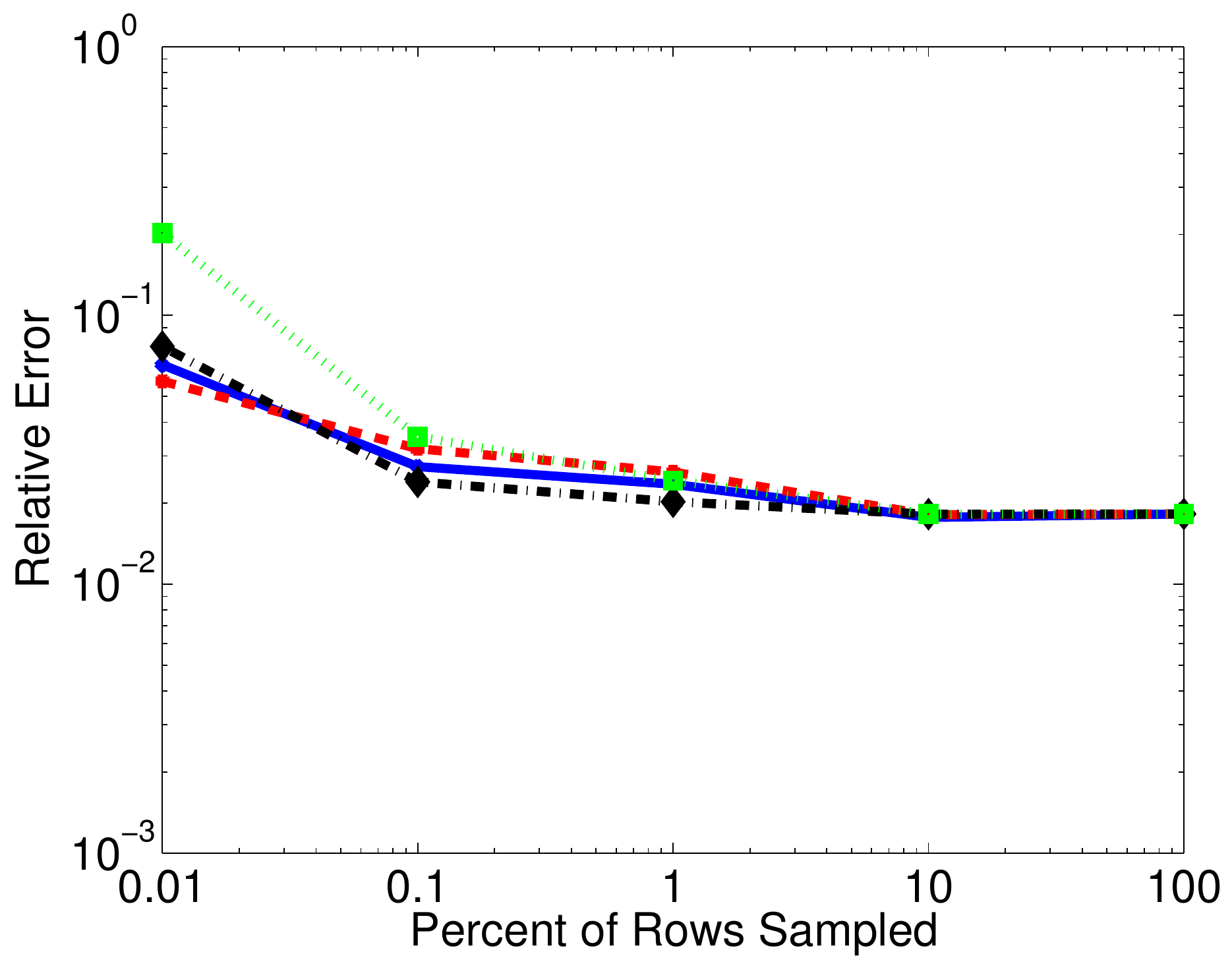}}
        \subfigure[$d = 4$]{\includegraphics[width=0.4\textwidth]{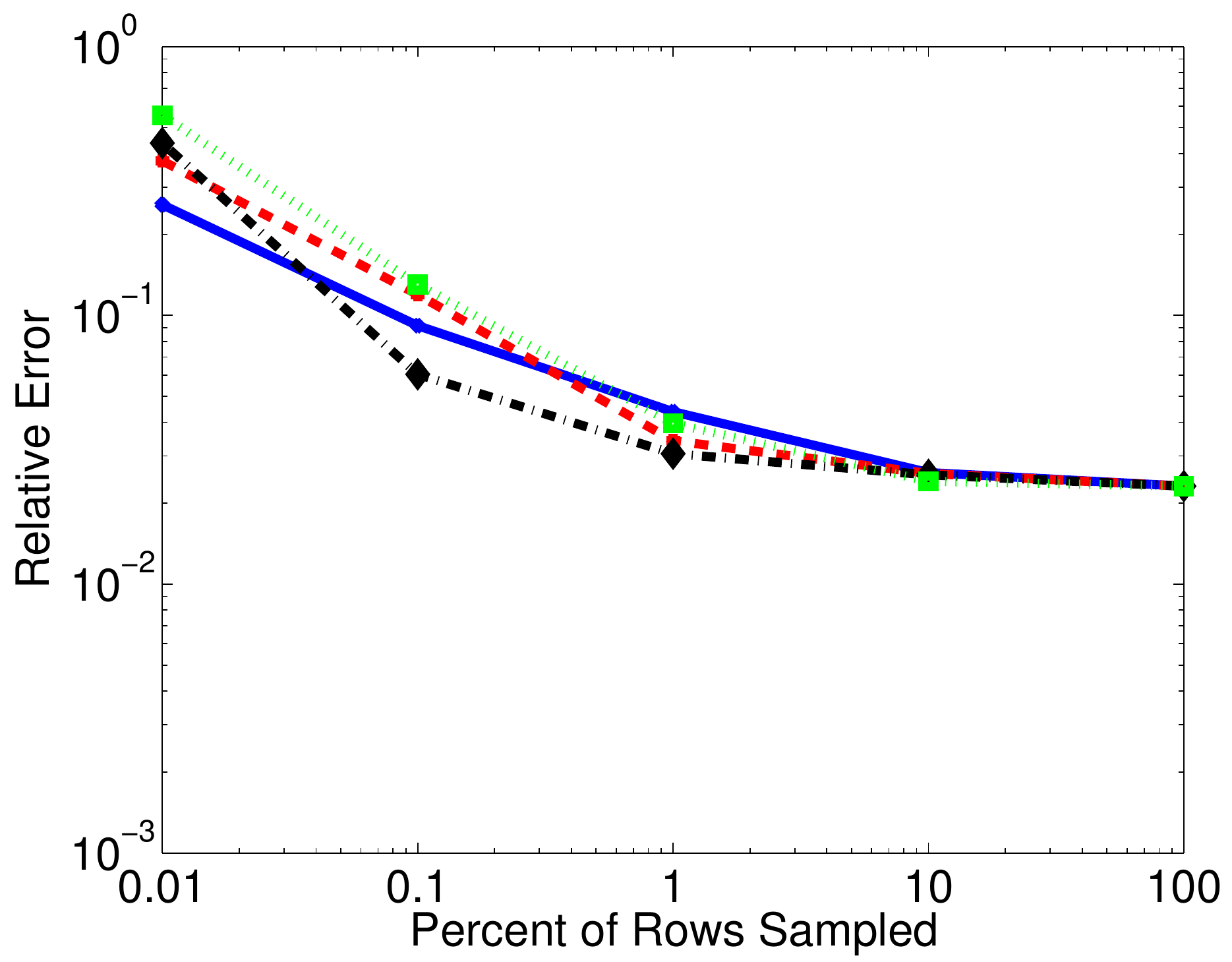}}
        \subfigure[$d = 5$]{\includegraphics[width=0.4\textwidth]{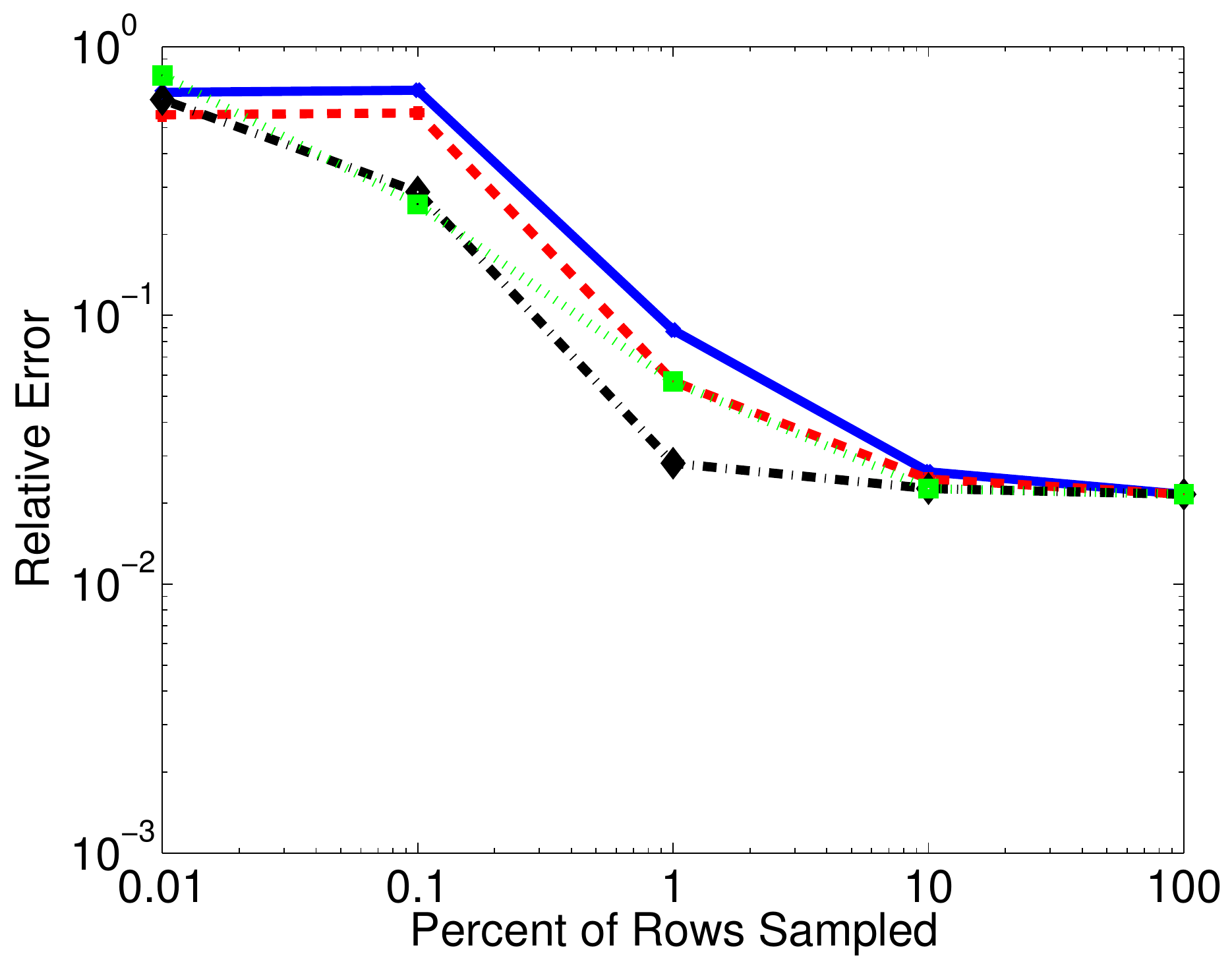}}
\caption{\textbf{ID compression; Laplace kernel; normal data.}
We show the approximation error of the ID obtained from a subsampled matrix 
$\subK$. 
	We draw $N = 10^5$ points from the $d$-dimensional standard normal 	
	distribution ($d$ in subfigure captions) and set $n = 500$, $\xi = 2$,
  and $x_c$ at the origin.
  We set $\epsilon = 10^{-2}$ and choose 
   the rank $r$ so that it is the smallest $r$ such that 
   $\sigma_{r+1}(K)/\sigma_1(K) < \epsilon$.
  Each trend line represents a
  different subsampling method, with \underline{blue for the uniform
  distribution}, \underline{red for distances}, \underline{black for leverage}, and \underline{green for
  the deterministic selection of nearest neighbors}. 
\label{fig_laplace_subsampling_ws_2}}
\end{figure}

\begin{figure}
       \centering
       \subfigure[$d = 2$]{\includegraphics[width=0.45\textwidth]{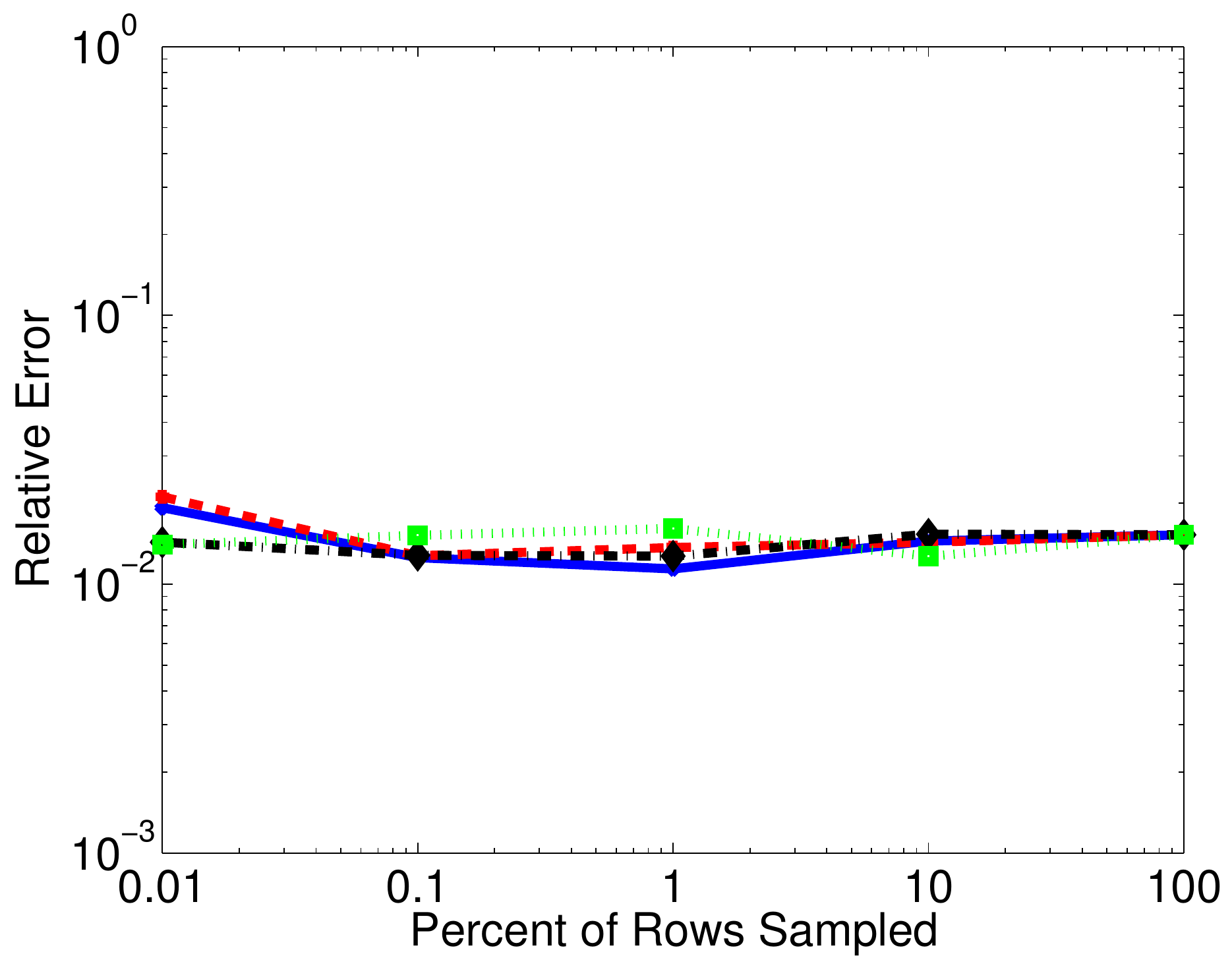}}
       \subfigure[$d = 3$]{\includegraphics[width=0.45\textwidth]{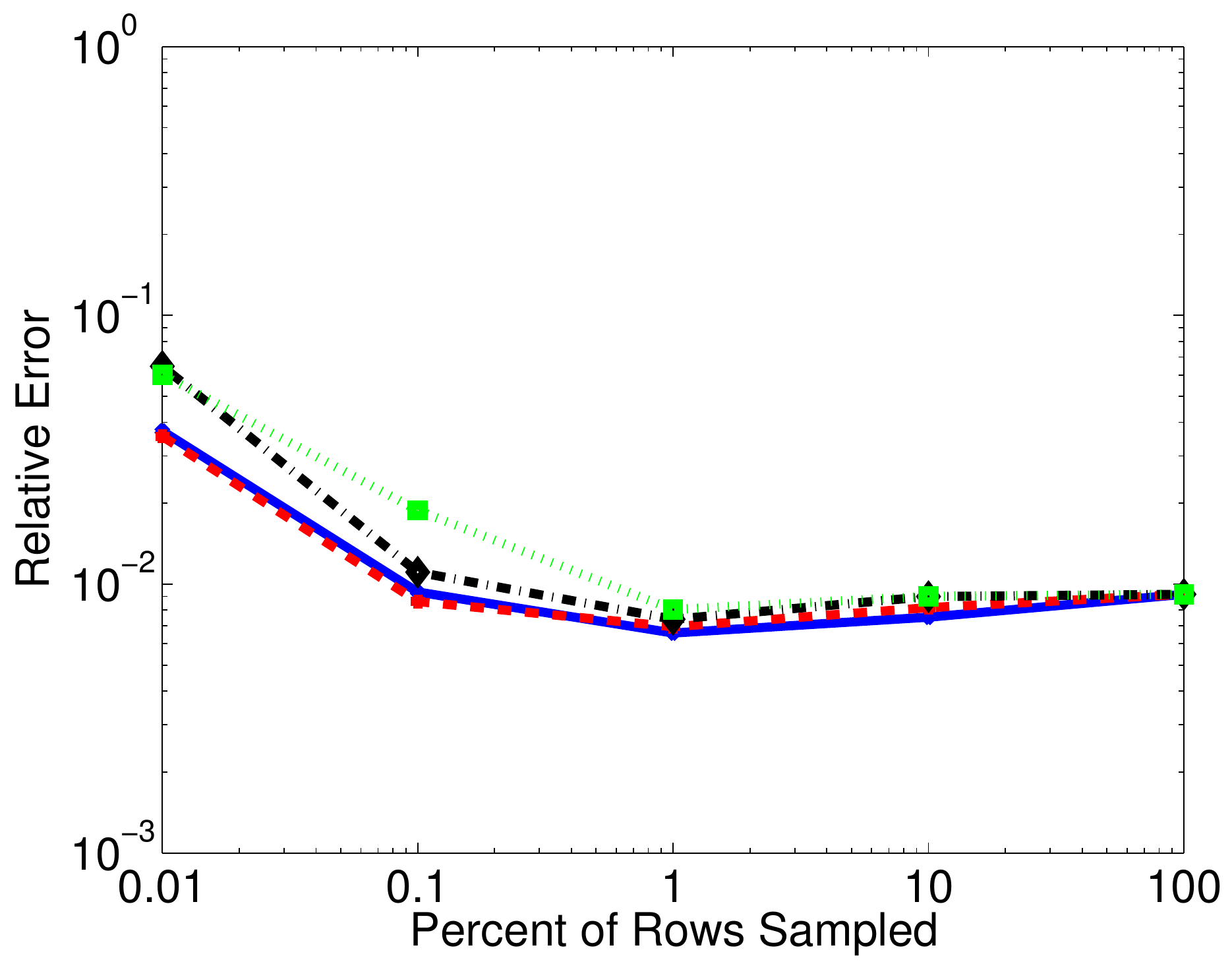}}
       \subfigure[$d = 4$]{\includegraphics[width=0.45\textwidth]{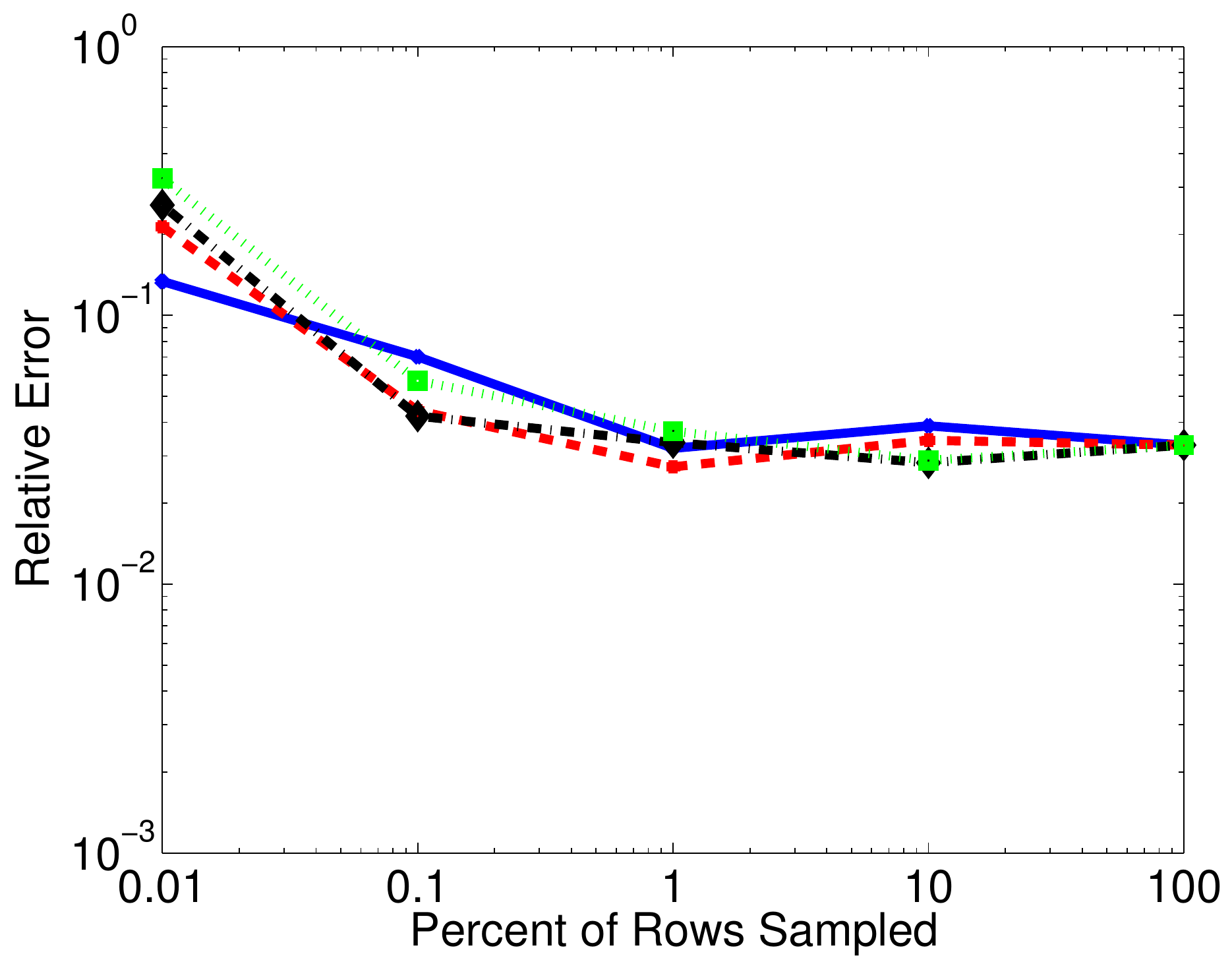}}
       \subfigure[$d = 5$]{\includegraphics[width=0.45\textwidth]{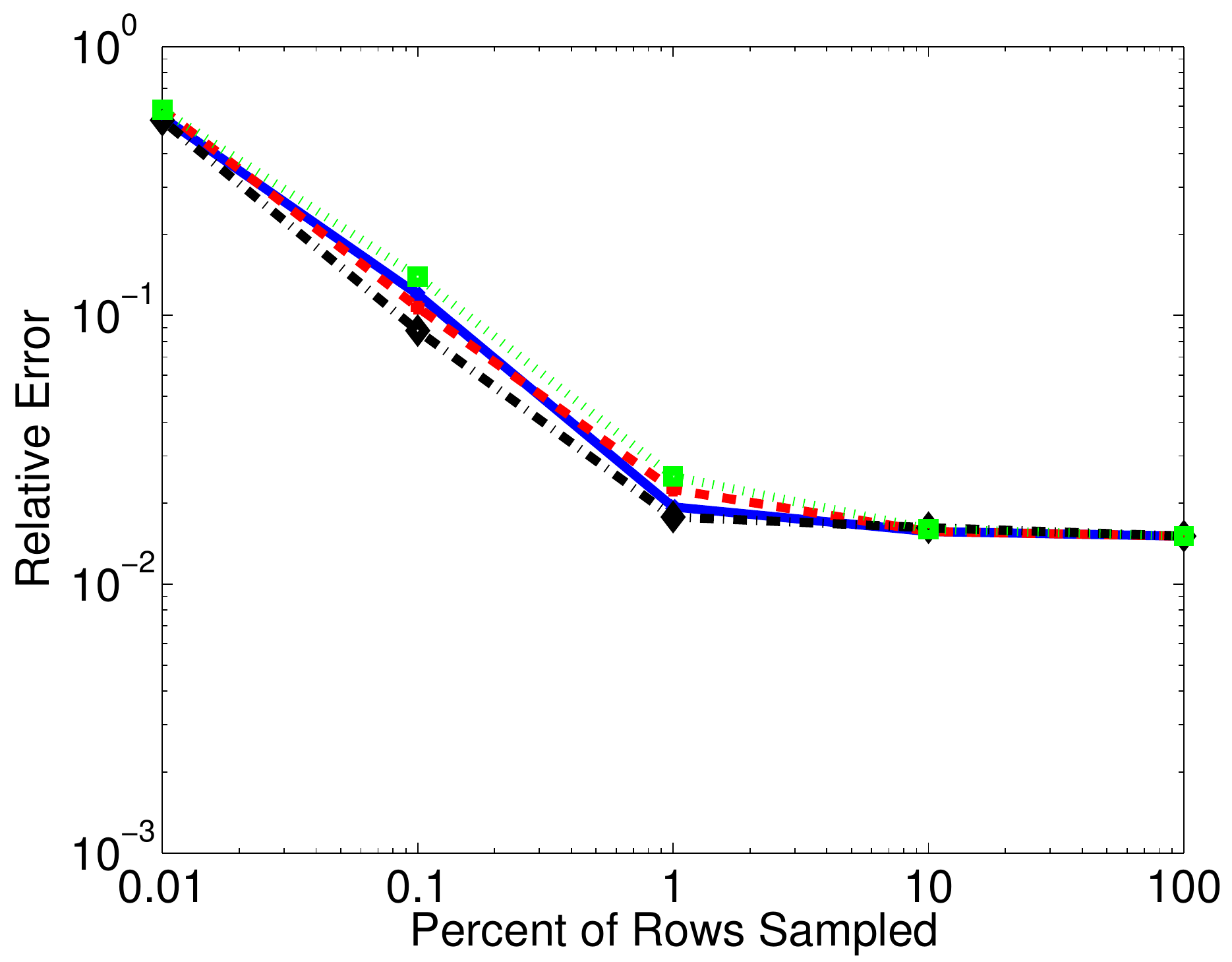}}
\caption{\textbf{ID compression; Laplace kernel; normal data.}
We show the approximation error of the ID obtained from a subsampled matrix 
$\subK$. 
	We draw $N = 10^5$ points from the $d$-dimensional standard normal 	
	distribution ($d$ in subfigure captions) and set $n = 500$, $\xi = 4$,
  and $x_c$ at the origin.
  We set $\epsilon = 10^{-2}$ and choose 
   the rank $r$ so that it is the smallest $r$ such that 
   $\sigma_{r+1}(K)/\sigma_1(K) < \epsilon$.
  Each trend line represents a
  different subsampling method, with \underline{blue for the uniform
  distribution}, \underline{red for distances}, \underline{black for leverage}, and \underline{green for
  the deterministic selection of nearest neighbors}. 
\label{fig_laplace_subsampling_ws_4}}
\end{figure}

\subsection{Polynomial kernel}

The polynomial kernel is defined as:
\begin{equation}
\Ker(y,x) = \left( \frac{x^T y}{h} + c \right)^p
\end{equation}
This kernel is characterized by three parameters: the degree $p$, bandwidth 
$h$, and a constant $c$.  However, the constant $c$ can be set to 1 without 
loss of generality \cite{chang2010training}.

\subsubsection{Choice of parameters}

We examine quadratic ($p = 2$) and cubic ($p = 3$) polynomial kernels. 
As $p$ increases, the kernel 
matrix will be dominated by the inner products of the largest magnitude data 
vectors. 
Unlike in the Gaussian case, we do not have any \emph{a priori} scale for the 
bandwidth. We therefore resort to direct experimentation to cover a wide range 
of values of $h$.

\subsubsection{Compression of kernel submatrices}

We begin by examining the singular values of the polynomial kernel submatrices.
We show results for the quadratic kernel in
Figure~\ref{fig_polynomial_p_2_spectra_normal_data} and cubic kernel in
Figure~\ref{fig_polynomial_p_3_spectra_normal_data} for a range of values of
$h$.

The most striking feature of these plots is the sharp drop-off in the spectrum
for most values of $d$ and $h$, but especially for smaller values of $d$.
We see that as $h$ increases, the spectrum decreases more sharply, since the
value of the kernel approaches one for all arguments as $h$ grows.
The spectra for the quadratic and cubic kernels are qualitatively similar.


\begin{figure}
        \centering
        \subfigure[$h = 0.01$.]{\includegraphics[width=0.3\textwidth]{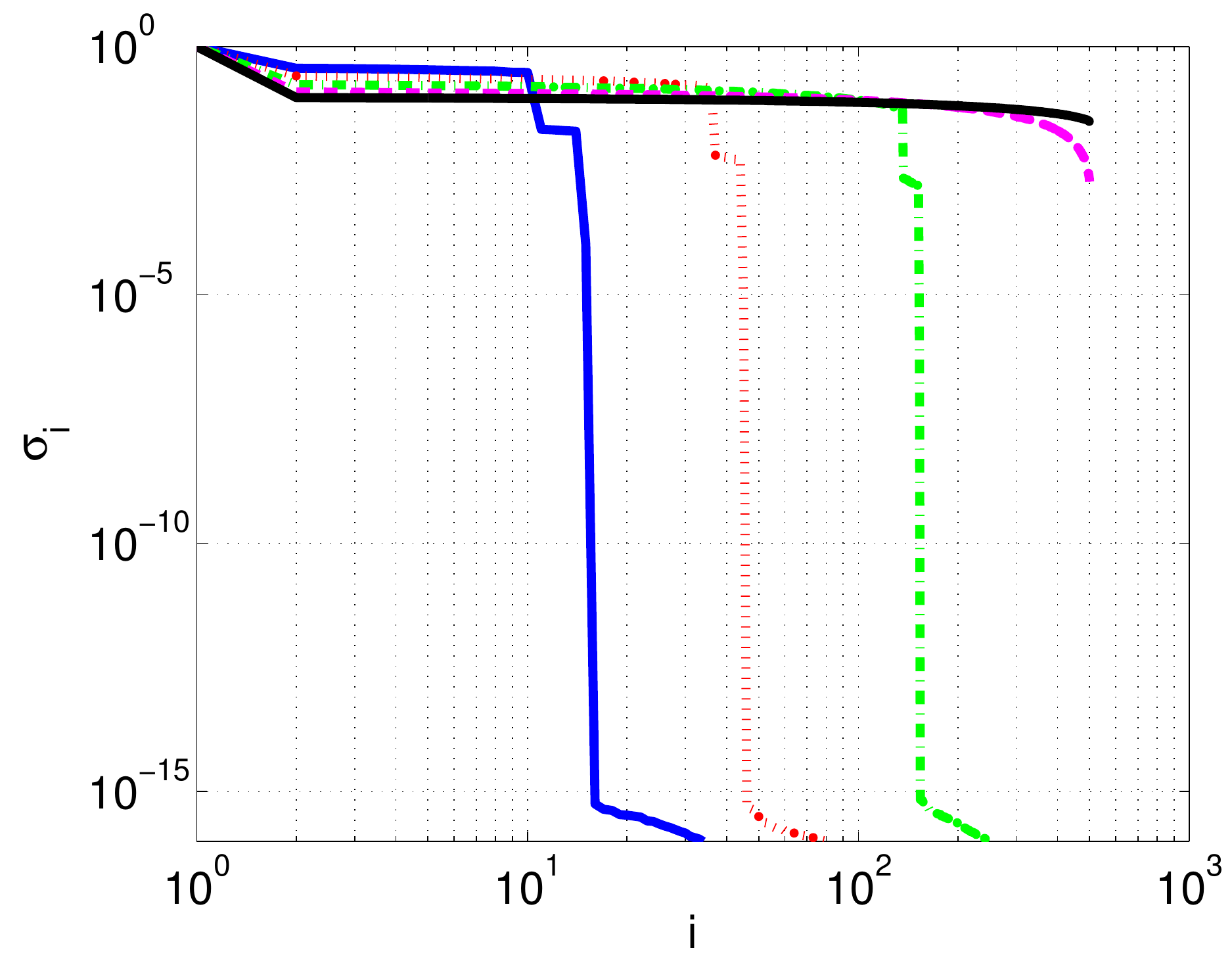}}
        \subfigure[$h = 0.1$.]{\includegraphics[width=0.3\textwidth]{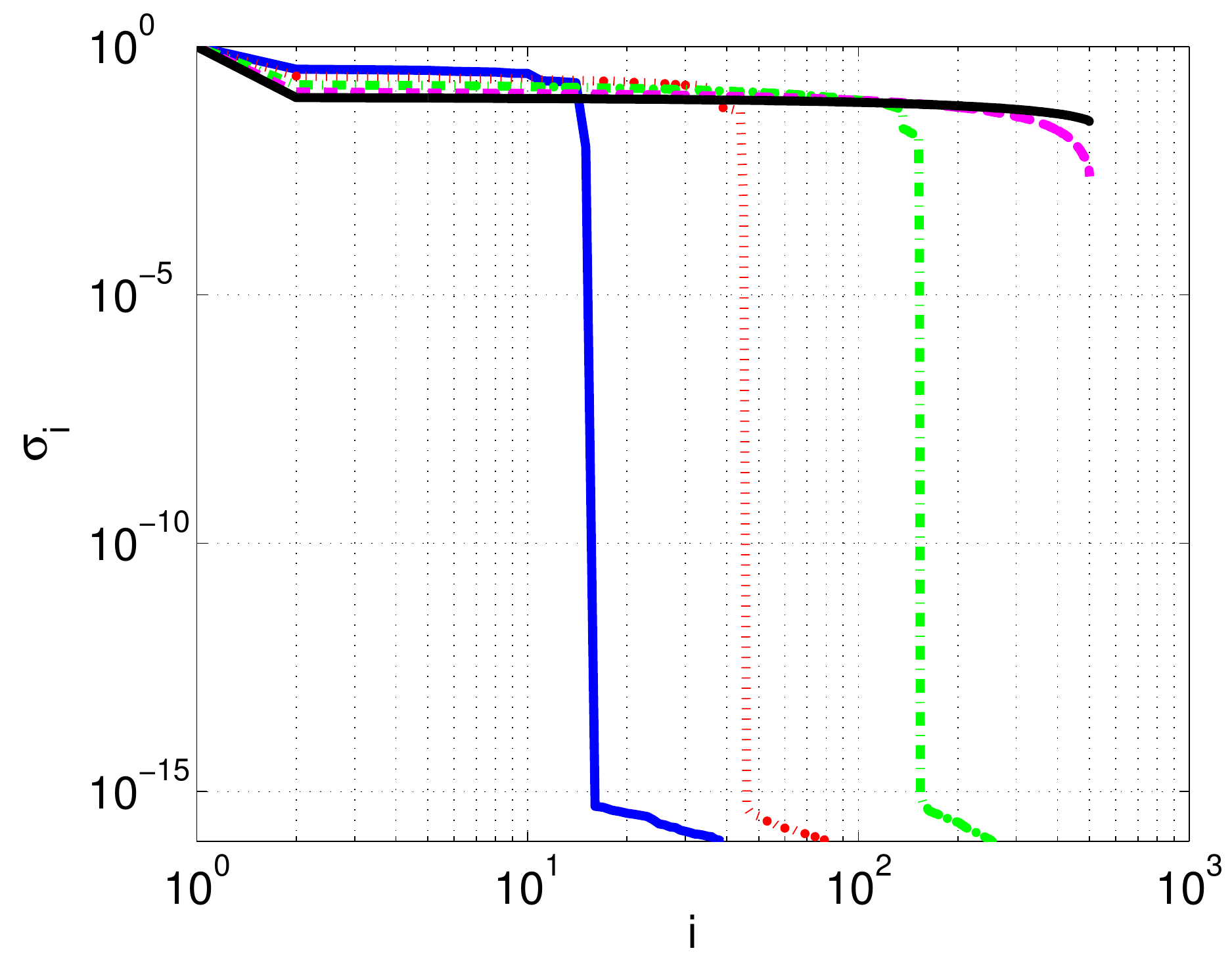}}
        \subfigure[$h = 1$.]{\includegraphics[width=0.3\textwidth]{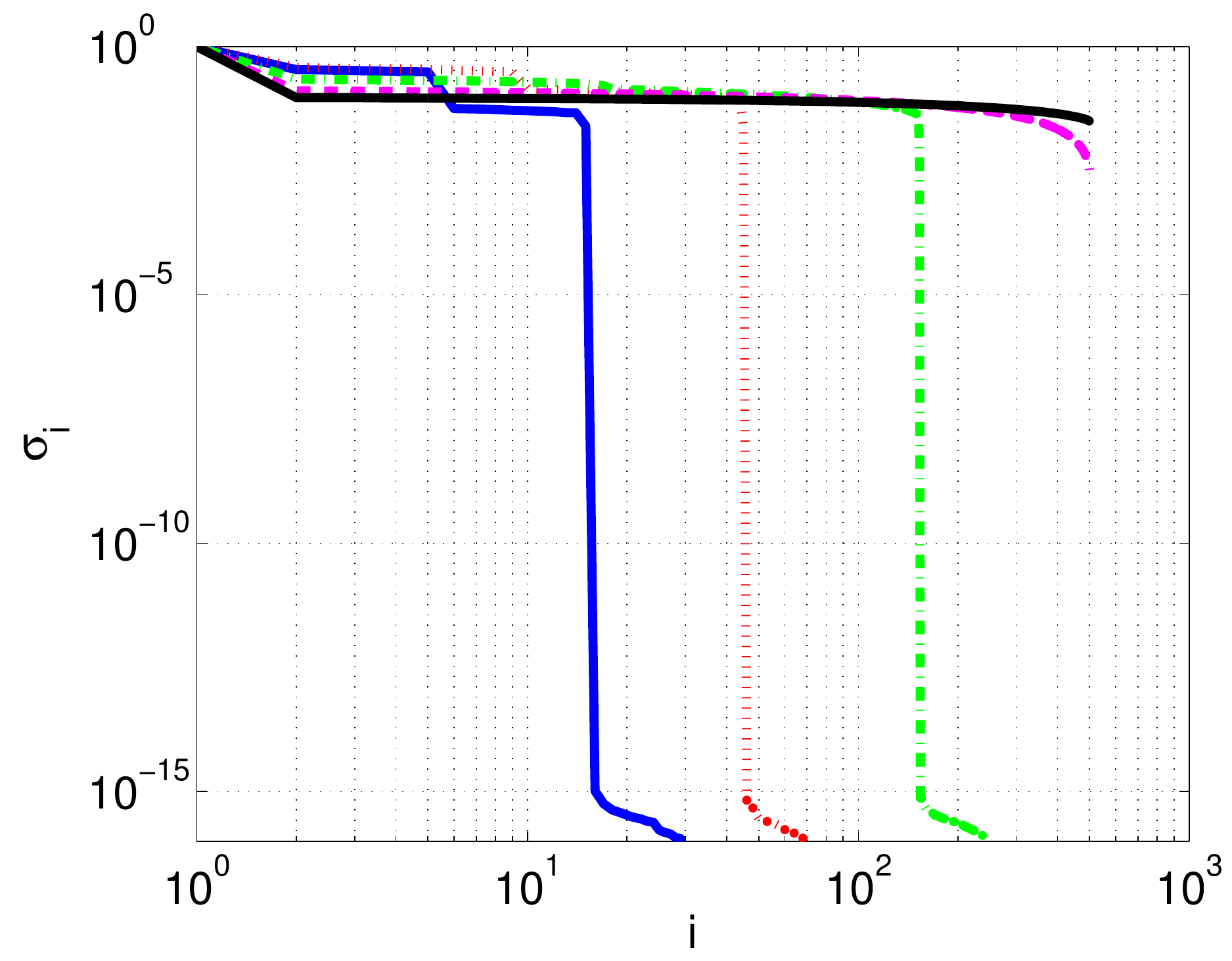}}
        \subfigure[$h = 10$.]{\includegraphics[width=0.3\textwidth]{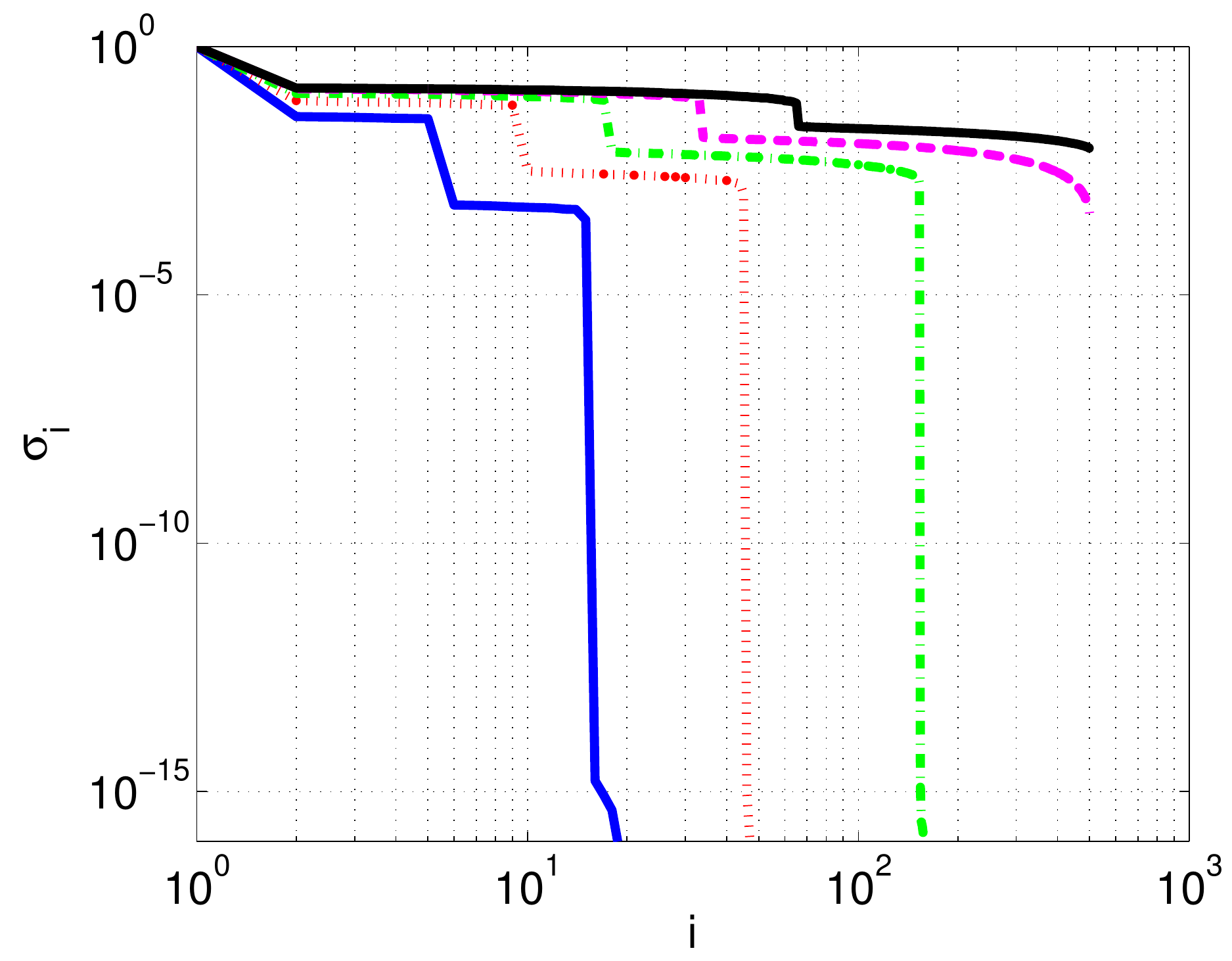}}
        \subfigure[$h = 100$.]{\includegraphics[width=0.3\textwidth]{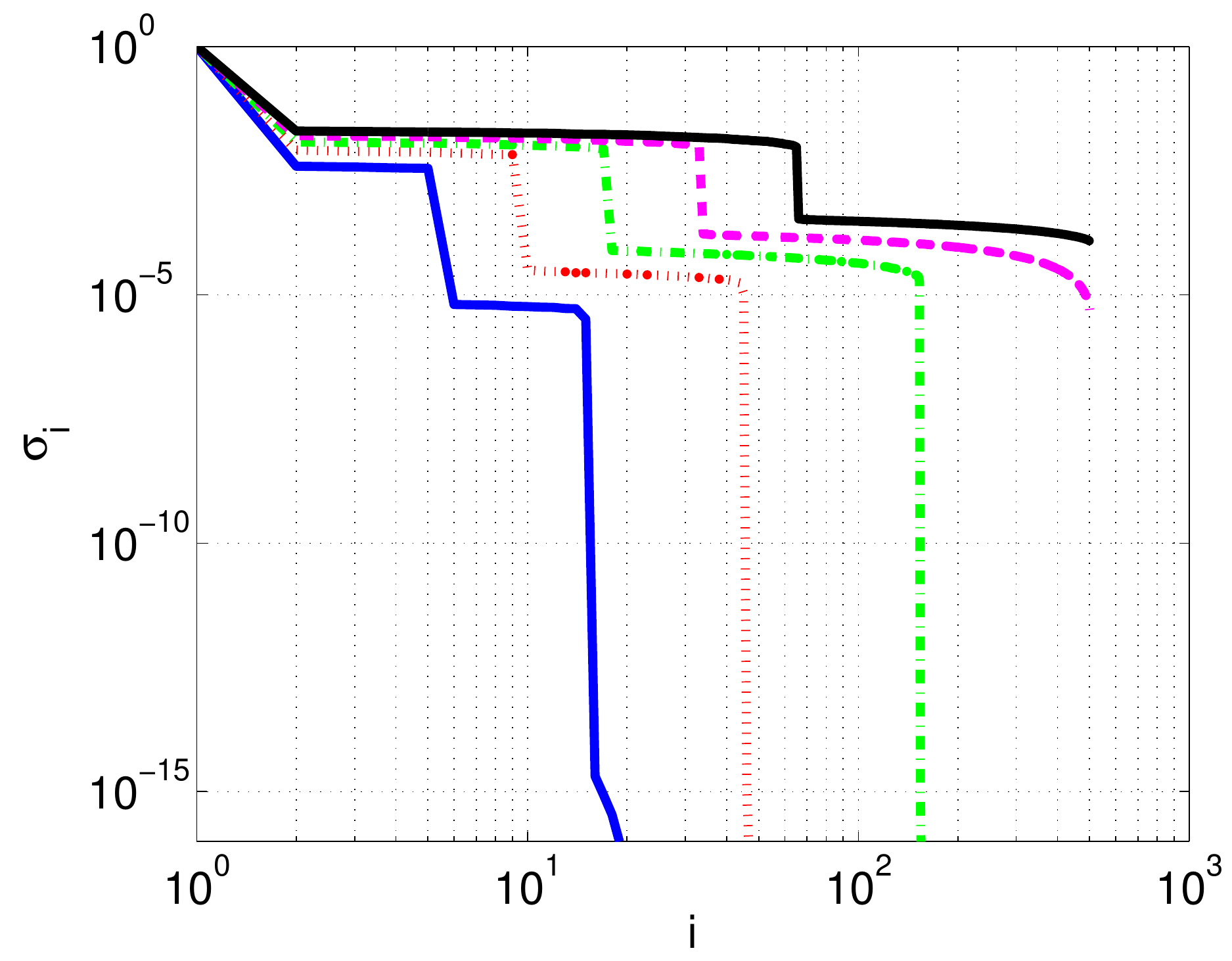}}
        \subfigure[$h = 1000$.]{\includegraphics[width=0.3\textwidth]{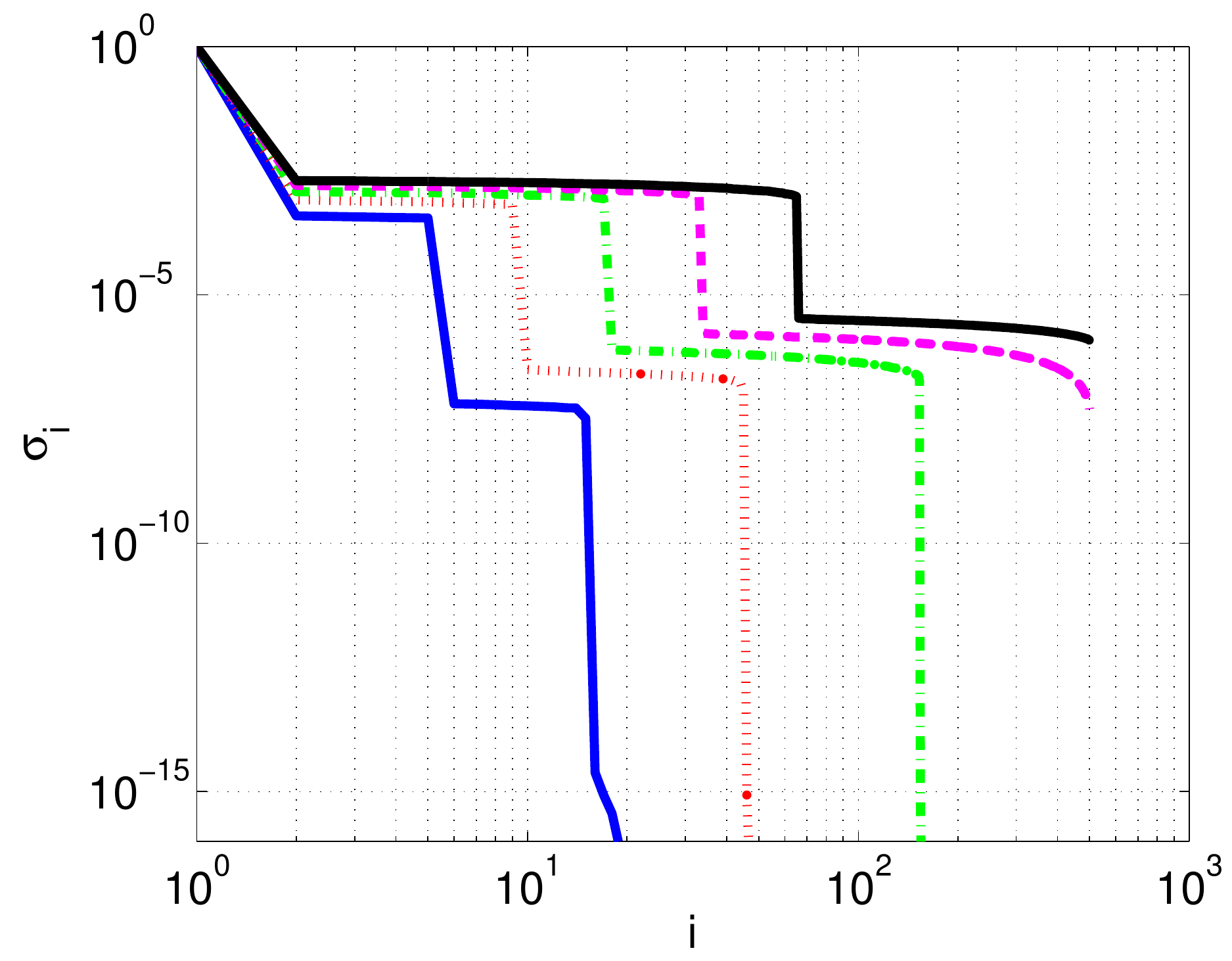}}
\caption{\textbf{Singular values of the polynomial kernel.}
We report the compressibility of the far field by computing the
singular values of $K$ for the polynomial kernel with $p = 2$.
We draw $N = 10^5$ points from
a standard normal distribution. We set $n = 500$
  and $\xi = 1$. The specified values of $\kappa$ correspond to the bandwidths
  given in Table \ref{table_normal_h_values}.
  The trend lines show $d = 4$
  (blue), $d = 8$ (red), $d = 16$ (green), $d = 32$ (magenta), and $d
  = 64$ (black).
\label{fig_polynomial_p_2_spectra_normal_data}}
\end{figure}

\begin{figure}
        \centering
        \subfigure[$h = 0.01$.]{\includegraphics[width=0.3\textwidth]{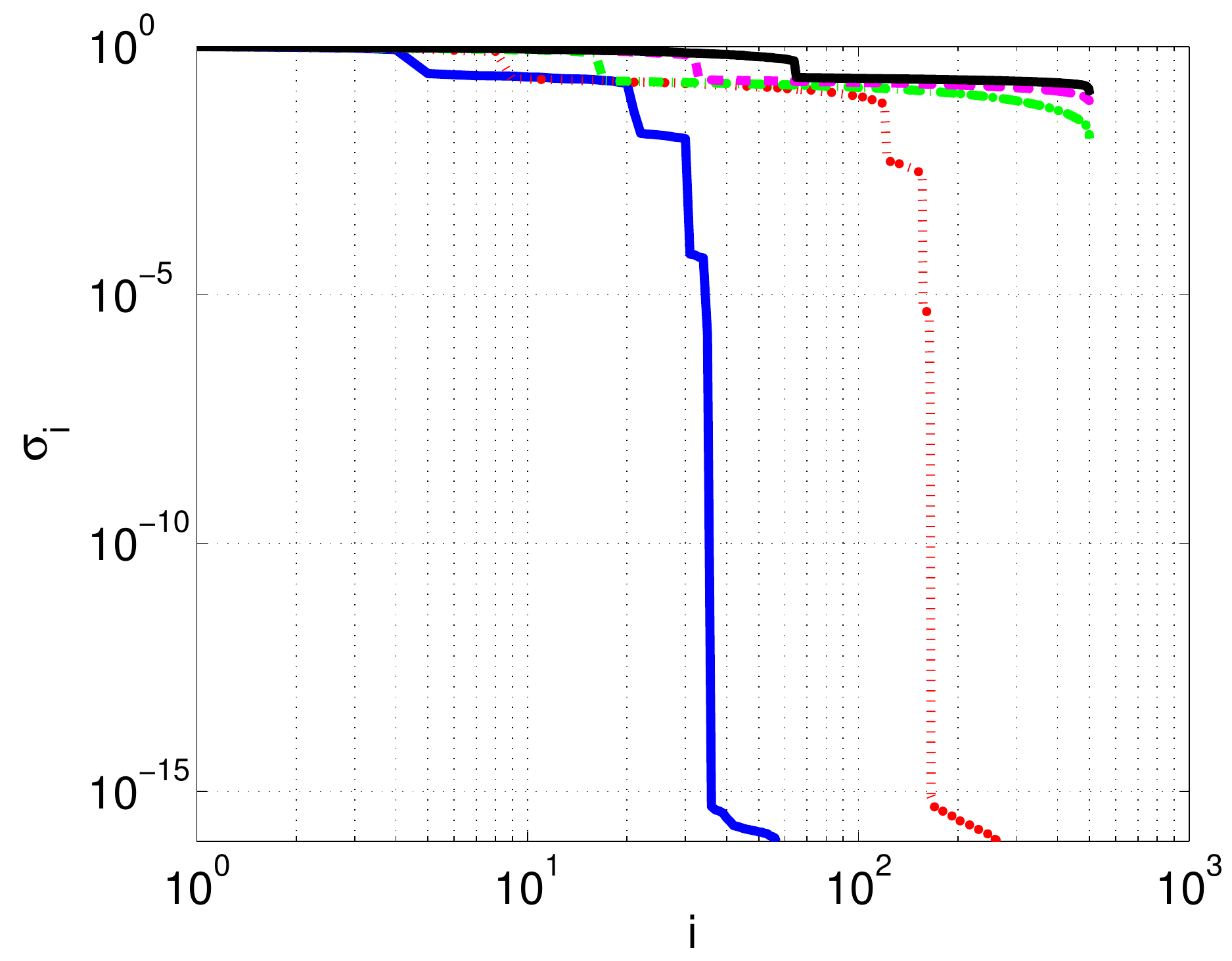}}
        \subfigure[$h = 0.1$.]{\includegraphics[width=0.3\textwidth]{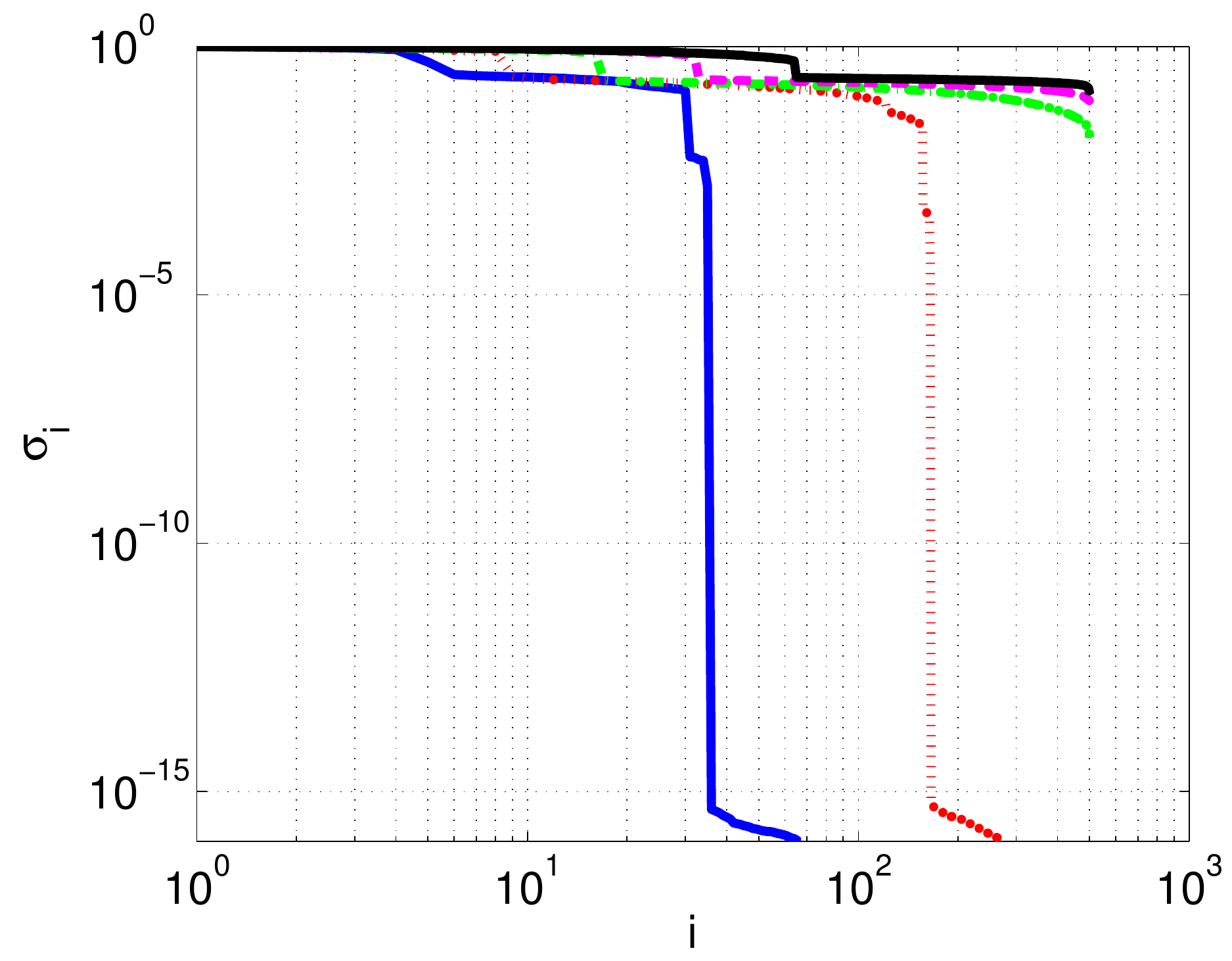}}
        \subfigure[$h = 1$.]{\includegraphics[width=0.3\textwidth]{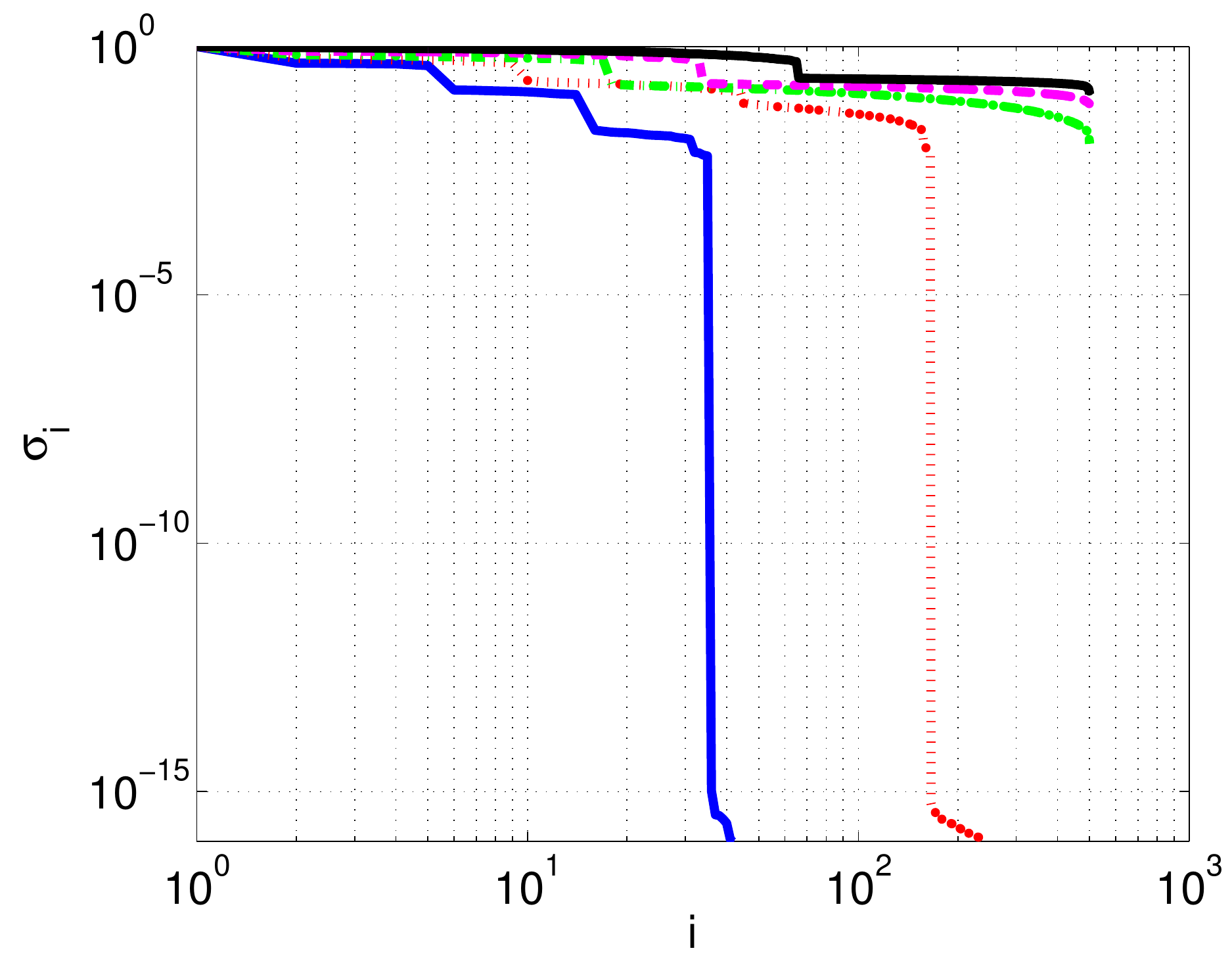}}
        \subfigure[$h = 10$.]{\includegraphics[width=0.3\textwidth]{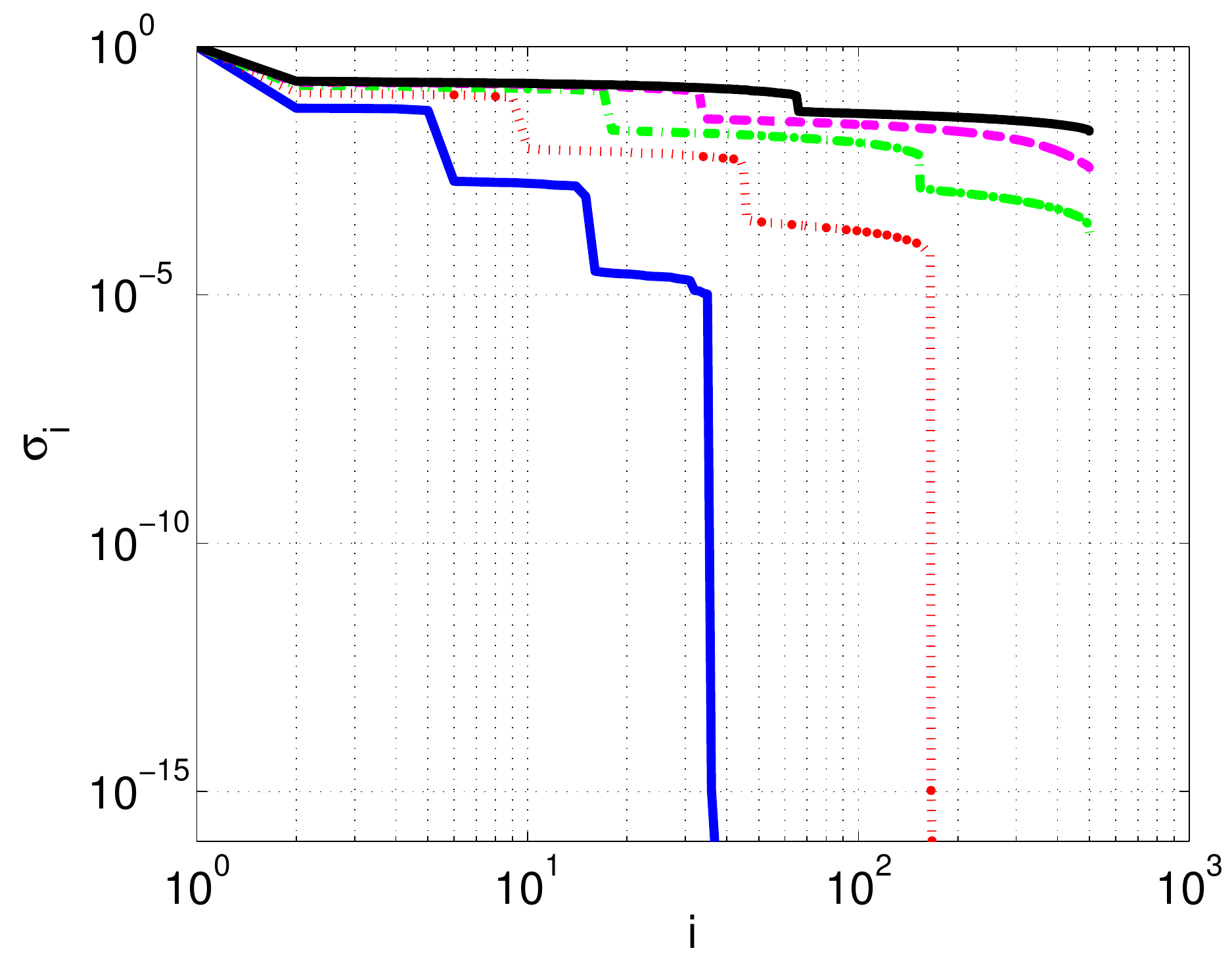}}
        \subfigure[$h = 100$.]{\includegraphics[width=0.3\textwidth]{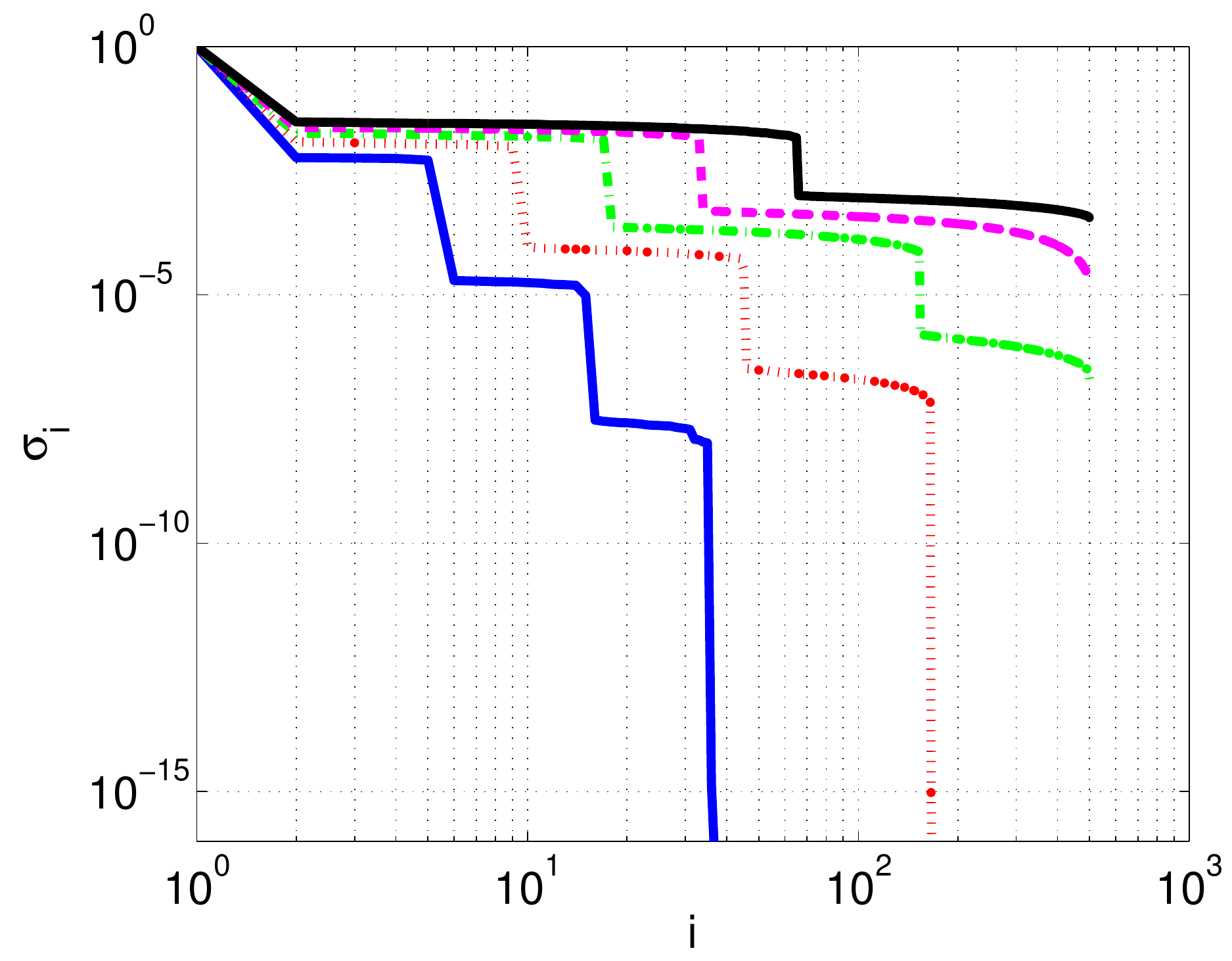}}
        \subfigure[$h = 1000$.]{\includegraphics[width=0.3\textwidth]{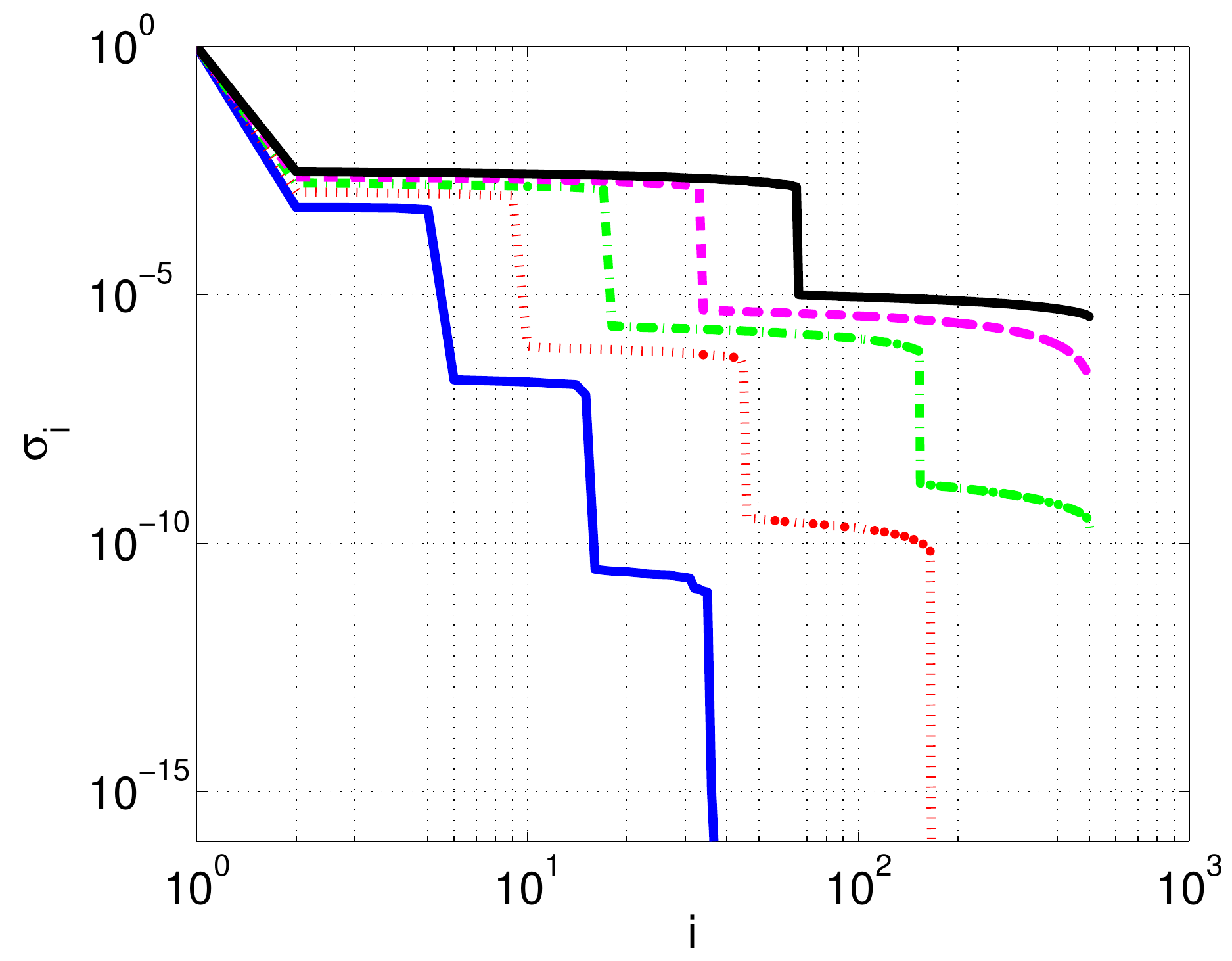}}
\caption{\textbf{Singular values of the polynomial kernel.}
We report the compressibility of the far field by computing the
singular values of $K$ for the polynomial kernel with $p = 3$.
We draw $N = 10^5$ points from
a standard normal distribution. We set $n = 500$
  and $\xi = 1$. The specified values of $\kappa$ correspond to the bandwidths
  given in Table \ref{table_normal_h_values}.
  The trend lines show $d = 4$
  (blue), $d = 8$ (red), $d = 16$ (green), $d = 32$ (magenta), and $d
  = 64$ (black).
\label{fig_polynomial_p_3_spectra_normal_data}}
\end{figure}

\subsubsection{Subsampling}

We show subsampling results for a subset of parameters in 
Figure~\ref{sfig_polynomial_subsampling_mixed}.
Once again, $1\%$ of the rows is sufficient to 
capture the approximation accuracy of the decomposition of the whole matrix in 
this case. However, we note that in some experiments 
(\emph{e.g.}~Figure~\ref{subfig_bad_nn_poly1}), the nearest neighbors 
subsampling method performs 
substantially worse than the other methods. Note that the polynomial kernel is 
not a function of the distance between the points. Therefore, in this case, we 
do not expect the nearest neighbors to necessarily be a good approximation of 
either the leverage scores or the Euclidean norms of the rows.

\begin{figure}
        \centering
        \subfigure[$d = 4, p = 2, h = 0.01.$]{\includegraphics[width=0.3\textwidth]{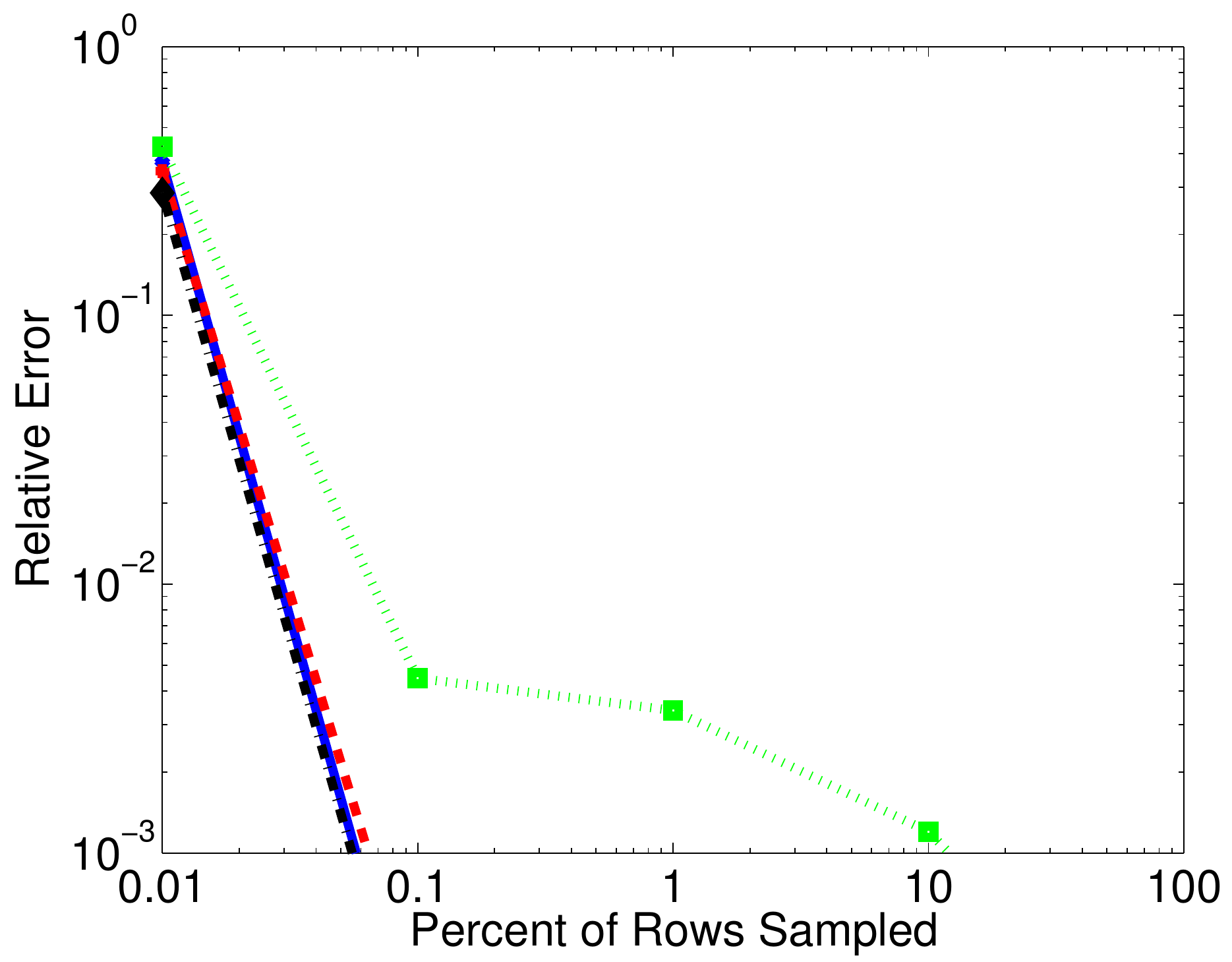}}
        \subfigure[$d = 8, p = 2, h = 0.01$.\label{subfig_bad_nn_poly1}]{\includegraphics[width=0.3\textwidth]{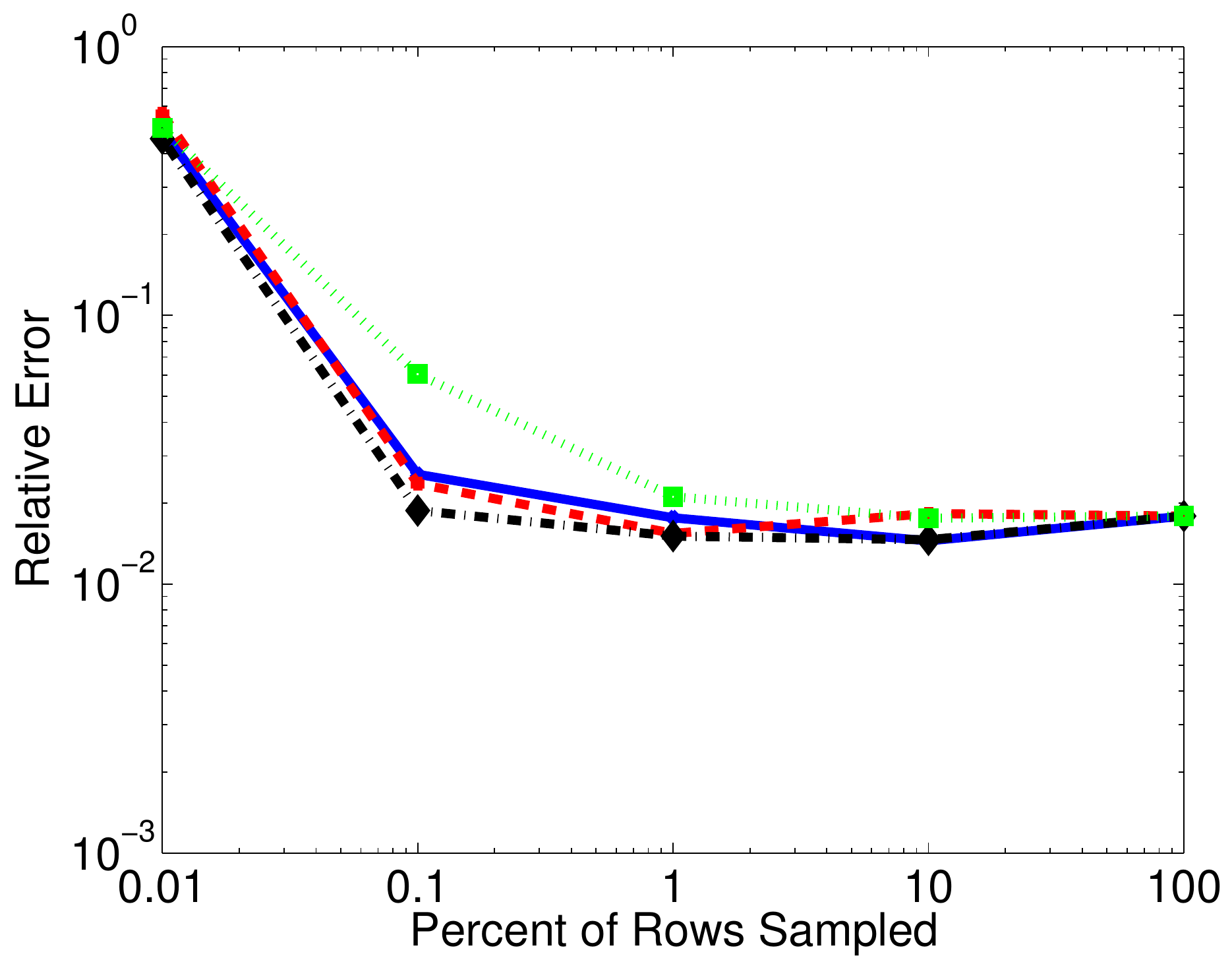}}
        \subfigure[$d = 32, p = 2, h = 0.01$.]{\includegraphics[width=0.3\textwidth]{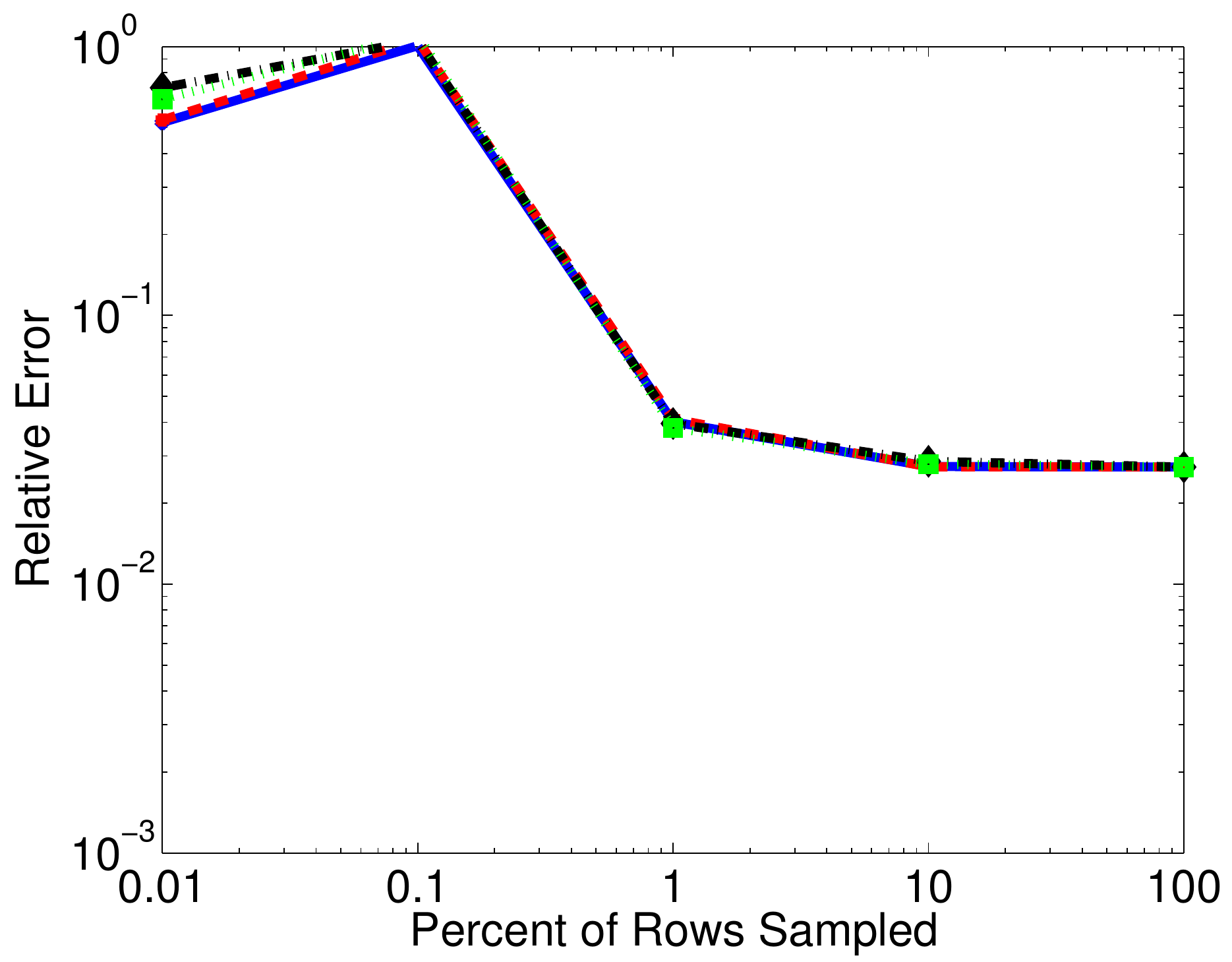}}
        \subfigure[$d = 4, p = 3, h = 0.01.$]{\includegraphics[width=0.3\textwidth]{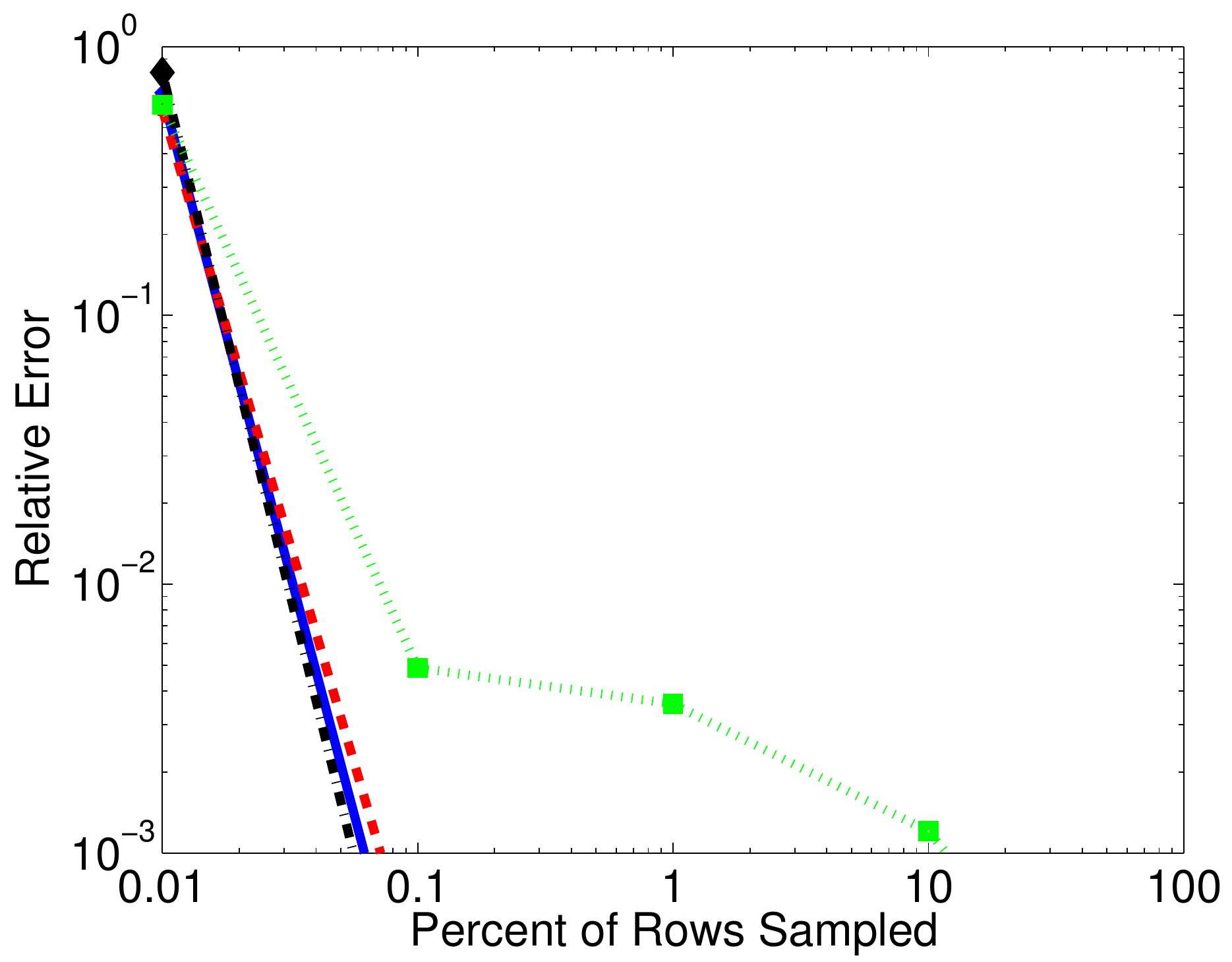}}
        \subfigure[$d = 8, p = 3, h = 0.01.$]{\includegraphics[width=0.3\textwidth]{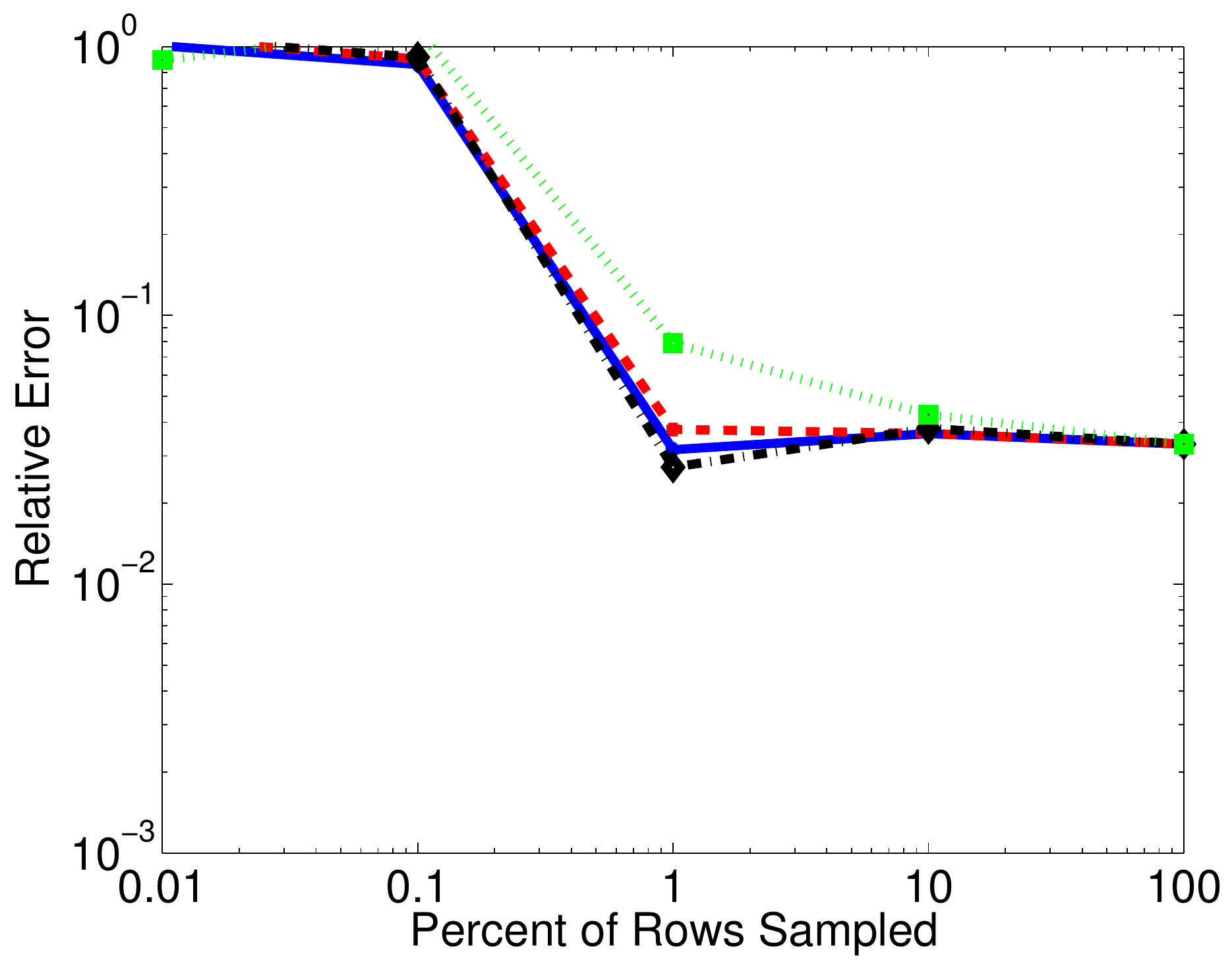}}
        \subfigure[$d = 4, p = 2, h=100.$]{\includegraphics[width=0.3\textwidth]{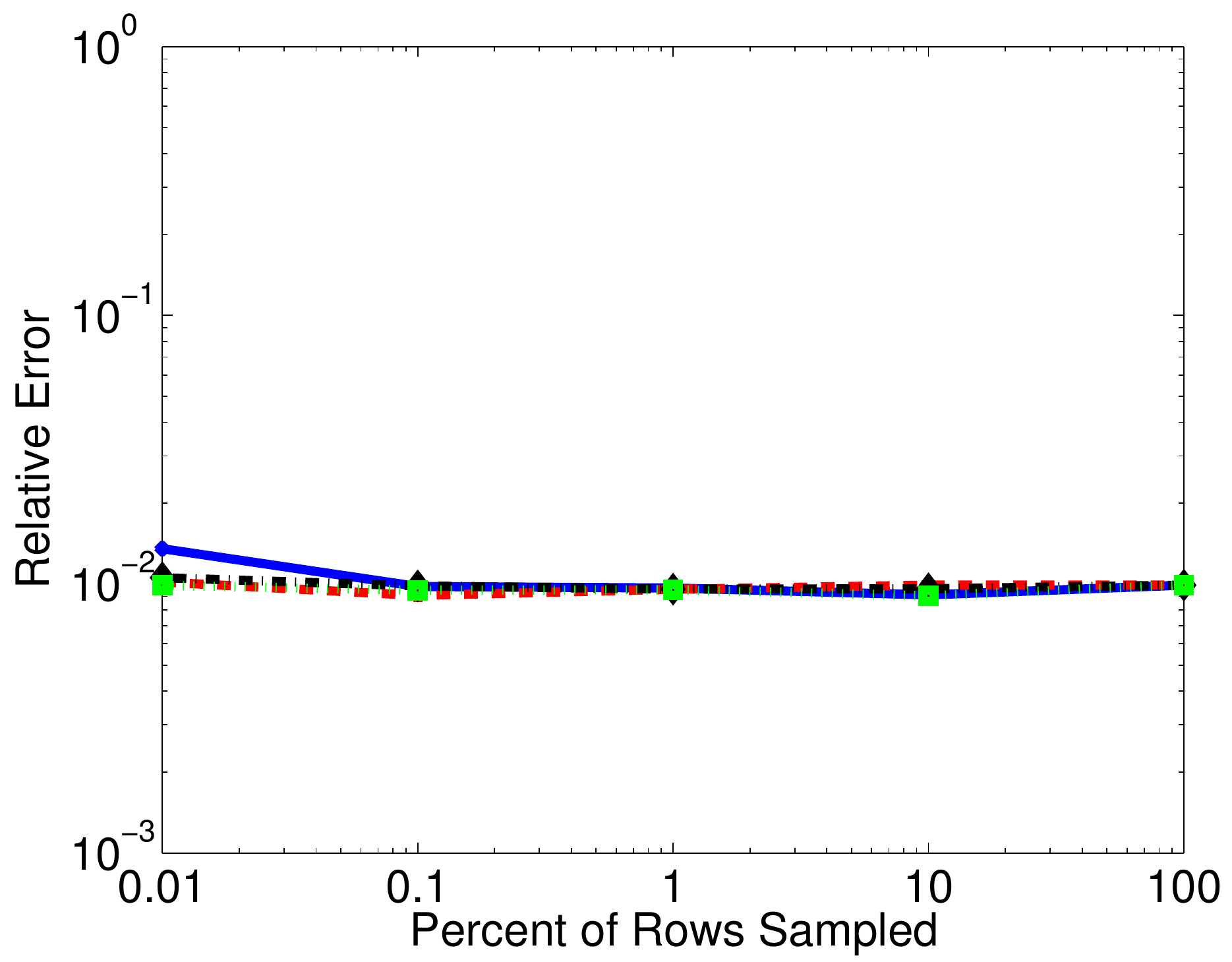}}
        \subfigure[$d = 32, p = 2, h = 100.$]{\includegraphics[width=0.3\textwidth]{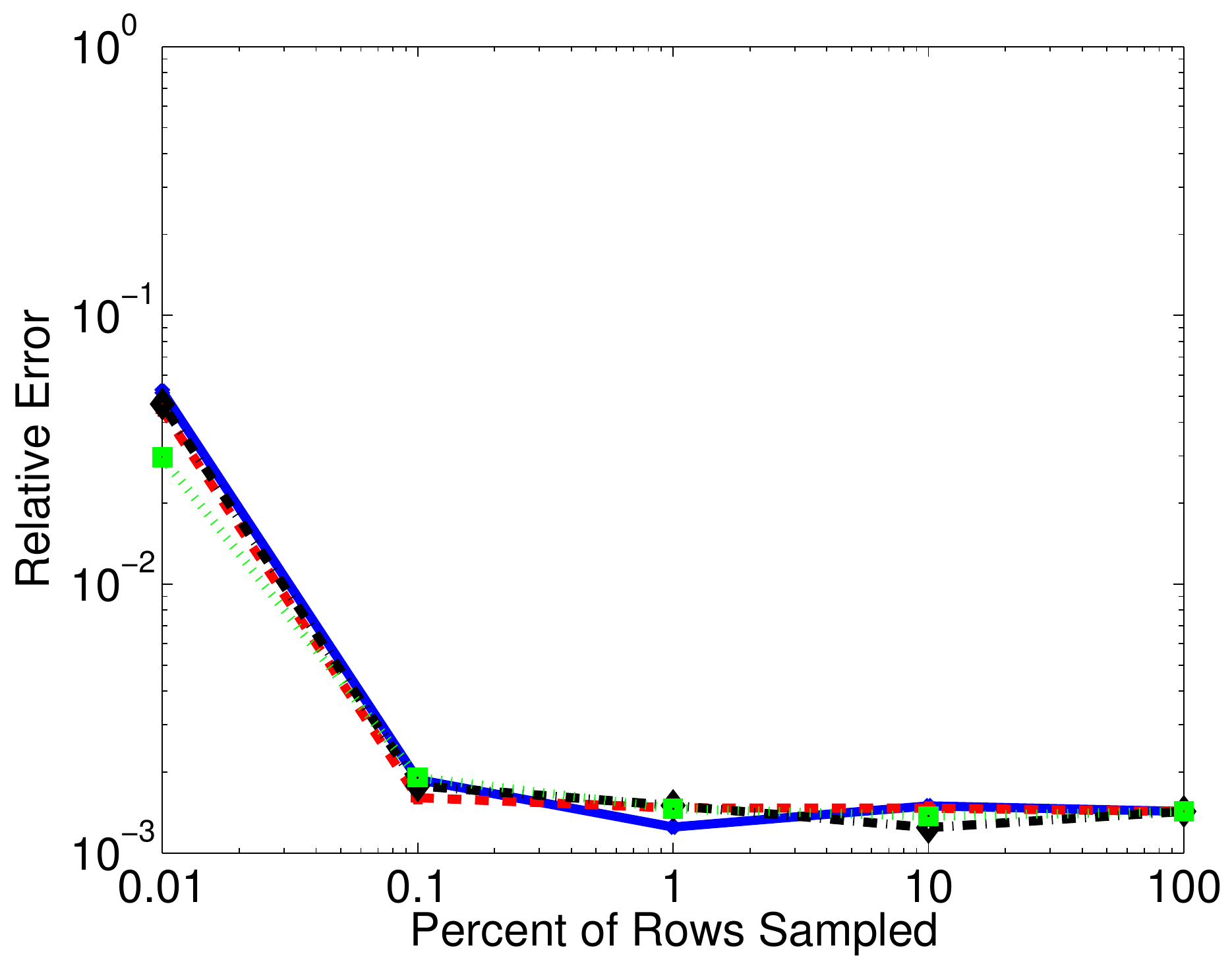}}
        \subfigure[$d = 4, p = 3, h = 100.$]{\includegraphics[width=0.3\textwidth]{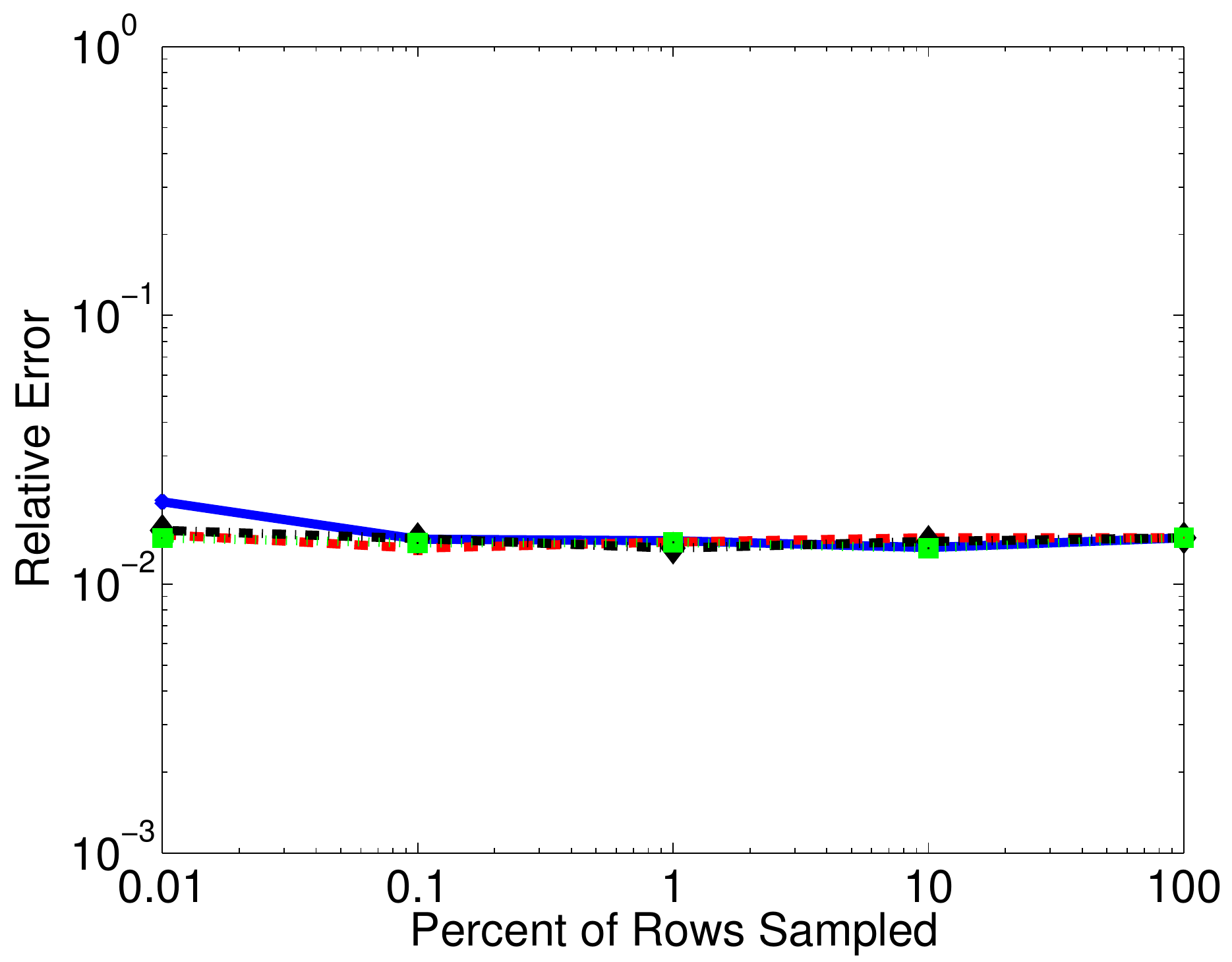}}
        \subfigure[$d = 64, p = 3, h = 100.$]{\includegraphics[width=0.3\textwidth]{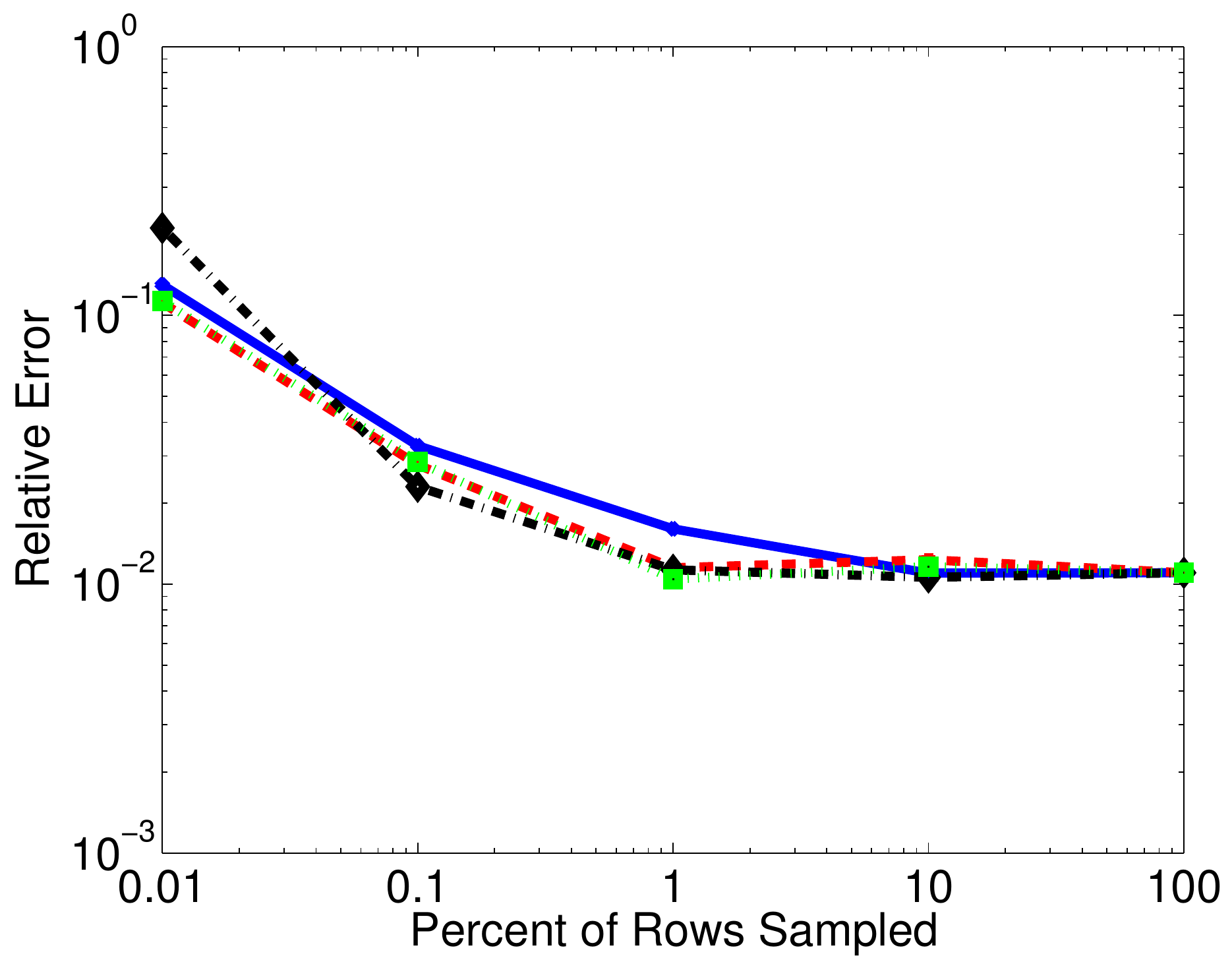}}
        \caption{\textbf{ID compression; polynomial kernel; normal data.}
We show the approximation error of the ID obtained from a subsampled matrix
$\subK$.
  We draw $N = 10^5$ points from the $d$-dimensional standard normal
  distribution and set $n = 500$.
  We set $\epsilon = 10^{-2}$ and choose
 the rank $r$ so that it is the smallest $r$ such that
 $\sigma_{r+1}(K)/\sigma_1(K) < \epsilon$.
  Each trend line represents a
  different subsampling method, with \underline{blue for the uniform
  distribution}, \underline{red for distances}, \underline{black for leverage}, and \underline{green for
  the deterministic selection of nearest neighbors}.
\label{sfig_polynomial_subsampling_mixed}}
\end{figure}

\subsubsection{Low intrinsic dimension}

We show results for our low intrinsic dimension synthetic data in
Figure~\ref{fig_polynomial_high_d_p_2}.  Once again, our sampling methods are
effective in this case, despite the high ambient dimension.  However, we see an
interesting trend for $h = 0.01$.  The deterministic selection of rows based on
the nearest target points shows significantly larger error than any of the
other methods, including the random selection of rows with probabilities based
on distances. Note that this kernel does not decrease with increasing distance
between its arguments.  Therefore, it is not surprising that nearest neighbors
do not necessarily capture the most important target points.

\begin{figure}[tbph]
        \centering
        \subfigure[$h = 0.01.$\label{subfig_bad_nn_poly2}]{\includegraphics[width=0.45\textwidth]{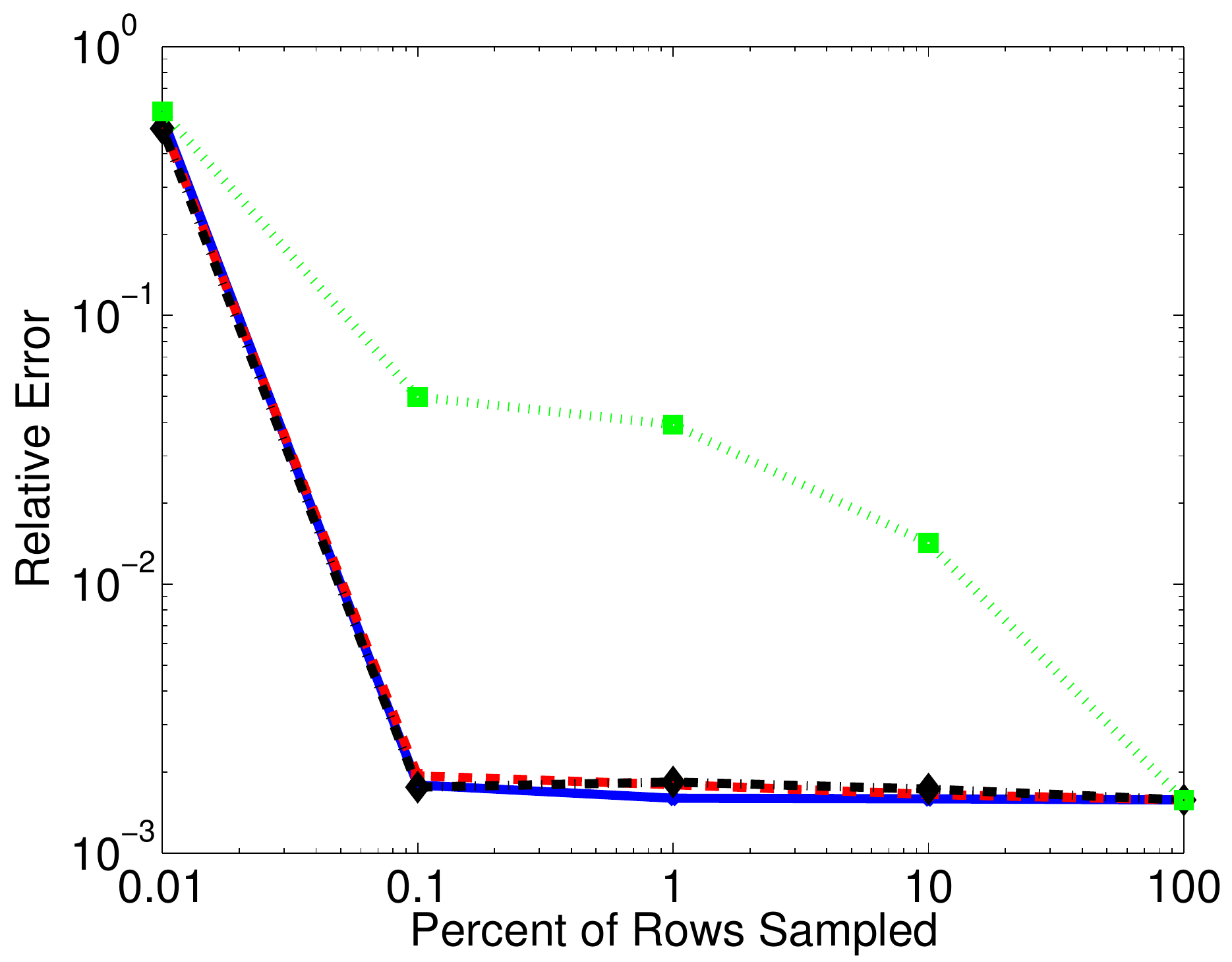}}
        \subfigure[$h = 100.$]{\includegraphics[width=0.45\textwidth]{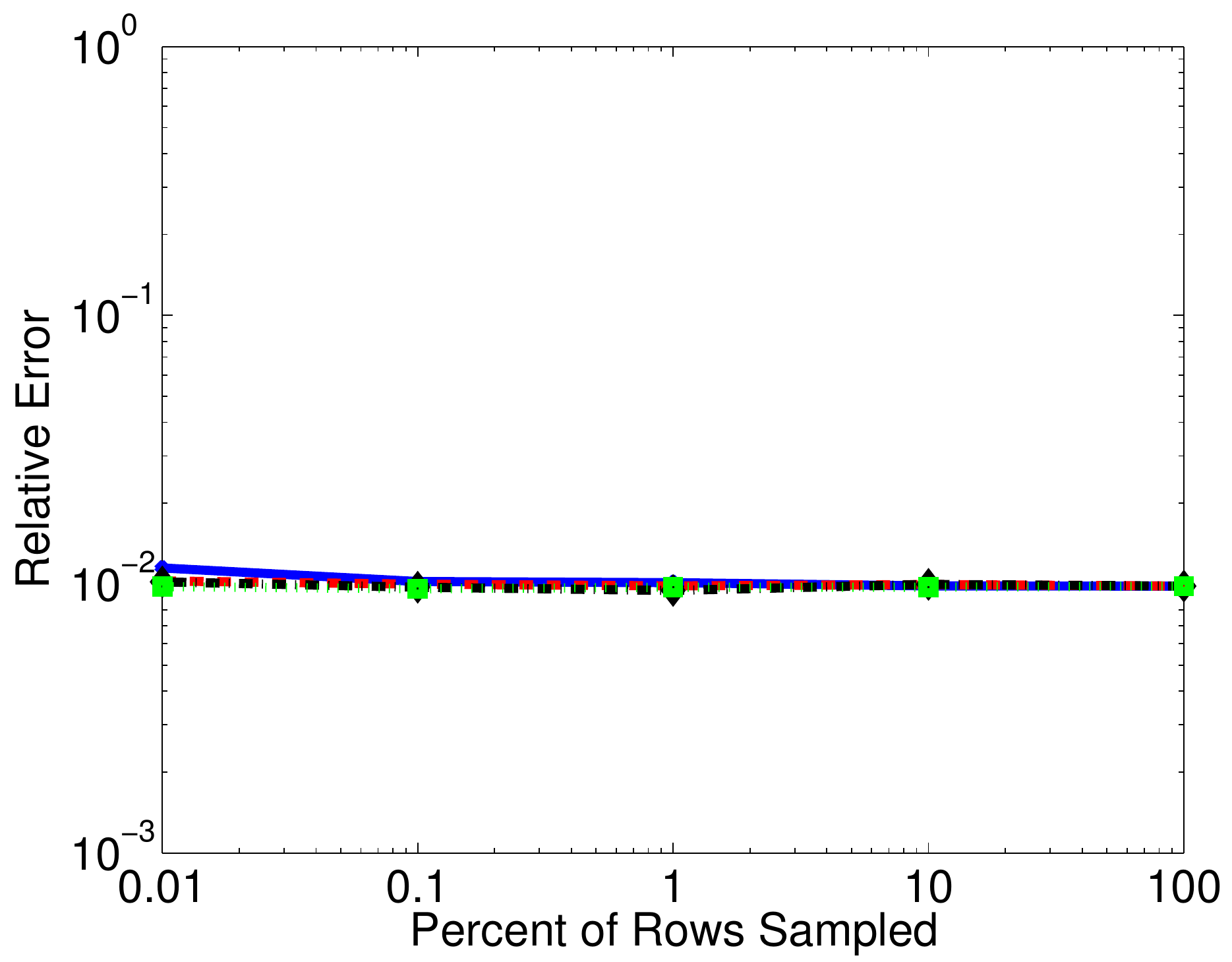}}
\caption{\textbf{ID compression; polynomial kernel; low-dimensional data,
\boldmath$p = 2$\unboldmath.}
We show the approximation error of the ID obtained from a subsampled matrix
$\subK$.
  We draw $N = 10^5$ points from our low-dimensional data distribution with
  intrinsic dimension 4 and ambient dimension 1,000. We set $n = 500$.
  We use the bandwidth given in the subfigure captions.
  We set $\epsilon = 10^{-2}$ and choose
 the rank $r$ so that it is the smallest $r$ such that
 $\sigma_{r+1}(K)/\sigma_1(K) < \epsilon$.
  Each trend line represents a
  different subsampling method, with \underline{blue for the uniform
  distribution}, \underline{red for distances}, \underline{black for leverage}, and \underline{green for
  the deterministic selection of nearest neighbors}.
  \label{fig_polynomial_high_d_p_2}}
\end{figure}


\subsubsection{Real data sets}

We show results on two of our real data sets in Figures
\ref{fig_polynomial_color_hist} and \ref{fig_polynomial_color_moments}.  Once
again, we see that subsampling rows is extremely effective for both these data
sets.

\begin{figure}
        \centering
        \subfigure[$h = 0.01, p = 2.$]{\includegraphics[width=0.45\textwidth]{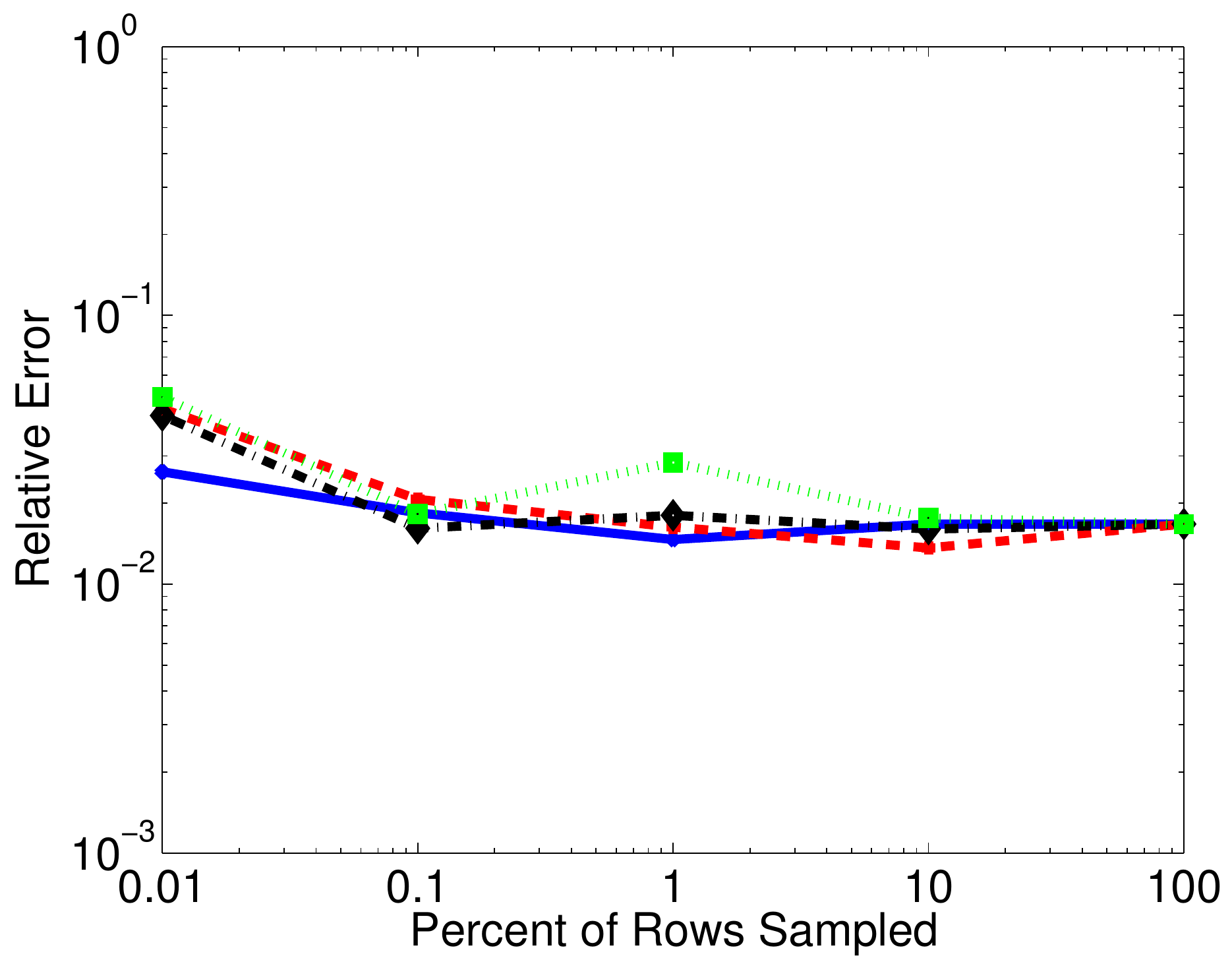}}
        \subfigure[$h = 0.01, p = 3.$]{\includegraphics[width=0.45\textwidth]{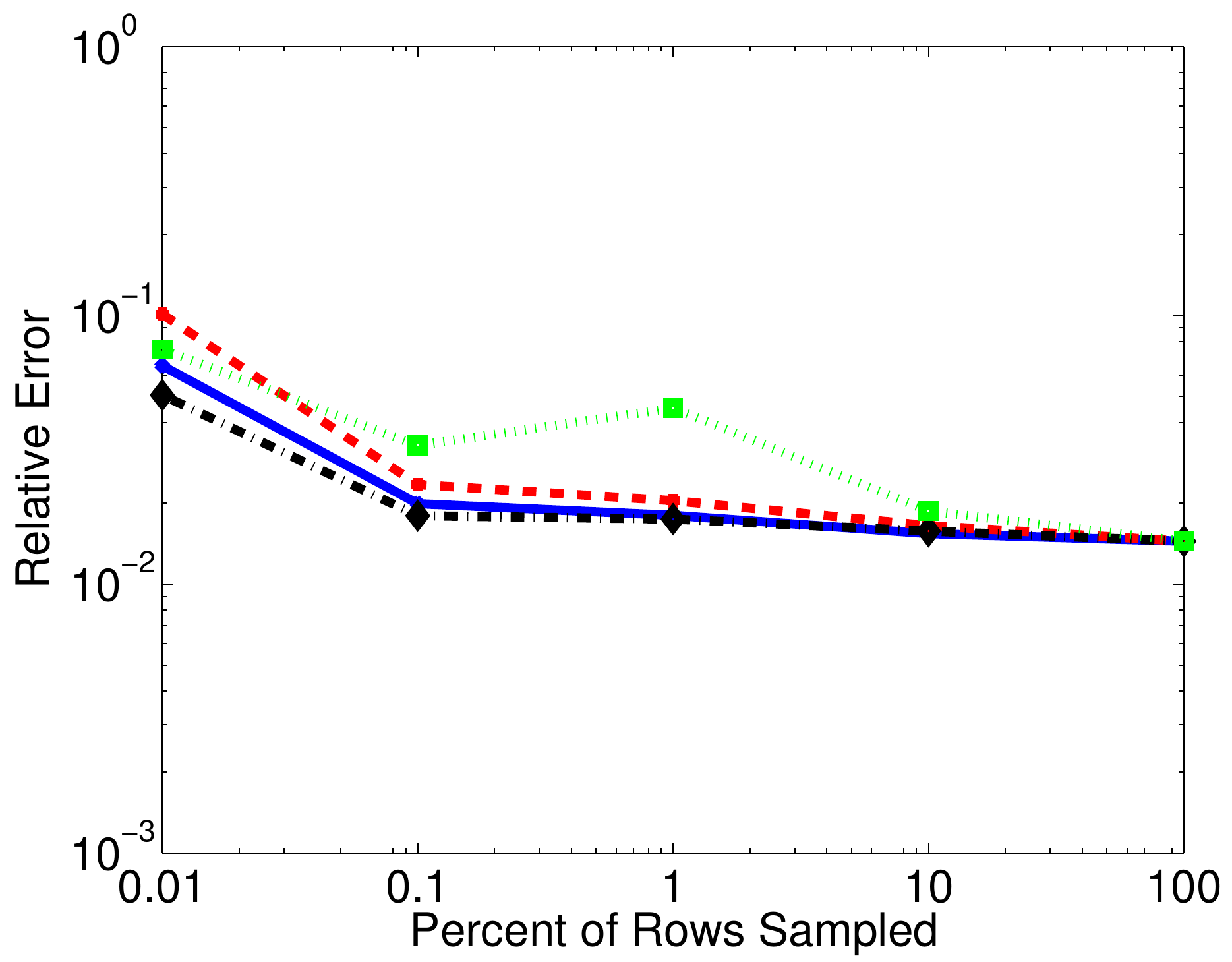}}
\caption{\textbf{ID compression; polynomial kernel; Color Histogram data.}
We show the approximation error of the ID obtained from a subsampled matrix
$\subK$. We use the Color Histogram data set, with $N = 68,040$ and $d = 32$.
We choose $x_c$ uniformly at random and set $n = 500$.
  We use the bandwidth given in the subfigure captions.
  We set $\epsilon = 10^{-2}$ and choose
 the rank $r$ so that it is the smallest $r$ such that
 $\sigma_{r+1}(K)/\sigma_1(K) < \epsilon$.
  Each trend line represents a
  different subsampling method, with \underline{blue for the uniform
  distribution}, \underline{red for distances}, \underline{black for leverage}, and \underline{green for
  the deterministic selection of nearest neighbors}.
\label{fig_polynomial_color_hist}}

\end{figure}

\begin{figure}
        \centering
        \subfigure[$h = 0.01, p = 2.$]{\includegraphics[width=0.45\textwidth]{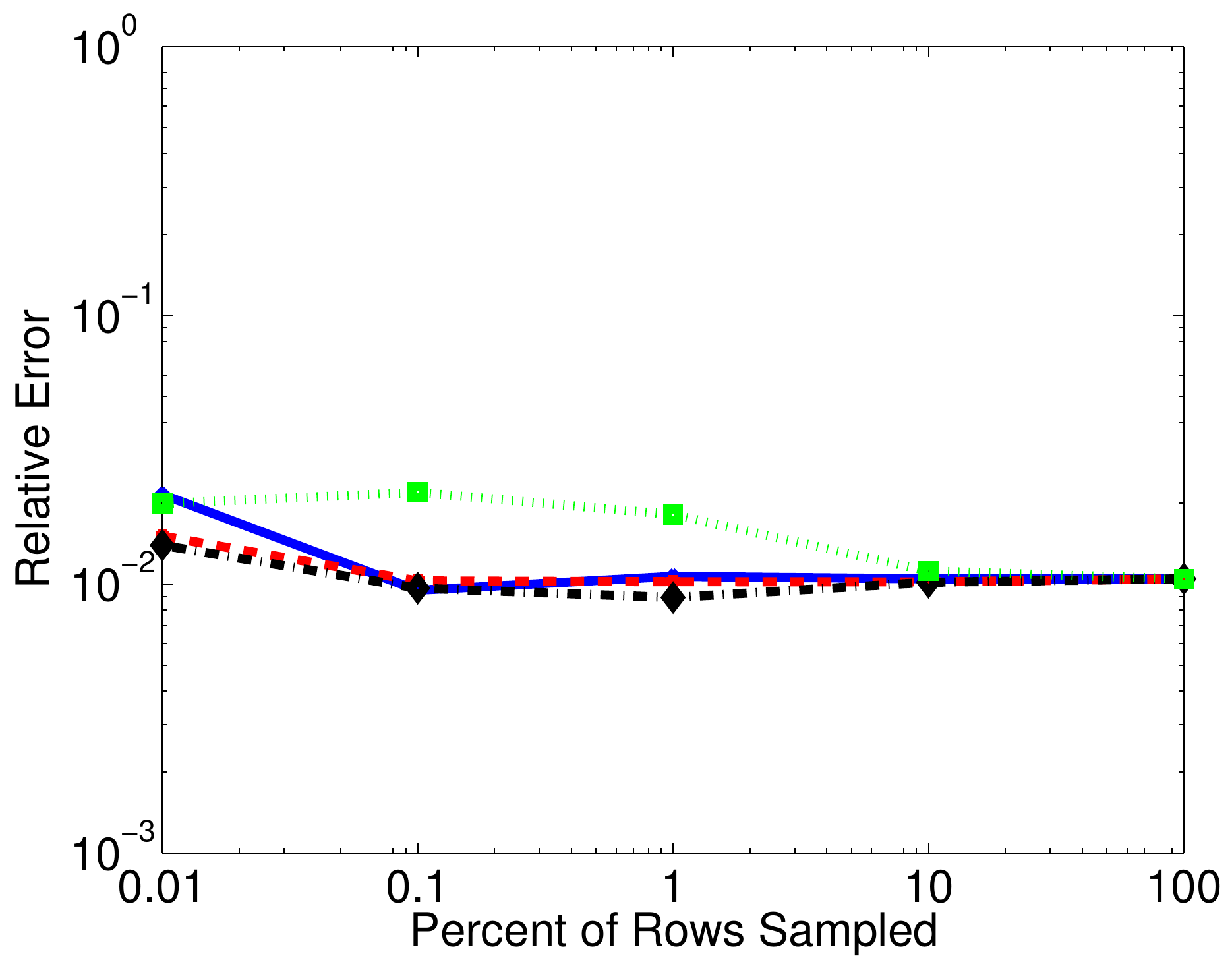}}
        \subfigure[$h = 0.01, p = 3.$]{\includegraphics[width=0.45\textwidth]{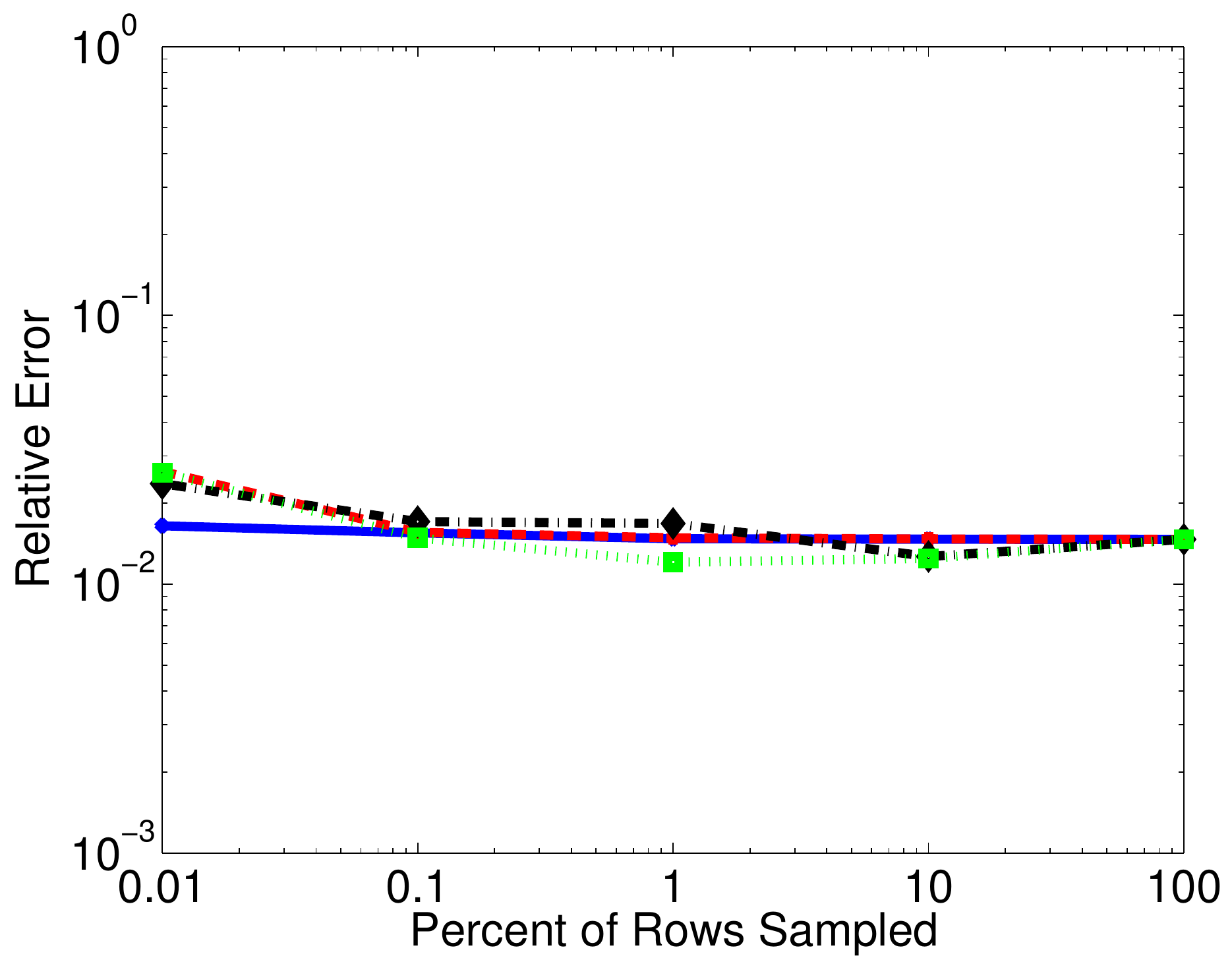}}
\caption{\textbf{ID compression; polynomial kernel; Color Moments data.}
We show the approximation error of the ID obtained from a subsampled matrix
$\subK$. We use the Color Moments data set, with $N = 68,040$ and $d = 32$.
We choose $x_c$ uniformly at random and set $n = 500$.
  We use the bandwidth given in the subfigure captions.
  We set $\epsilon = 10^{-2}$ and choose
 the rank $r$ so that it is the smallest $r$ such that
 $\sigma_{r+1}(K)/\sigma_1(K) < \epsilon$.
  Each trend line represents a
  different subsampling method, with \underline{blue for the uniform
  distribution}, \underline{red for distances}, \underline{black for leverage}, and \underline{green for
  the deterministic selection of nearest neighbors}.
\label{fig_polynomial_color_moments}}
\end{figure}

%

\section{Conclusion\label{sec_conclusion}}

We demonstrate that the method of randomly sampling rows to build an 
approximation of a kernel submatrix may be useful for the efficient 
construction of outgoing representations in fast kernel summation algorithms.  
This approach has two major advantages: first, it requires no prior knowledge 
about the kernel function, and second, it can work even for extremely high 
dimensional data. We show empirically that this approach is effective for 
several commonly used kernel functions and data sets. We also prove a new 
theorem about random sampling, showing that uniformly chosen rows can provide a 
very accurate approximation in many cases. 
Our next steps will include an exploration of this approach in the context of a 
treecode or FMM. To integrate this with a treecode, one needs to decide on 
the number of points per box and on the number of rows to subsample while 
maximizing accuracy and minimizing cost.

%
%
%

\clearpage

\section{Appendix\label{sec_appendix}}
We include full proofs to the theorems in the main text.

\subsection{Computation of the interpolative decomposition}

We compute the interpolative decomposition of a matrix
$K \in \reals^{m \times n}$ as follows:
\begin{enumerate}
\item Fix an approximation rank $r$.
\item Compute a rank-revealing QR factorization
\cite{gu1996efficient} $K = Q R \Pi^T$, where $Q \in 
\reals^{m \times n}$ has orthonormal columns, $\Pi$ is a column permutation of 
$K$, and $R \in \reals^{n \times n}$ is an upper triangular matrix which we partition as
\begin{equation}
R = \left[
\begin{array}{cc}
R_{11} & R_{12} \\
0 & R_{22}
\end{array}
 \right]
\end{equation}
where $R_{11} \in \reals^{r \times r}$.
\item The projection matrix $P$ is the minimum-norm solution to the 
under-determined system
\begin{equation}
R_{11} P = R_{12}
\end{equation}
\item The skeleton is the first $r$ columns selected by $\Pi$.
\end{enumerate}

%


\subsection{Proof of Theorem~\ref{thm_column_sampling_approx2}}

We first state the two main results we use.  We then fill in the remaining 
details.
First, the deterministic bound 
on matrix projections~\cite{boutsidis2009improved, halko2011finding}.
\begin{theorem}
Let $A$ be an $n \times m$ matrix with singular value decomposition $U  
\Sigma  V^T$ and let $r \geq 0$. Let $\Omega$ be any matrix in $\mathbb{R}^{m 
\times s}$ for $s \geq r$. 
Partition $\Sigma$ into $\Sigma_1$ and $\Sigma_2$, such that $\Sigma_1$ 
contains the first $r$ singular values, and $\Sigma_2$ the rest. Let $\Omega_1 
= V_1^T \Omega$ and $\Omega_2 = V_2^T \Omega$, where $V_1$ is the first $r$ 
columns of $V$, and $V_2$ is the rest. 
Then, if $\Omega_1$ has full row rank, 
\begin{equation}
\| (I - \Pi) A \|^2 \leq \| \Sigma_2\|^2 + \| \Sigma_2 \Omega_2 
\Omega_1^\dagger\|^2,
\label{eqn_hmt_matrix_bounds}
\end{equation}
where $\Pi$ is an orthogonal projection onto the range of $A \Omega$, and 
$\Omega_1^\dagger$ is the Moore-Penrose pseudoinverse of $\Omega_1$. 
\label{thm_deterministic_subspace_approx_HMT}
\end{theorem}

We also make use of the matrix Chernoff inequality \cite{ahlswede2002strong, tropp2015introduction, tropp2012user}.
We state only the tail inequalities which we use for our proof. 
\begin{theorem}
  \label{thm_matrix_chernoff}
  Let $\{X_k\}$ be a random sequence of independent, Hermitian matrices of 
  dimension $r$ and let $Y = \sum X_k$. Let 
  \begin{equation}
    0 \leq \lmin(X_k) \quad \textrm{ \emph{and} } \quad \lmax(X_k) \leq L \quad \textrm{ for all } k
  \end{equation}
  
  Define the minimum and maximum eigenvalues of the expected value of $Y$:
  \begin{equation}
    \mmin = \lmin(\expect Y) \quad \textrm{ \emph{and} } \quad \mmax = \lmax(\expect Y)
  \end{equation}
  
  Then, we have that for any $\epsilon \in [0,1)$
  \begin{equation}
    \prob\left\{ \lmin(Y) \leq (1 - \epsilon) \mmin \right\} 
    \leq r \left[ \frac{e^{-\epsilon}}{(1-\epsilon)^{1-\epsilon}} \right]^{\mmin / L}
    \label{e:chernoff_lower}
  \end{equation}
  and, for any $\epsilon \geq 0$,
  \begin{equation}
    \label{e:chernoff_upper}
    \prob\left\{ \lmax(Y) \geq (1 + \epsilon) \mmax \right\} 
    \leq r \left[ \frac{e^{\epsilon}}{(1+\epsilon)^{1+\epsilon}} \right]^{\mmax / L}.
  \end{equation}
\end{theorem}

We now prove Theorem~\ref{thm_column_sampling_approx2}. We choose a sampling
matrix $\Omega$ corresponding to a random subset of columns (chosen with 
replacement) and 
construct the matrices $\Omega_1$ and $\Omega_2$ from the statement of Theorem 
\ref{thm_deterministic_subspace_approx_HMT}. We then apply the Chernoff bound
(Thm.~\ref{thm_matrix_chernoff}) to bound the singular values of the 
matrices in \eqref{eqn_hmt_matrix_bounds}. We then extend the proof to the 
case of sampling without replacement. 

\begin{proof}
\emph{(of Theorem \ref{thm_column_sampling_approx2}).} Let $U \Sigma V^T$ be 
the SVD of $A$, and let $V_1$ ($\Sigma_1$) be the first $r$ right singular 
vectors (values) and $V_2$ ($\Sigma_2$) be the rest. 

Given $\epsilon$ and $\delta$,
sample $\ns$ integers from $\{1, \ldots, m\}$ uniformly with replacement, where 
$\ns$ satisfies \eqref{eqn_sample_complexity}.
Let $\subA$ be the matrix whose columns are the corresponding $\ns$ columns of 
$A$, scaled by $\sqrt{\frac{m}{\ns}}$. 
Let $O \in \mathbb{R}^{m \times \ns}$ 
be the matrix whose columns consist of the 
standard basis vectors in $\reals^m$ corresponding to the sampled columns. 
Let $\Omega = \sqrt{\frac{m}{\ns}} O$.  Note that $\subA = A \Omega$.
Define the $r \times \ns$ matrix $\Omega_1 = V_1^T \Omega$ and the $(m - r) 
\times \ns$ matrix $\Omega_2 = V_2^T \Omega$. 

Assume for now that $\Omega_1$ has full row rank (we prove below that this 
occurs with high probability).
Using
this, we apply 
Thm.~\ref{thm_deterministic_subspace_approx_HMT} to obtain
\begin{equation}
\|(I - \Pi) A \|^2 \leq \|\Sigma_2\|^2 + \|\Sigma_2 \Omega_2 \Omega_1^\dagger \|^2
\label{line_bound_part_1}
\end{equation}
where $\Pi$ is an orthogonal projection onto the span of the columns of $\subA$. 
Note that $\| \Sigma_2 \|^2 = \sigma_{r+1}^2$, since $\Sigma_2$ is the matrix 
$\diag(\sigma_{r+1}, \ldots, \sigma_n)$.

Furthermore, since the rows of $\Omega_1$ are linearly independent (by 
assumption), we have that 
\begin{equation}
\Omega_1^\dagger  =  \Omega_1^T \left( \Omega_1 \Omega_1^T \right)^{-1}
\label{eqn_proof_matrix_product}
\end{equation}

We now bound the quantity on the right hand side of Equation 
\ref{line_bound_part_1} as:
\begin{eqnarray}
  \left\|\Sigma_2 \Omega_2 \Omega_1^\dagger \right\|^2 
	& = & \left\|\Sigma_2 \Omega_2 \Omega_1^T 
		\left( \Omega_1 \Omega_1^T \right)^{-1} \right\|^2 \\
	& \leq & \left\|\Sigma_2 \right\|^2 \left\| \Omega_2 \Omega_1^T \right\|^2 
		\left\|\left( \Omega_1 \Omega_1^T \right)^{-1} \right\|^2 
\label{eqn_three_terms_proof}
\end{eqnarray}
We can complete the proof by bounding each of the three terms in Equation 
\ref{eqn_three_terms_proof} and by showing that $\Omega_1$ has full row rank.
We use the Chernoff bound for both of these tasks. 

\textbf{Applying Chernoff bound.}
We now apply Thm.~\ref{thm_matrix_chernoff} to bound the minimum and maximum 
singular values of $\Omega_1$. 
We define a random variable $X_k$ by
\begin{equation}
  X_k = \frac{m}{s} V_1^T e_j e_j^T V_1 \quad \textrm{ with probability } m^{-1} \textrm{ for all } j = 1, \ldots m
\end{equation}
where $e_j$ is the $j^{\textrm{th}}$ standard basis vector and $V_1$ is as 
before. 

We draw $s$ such $X_k$ independently (with replacement). Then, 
\begin{equation}
Y = \sum X_k = V_1^T \Omega \Omega^T V_1  
\end{equation}
Note that $\lambda_i(Y) = \sigma_i(\Omega_1)^2$.

Using \eqref{e:gamma_def}, we have that the maximum eigenvalue of $X_k$ is 
\begin{equation}
L = \frac{m}{s} \coherence^{(r)}.
\end{equation}
Also, we have that
\begin{equation}
\mathbb{E} X_k = \sum_j m^{-1} \frac{m}{s} V_1^T e_j e_j^T V_1 = s^{-1} V_1^T V_1 = s^{-1} I  
\end{equation}
and, by linearity of expectation
\begin{equation}
\mathbb{E} Y = I,
\end{equation}
so $\mu_{\textrm{min}} = \mu_{\textrm{max}} = 1$.

Then, the conditions of Theorem~\ref{thm_matrix_chernoff} hold. 
We have two separate failure events -- either the largest eigenvalue of $Y$
is too large or the smallest is too small.
In the worst case, these two events are disjoint. So, the probability of either 
happening is bounded by the sum of the two probabilities. Therefore, we have 
\begin{equation}
  \prob\left(\lmin(Y) \leq (1 - \epsilon) \textrm{ or } \lmin(Y) \geq (1 + \epsilon) \right)  <  2 r \left[ \frac{e^{\epsilon }}{(1+\epsilon)^{1+\epsilon}} \right]^{1/L}  
\end{equation}
for $\epsilon \in [0, 1)$, where we use the fact that the right hand side of
\eqref{e:chernoff_upper} is larger than the right hand side of 
\eqref{e:chernoff_lower}.

Letting $\delta$ be our tolerance for failure, we can solve for the number
of samples needed as a function of $\epsilon$.
\begin{equation}
  \ns = m \coherence \log\left( \frac{2r}{\delta} \right) \left[ \log\left(  \frac{(1+\epsilon)^{1+\epsilon}}{e^\epsilon} \right) \right]^{-1}
\end{equation}

We now have that, except with probability at most $(1 - \delta)$,
\begin{equation}
  \lmax(Y) = \sigma_1^2(\Omega_1) \leq (1 + \epsilon) \quad \textrm{ and } \quad 
  \lmin(Y) = \sigma_r^2(\Omega_1) \geq (1 - \epsilon)
\end{equation}
In this event, since its smallest singular value is bounded away from zero, 
$\Omega_1$ has full row rank and we can finish 
bounding the terms in \eqref{eqn_three_terms_proof}.

We have that 
\begin{equation}
\lmax(\Omega_1 \Omega_1^T) = \| \Omega_1\|^2 \leq 1 + \epsilon
\end{equation}
and
\begin{equation}
\lmin(\Omega_1 \Omega_1^T)^{-1} = \| (\Omega_1 \Omega_1^T)^{-1} \| \leq (1 - \epsilon)^{-1}.
\end{equation}
We also use the bound
\begin{equation}
\|\Omega_2\|^2 = \| V_2^T \Omega \|^2 \leq \|V_2 \|^2 \|\Omega \|^2  = \frac{m}{s}
\end{equation}

Combining these with \eqref{eqn_three_terms_proof}, we have that
\begin{eqnarray}
\left\|(I - \Pi) A \right\|^2 
	& \leq & \left\|\Sigma_2 \right\|^2 + \left\| \Sigma_2 \Omega_2 \Omega_1^\dagger \right\|^2 \\
	& \leq & \sigma_{r+1}^2 + \left\|\Sigma_2 \right\|^2  \left\|\Omega_2 \right\|^2 \left\|\Omega_1^T \right\|^2 
		\left\| \left(\Omega_1 \Omega_1^T\right)^{-1} \right\|^2 \\
	& = & \left(1 + \frac{m}{s} (1 + \epsilon) (1 - \epsilon)^{-2} \right) \sigma_{r+1}^2
\end{eqnarray}

\textbf{Sampling without replacement.}
We have proved the result in the case of sampling with replacement. 
We can prove identical bounds for sampling without replacement in a 
straightforward way. Let $Z_k$ be a sequence of matrices that are equal to 
$X_k$ but are sampled without replacement. Then, from \cite{gross2010note}, we
have that  
\begin{equation}
  \mathbb{E}(\tr \exp(t \sum_k Z_k)) \leq   \mathbb{E}(\tr \exp(t \sum_k X_k))
\end{equation}
-- i.e. the MGF for $Z_k$ is dominated by the MGF for $X_k$. Then, we can 
complete the proof by following the 
tail bound proofs in \cite{tropp2015introduction}.
\end{proof}

\subsection{Proof of Theorem~\ref{thm_full_error_bound}}

We first require a theorem which gives us a bound on a rank $r$ 
approximation of a subsampled matrix \cite{halko2011finding}.
\begin{theorem}
Let $A \in \reals^{n \times m}$, $\subA \in \reals^{n \times s}$, and $\Pi$ be 
a projection onto the columns of $\subA$.
Let $\aA$ be the best rank $r$ approximation to $\Pi A$. Then,
\begin{equation}
\|A - \aA \| \leq \sigma_{r+1}(A) + \|(I - \Pi) A \|
\end{equation}
\label{thm_halko_svd_to_proj}
\end{theorem}
In other words, we only incur at most another factor of $\sigma_{r+1}(A)$ error if we do the SVD or ID on the projection of $A$ onto the columns we sampled. 


\begin{proof}
\emph{(of Theorem \ref{thm_full_error_bound}).}
We factor out the term $\|q\|$. Let $\subK_{(r)}$ be the best rank $r$ 
approximation of $\subK$. We insert this matrix using the triangle inequality
to obtain
\begin{equation}
\|K q - \subaK q \| \leq \|q\| \left( \|K - \subK_{(r)} \| 
	+ \|\subK_{(r)} - \subaK\| \right)
  \label{eqn_full_error_eqn1}
\end{equation}

We now bound the two terms on the right side of Equation 
\ref{eqn_full_error_eqn1} separately.  The first term can be bounded using 
theorem \ref{thm_halko_svd_to_proj}.
\begin{equation}
\|K - \subK_{(r)} \| \leq \sigma_{r+1}(K) + \|K \left(I - \Pi \right) \|
\end{equation}
where $\Pi$ is the projection onto the span of the subsampled rows.
Since we assume that the number of samples is chosen to satisfy theorem
\ref{thm_column_sampling_approx2}, we apply it to $K^T$ obtain
\begin{equation}
\|K \left(I - \Pi \right) \| \leq 
	\left( 1 + \left( 1 + \epsilon \right) \left( 1 - \epsilon \right)^2 \frac{m}{s} \right)^{\frac{1}{2}} \sigma_{r+1}(K)
\label{eqn_full_error_sampling_bound}
\end{equation}

The second term is just the error between using the SVD and ID to form rank 
$r$ approximations to the matrix $\subK$. Once again employing the triangle
inequality and Theorem \ref{thm_compute_id}, we have that
\begin{equation}
	\begin{array}{rcl}
\|\subK_{(r)} - \subaK\| & \leq & 
		\|\subK_{(r)} - \subK \| + \|\subK - \subaK\| \\
	& \leq & \sigma_{r+1}(\subK) + 
			\left(1 + n r (n - r) \right)^{\frac{1}{2}} \sigma_{r+1}(\subK)
	\end{array}
\end{equation}
Combining these bounds completes the proof.
\end{proof}

%
%

%
%
%




%
%
%
%
%

\section*{Acknowledgements}
This material is based upon work supported by AFOSR grants
FA9550-12-10484 and FA9550-11-10339; and NSF grants CCF-1337393,
OCI-1029022; and by the U.S. Department of Energy, Office of Science,
Office of Advanced Scientific Computing Research, Applied Mathematics
program under Award Numbers DE-SC0010518, DE-SC0009286, and DE-
FG02-08ER2585; and by the Technische UniversitŠt MŸnchen - Institute
for Advanced Study, funded by the German Excellence Initiative (and
the European Union Seventh Framework Programme under grant agreement
291763).  Any opinions, findings, and conclusions or recommendations
expressed herein are those of the authors and do not necessarily
reflect the views of the AFOSR or the NSF. Computing time on the Texas
Advanced Computing Centers Stampede system  was provided by an
allocation from TACC and the NSF.

\bibliographystyle{plain}
\bibliography{lafmm_sisc}

\end{document}